\documentclass[twoside,11pt]{article}
\usepackage{amsthm}
\usepackage[preprint]{jmlr2e}
\usepackage[T1]{fontenc}
\usepackage{mathptmx}
\usepackage{courier}
\usepackage{xurl}
\usepackage{microtype}
\usepackage{adjustbox}
\hypersetup{hidelinks,pdftitle={Distributionally Robust Deep Q-learning},pdfauthor={Chung I Lu; Julian Sester; Aijia Zhang}}
\usepackage[utf8]{inputenc}
\usepackage{amsmath}
\usepackage{amsfonts}
\usepackage{amssymb}

\usepackage{xcolor}
\usepackage{graphicx}
\usepackage[ruled,vlined]{algorithm2e}
\newcommand\mycommfont[1]{\footnotesize\ttfamily\textcolor{black}{#1}}
\SetCommentSty{mycommfont}

\usepackage{booktabs}
\usepackage{subfigure}

\newcommand{\D}{{\mathop{}\!\mathrm{d}}}

\newcommand{\R}{\mathbb{R}}
\newcommand{\Q}{\mathbb{Q}}

\newcommand{\HH}{\mathcal{H}}
\newcommand{\K}{\mathbb{K}}
\newcommand{\N}{\mathbb{N}}

\newcommand{\Mone}{\mathcal{M}_1(\Omega)}
\newcommand{\MoneX}{\mathcal{M}_1(\X)}
\newcommand{\PP}{\mathbb{P}}

\newcommand{\E}{\mathbb{E}}

\newcommand{\T }{\mathcal{T}}
\newcommand{\A}{\mathcal{A}}
\newcommand{\X}{\mathcal{X}}
\newcommand{\one}{ 1 \hspace{-3pt} \mathrm{l}}

\newcommand{\ab}{{\mathbf{a}}}
\newcommand{\Sball}{\mathcal{B}_{\varepsilon,\delta}}

\DeclareMathOperator*{\argmax}{arg\,max}
\DeclareMathOperator{\supp}{supp}
\newcommand{\calQ}{\mathcal{Q}}

\newcommand{\norm}[1]{\left\lVert #1 \right\rVert}

\numberwithin{equation}{section}

\newtheorem{defn}{Definition}[section]

\newtheorem{rem}[defn]{Remark}
\newtheorem{thm}[defn]{Theorem}
\newtheorem{prop}[defn]{Proposition}
\newtheorem{cor}[defn]{Corollary}
\newtheorem{lem}[defn]{Lemma}

\newtheorem{s_asu}[defn]{Standing Assumption}
\newtheorem{asu}[defn]{Assumption}
\newtheorem{opt}[defn]{Optimisation Problem}

\ShortHeadings{Distributionally Robust Deep Q-learning}{Lu, Sester and Zhang}
\firstpageno{1}
\title{Distributionally Robust Deep $Q$-learning}
\author{\name Chung I Lu \\
\name Julian Sester \email jul\_ses@nus.edu.sg \\
\name Aijia Zhang \\
\addr Department of Mathematics, National University of Singapore\\
Level 4, Block S17, 10 Lower Kent Ridge Road, Singapore 119076}
\begin{document}
\maketitle
\begin{abstract}
We propose a novel distributionally robust $Q$-learning algorithm for the non-tabular case accounting for continuous state spaces where the state transition of the underlying Markov decision process is subject to model uncertainty. The uncertainty is taken into account by considering the worst-case transition from a ball around a reference probability measure. To determine the optimal policy under the worst-case state transition, we solve the associated non-linear Bellman equation by dualising and regularising the Bellman operator with the Sinkhorn distance, which is then parameterised with deep neural networks. This approach allows us to modify the Deep Q-Network algorithm to optimise for the worst case state transition.
    We illustrate the tractability and effectiveness of our approach through several applications, including a portfolio optimisation task based on S\&{P}~500 data.
    { We also establish convergence guarantees for exact and approximate robust fitted $Q$-iteration, decompose the numerical RDQN error into interpretable components, and discuss extensions to compact continuous action sets.}
\end{abstract}

\section{Introduction} \label{sec:intro}

Reinforcement learning (RL) has emerged as a powerful paradigm for training intelligent agents to make optimal decisions in complex environments \citep{sutton2018reinforcement}.
A central concept within RL is Markov Decision Processes (MDPs), which provide a mathematical formalism for sequential decision-making under uncertainty. Traditional RL algorithms often assume a perfectly known model of the environment's dynamics.
However, in many real-world applications, the state transition probabilities of the underlying MDP are subject to uncertainty due to factors such as limited data, noisy measurements, or inherent stochasticity.
This model misspecification can lead to policies that perform poorly when deployed in the actual environment, simply because the agent is trained in the wrong model.

To address the challenge of model uncertainty, Distributionally Robust Optimisation (DRO) \citep{rahimian2019distributionally} offers a principled approach by considering a set of plausible probability distributions, an ambiguity set, and optimising against the worst-case distribution within this set.
Recent research has begun to integrate DRO into MDPs \citep{wiesemann2013robust}, leading to the development of distributionally robust RL algorithms.
One of the types of ambiguity sets that has gained traction is the Wasserstein ball around a possibly estimated reference measure, which quantifies the distance between probability measures using the Wasserstein metric \citep{villani2009optimal}.
While significant progress has been made in the tabular setting with discrete state and action spaces for this type of ambiguity sets, extending these techniques to continuous state spaces presents substantial challenges, particularly in solving the associated non-linear Bellman equation.

In this paper, we propose a novel distributionally robust Q-learning algorithm specifically designed for continuous state spaces where the state transition of the underlying Markov decision process is subject to model uncertainty.
We model this uncertainty by considering the worst-case transition from a ball around a reference probability measure, quantified by the Sinkhorn distance which is a regularised version of the Wasserstein distance \citep{cuturi2013sinkhorn}.
To determine the optimal policy under these worst-case transitions, we tackle the associated non-linear Bellman equation by dualising the Bellman operator \citep{wang2021sinkhorn}.
This regularised problem yields a more tractable dual formulation, which we then parameterise using deep neural networks (see \citet{hornik1990universal}, \citet{kidger2020universal}, \citet{pinkus1999approximation}, \citet{scarselli1998universal}), allowing us to adapt the Deep Q-Network (DQN) algorithm \citep{mnih2015human} to optimise for the worst-case state transition.
The case of Wasserstein uncertainty can be approximated by choosing a small value for the regularisation parameter.

\subsection{Contributions}

Our main contributions can be summarised as follows:

\begin{enumerate}
	\item We introduce a distributionally robust Q-learning framework for continuous state spaces and discrete action spaces based on the Sinkhorn distance.
	\item We prove that the dynamic programming principle applies to the robust MDP using the Sinkhorn ball as the ambiguity set, hence allowing us to derive a robust Bellman equation.
	\item We address the intractability of the robust Bellman equation by dualising the optimisation problem leading to a more tractable formulation.
	\item We develop a practical algorithm, \textbf{Robust DQN (RDQN)}, by parameterising the robust Q-function with deep neural networks and deriving a modified loss function based on the dual formulation, enabling optimisation using stochastic gradient descent.
	\item We provide theoretical guarantees for the existence of solutions when the state space is compact.
        {\item We extend the robust operators to compact continuous action sets, quantify action-grid and entropy-regularisation errors, and relate critic accuracy to robust policy performance. These results concern exact or error-controlled fitted iteration, not unrestricted actor-critic training.}
 	\item We provide evidence for the applicability of RDQN through three illustrative applications: first a toy example involving agent-environment interaction, and two more realistic and complex settings for portfolio optimisation based on the S\&{P}~500 index and a set of US stocks.
\end{enumerate}

\subsection{Related Literature}

The application of DRO to RL has led to the emergence of distributionally robust RL, aiming to develop policies that are resilient to uncertainties in the underlying MDP parameters.
The theoretical foundations were laid in \citet{iyengar2005robust} and \citet{nilim2005robust} which established the concept of distributional robustness in discrete time MDPs and derived the distributionally robust Bellman equation under the assumption of discrete state and action spaces.
The result was extended to continuous state and action spaces in \citet{bauerle2022distributionally} and \citet{neufeld2024non} for finite time horizon and \citet{neufeld2023markov} for infinite time horizon.
Central to the results is that the ambiguity set has a rectangularity property which means that the ambiguity related to any state-action pair is independent of the ambiguity related to other state-action pairs.
Ambiguity sets that fall within this framework include the balls around a reference measure, which are defined using distances such as the Wasserstein distance or Kullback-Leibler (KL) divergence.

Algorithms in the literature aiming to solve robust MDPs can be classified based on several factors: the type of ambiguity set employed, the nature of the state and action spaces (e.g., continuous vs. discrete), and the approach taken to solve the MDP.

One of the earlier works in this area is \citet{xu2010distributionally} which defines the ambiguity based on uncertainties in parameters of a distribution but the algorithm solving the problem sequentially is meant for smaller state and action spaces.

\citet{liu2022distributionally}, \citet{wang2023finite} and \citet{wang2024sample} focus on the finite state and action space setting using a KL divergence ball as the ambiguity set.
All three works use  $Q$-learning approaches (see \citet{watkins1992q}, \citet{hasselt2010double}, \citet{tan2017deep}, \citet{fan2020theoretical}), each improving on the sample complexity.
The use of the KL divergence is motivated by the use of a duality to derive the robust Bellman operator.
Computing the operator requires sampling the transition from the same state-action pair under the reference measure multiple times.
This can be simplified in the analogous Sinkhorn duality where the sampling can be based on a suitably chosen prior measure.
\citet{smirnova2019distributionally} also uses the KL divergence but the ambiguity is instead defined on the policy and they use function approximation to extend to continuous state and action spaces with an actor-critic algorithm.

\citet{panaganti2022robust} uses a ball based on the total variation in a finite state and action space and uses a fitted $Q$-iteration approach.
\citet{ramesh2024distributionally} uses Gaussian processes to model the transition and derives sample complexity for error bounds depending on the type of ambiguity set, assuming the optimal robust policy can be obtained by some chosen algorithm.
\citet{wang2021online} and \citet{wang2022policy} use the R-contamination model to define the ambiguity set and solve the continuous state and action space robust MDP using a $Q$-learning and policy gradient based algorithm respectively.
\citet{decker2024robust} uses a finite set of pre-defined measures in a discrete state and action space and also utilises $Q$-learning.

\citet{yang2017convex} and \citet{yang2020wasserstein} consider the Wasserstein ball but frame the problem as a convex optimisation problem.
\citet{neufeld2024robust} also uses the Wasserstein ball for finite state and action spaces and solving the robust MDP using a $Q$-learning approach but it has a similar sampling issue as the KL divergence approach.

There are also works that focus on scenarios where the ambiguity set is not state-action rectangular leading to general robust MDPs which are known to be NP-hard \citep{wiesemann2013robust}.
The complexity of algorithms used to solve these general MDPs varies based on the specific structure of the ambiguity set.
In \citet{wang2023policy}, a double loop policy gradient based method is proposed to solve robust MDPs with state rectangular ambiguity sets.
For generally non-rectangular ambiguity sets, \citet{li2023policy} proposed an actor-critic algorithm that approximates a solution with the complexity scaling with inverse-quartic dependence on the error.
Other works that consider non-sa-rectangular ambiguity sets include \citet{kumar2023efficient}, \citet{wiesemann2013robust}, \citet{mannor2016robust}, \citet{neufeld2024robustb} and \citet{goyal2023robust}.

Another line of work investigates the relationship between regularisation and robustness in the context of MDPs.
These include \citet{eysenbach2021maximum} and \citet{derman2023twice} which show the equivalence between the regularised MDP and certain classes of robust MDPs.

In terms of the use of the Wasserstein or Sinkhorn distance, it comes from the field of optimal transport \citep{villani2009optimal}.
The original Wasserstein distance is computationally expensive to compute, especially in high-dimensional spaces \citep{genevay2019sample}.
\citet{cuturi2013sinkhorn} introduced the Sinkhorn distance, which regularises the Wasserstein distance by adding an entropic term to the optimisation problem thereby enforcing some smoothness.
It is named after the algorithm \citep{sinkhorn1964relationship} used to compute the distance and has better computational and sample complexity than the Wasserstein distance (see \citet{lin2022efficiency}, \citet{genevay2019sample}).
There are additional variants that have been proposed, such as the Sinkhorn divergence, unbalanced optimal transport and combinations of these (see e.g. \citet{feydy2019interpolating}, \citet{sejourne2019sinkhorn}).

We are using a more general definition of the Sinkhorn distance similar to \citet[Definition I]{wang2021sinkhorn} where the entropic regularisation is defined with respect to the product measure of the reference measure and a prior measure.
The prior measure is chosen such that all measures in the ambiguity set are absolutely continuous with respect to it.
This grants us added flexibility when computing the robust Bellman operator as it allows us to sample from the prior measure instead of the reference measure as in the KL divergence case.

The remainder of this paper is as follows.
Section \ref{sec:setting} lays out the framework for our approach before we build our algorithm in Section \ref{sec:algorithm} followed by experiments in Section \ref{sec:applications}.
We conclude in Section \ref{sec:conclusion} and leave proofs in Section \ref{sec:aux_results_proofs}.

\section{Setting and Preliminaries} \label{sec:setting}

In this section we present distributionally robust Markov decision processes in line with the presentation from \citet{neufeld2024robust}, and we discuss the associated optimisation problem that we aim to solve by the use of a modified Deep Q-Network (DQN) algorithm \citep{mnih2015human}.

\subsection{Distributionally Robust Markov Decision Processes}

To model the \emph{state space} of the Markov decision process, we consider a closed but not necessarily compact subset $\X \subseteq\R^d$ which we use to define the space on which the infinite time horizon stochastic process takes values, given by the infinite Cartesian product
\begin{equation*}
	\Omega:=\X^{\N}=\X\times \X \times \cdots.
\end{equation*}

We denote by $\Mone$ the set of all probability measures on $\Omega$ equipped with its Borel-$\sigma$-algebra.
We model the evolution of the attained states via an infinite horizon time-discrete stochastic process and, to this end, we define on $\Omega$ the stochastic process $\left(X_{t}\right)_{t\in \N_0}$  by the canonical process
$$
X_t(\left(\omega_0,\omega_1,\dots,\omega_t,\dots\right)):=\omega_t, \text{ for } (\omega_0,\omega_1,\dots,\omega_t,\dots) \in \Omega,~~t \in \N_0.
$$

To model the set of actions (also called controls), we fix a \emph{finite} set $A \subseteq \R^m$ and we define
\begin{align*}
	\A:&=\left\{\left(a_t(X_t)\right)_{t\in \N_0}~\middle|~ a_t:\X \rightarrow A \text{ Borel measurable for all } t \in \N_0 \right\}.
\end{align*}

Instead of fixing the distribution of the state transition between states from the underlying state space (in dependence of a chosen action), we want to account for a possible model misspecification by allowing for a range of possible distributions.
We model distributional uncertainty through optimal transport distances.

\begin{defn}[Wasserstein-distance] \label{def:wasserstein}
For any $\PP_1,\PP_2 \in \MoneX$ let the Wasserstein-distance $W(\PP_1,\PP_2)$ be defined as
\begin{equation*}
W(\PP_1,\PP_2):=\inf_{\pi \in \Pi(\PP_1,\PP_2)}\int_{\X \times \X} \|x-y\| \D \pi(x,y),
\end{equation*}
where $\|\cdot\|$ denotes the Euclidean norm on $\R^d$ and $\Pi(\PP_1,\PP_2)$ denotes the set of joint distributions of $\PP_1$ and $\PP_2$.
\end{defn}

\begin{defn}[Sinkhorn distance] \label{def:sinkhorn}
	Let $\delta>0$ denote some regularisation parameter, and let $\PP_1,\PP_2, \nu \in \MoneX$ where $\PP_2 \ll \nu$ \footnote{The definition of the Sinkhorn distance could be generalised to some extent by choosing two reference measures $\mu, \nu$ that are not necessarily probability measures such that $\PP_1 \ll \mu$, $\PP_2 \ll \nu$ and the relative entropy is defined as $H(\pi~|~\mu \otimes \nu)$ instead (see \citet{wang2021sinkhorn}).}. Then, $W_{\delta}(\PP_1,\PP_2)$ is defined as
	\begin{equation*}
		W_{\delta}(\PP_1,\PP_2):=\inf_{\pi \in \Pi(\PP_1,\PP_2)}\left\{\int_{\X \times \X} \|x-y\| \D \pi(x,y) + \delta H(\pi~|~\PP_1 \otimes \nu) \right\}
	\end{equation*}
	where $H(\pi~|~\PP_1 \otimes \nu)$ denotes the relative entropy of $\pi\in \Pi(\PP_1,\PP_2)$ with respect to the product measure $\PP_1 \otimes \nu$, i.e.,
	\begin{equation*}
		H(\pi~|~\PP_1 \otimes \nu):=\E_{(x,y)~\sim \pi}\left[\log\left(\frac{\D \pi(x,y)}{\D \PP_1(x) \D \nu(y)}\right)\right] = \int_{\X \times \X} \log\left(\frac{\D \pi(x,y)}{\D \PP_1(x) \D \nu(y)}\right) \D \pi(x,y),
	\end{equation*}
	where $\frac{\D \pi(x,y)}{\D \PP_1(x) \D \nu(y)}$ denotes the Radon--Nikodym derivative of $\pi$ w.r.t.\,$\PP_1 \otimes \nu$ evaluated at $(x,y) \in \X\times \X$.
\end{defn}

Definition~\ref{def:wasserstein} is technically for the Wasserstein-1 distance\footnote{The Wasserstein-p distance generalises the cost function used in the integral where \\ $W^p(\PP_1,\PP_2)=\inf_{\pi \in \Pi(\PP_1,\PP_2)}\int_{\X \times \X} \|x-y\|^p \D \pi(x,y)$} which we will simply refer to as the Wasserstein distance.
The Sinkhorn distance incorporates an entropic regularisation term compared to the Wasserstein distance.
We shall define
\begin{equation}
	W_\delta(\PP_1,\PP_2) |_{\delta=0} := W_0(\PP_1,\PP_2) := W(\PP_1,\PP_2)
\end{equation}

\begin{rem}[On the role of $\nu$]
	Most definitions of the Sinkhorn distance in the literature are given for the case $\nu=\PP_2$ (see e.g. \citet{cuturi2013sinkhorn}).
	Our choice of Definition~\ref{def:sinkhorn} with the probability measure $\nu$ is motivated by the flexibility it provides when sampling to estimate the $Q$ function values in Algorithm \ref{algo:robustdqn} compared to having to sample from the environment which can be more costly.
	It also plays an important role as a prior for the worst case distribution and fixes the support of all measures taken into account (see also \citet[Remark 2 and Remark 4]{wang2021sinkhorn}).
	We illustrate the impact of $\nu$ on the worst case distribution in Section \ref{sec:worst_case_dist}.
\end{rem}

From now on, we fix some \emph{reference measure} $\widehat{\PP}$ that constitutes the best estimate or guess of the real behaviour of the environment.
In practice such a measure can often be derived from the observed history of realised states (see, e.g. \citet{rust1994structural}, \citet{wiesemann2013robust}), e.g., via the empirical distribution.
Concerning the \emph{reference measure} $\widehat{\PP}$ we impose the following technical assumptions.

\begin{asu} \label{asu_1}
	We assume that there exists a continuous map (in the Wasserstein-1 topology $\tau_1$, see \citet[Definition 6.8]{villani2009optimal})
	\begin{equation} \label{eq:definition_p_hat}
		\begin{split}
			\X \times A &\rightarrow (\MoneX,\tau_1)\\
			(x,a) &\mapsto \widehat{\PP}(x,a)
		\end{split}
	\end{equation}
	such that $\widehat{\PP}(x,a)$ has finite first moment for all $(x,a) \in \X \times A$.
\end{asu}

In order to use the Sinkhorn distance to define an ambiguity set of probability measures, we define for any $\varepsilon>0$ and $\delta \geq 0$ the set-valued map
\begin{equation} \label{eq:def_sinkhorn}
	\X \times A \ni (x,a) \twoheadrightarrow \Sball\left(\widehat{\PP}(x,a)\right):=\left\{\PP\in \MoneX ~\middle|~W_{\delta}(\widehat{\PP}(x,a),\PP) \leq \varepsilon \right\},
\end{equation}
where $\Sball\left(\widehat{\PP}(x,a)\right)$ denotes the Sinkhorn ball with $\varepsilon$-radius and center $\widehat{\PP}(x,a)$.

Since $W_\delta(\widehat{\PP}(x,a),\widehat{\PP}(x,a)) \neq 0$ in general\footnote{Except in the Wasserstein case, i.e. $\delta=0$, in which we always have $W_0(\widehat{\PP}(x,a),\widehat{\PP}(x,a)) = 0$.}, we make the following assumption to ensure the Sinkhorn ball is not empty.
\begin{asu}[Reference measure is in the Sinkhorn ball] \label{asu_3}
	We assume
	\begin{equation}
		\varepsilon \geq \sup_{(x,a) \in \X \times A} W_\delta(\widehat{\PP}(x,a),\widehat{\PP}(x,a)).
	\end{equation}
\end{asu}

{
\begin{rem}\label{rem:epsilon_min}
    Assumption \ref{asu_3} ensures that the reference measure belongs
    to the Sinkhorn ball and thus implies that the Sinkhorn ball is non-empty, but non-emptiness alone requires a potentially
    smaller radius. More precisely, for $\delta>0$ and
    $(x,a)\in\X\times A$, the minimum radius is
    \begin{equation}\label{eq:epsilon_min}
        \varepsilon_{\min}
        := \min_{\PP\in\mathcal{M}_1(\X)}
        W_\delta(\widehat{\PP}(x,a),\PP)
        = -\delta \E_{X\sim\widehat{\PP}(x,a)}
        \left[
            \log\left(
                \E_{Y\sim\nu}
                \left[e^{-\|X-Y\|/\delta}\right]
            \right)
        \right].
    \end{equation}
    Indeed, disintegrating a coupling with first marginal
    $\widehat{\PP}(x,a)$ reduces the minimisation to choosing,
    for each $z\in\X$, a conditional distribution $K(z,\cdot)$
    that minimises
    \[
        \int_{\X}\|z-y\|\,K(z,dy)
        +\delta H(K(z,\cdot)\mid\nu).
    \]
    By the entropy variational formula
(see \citet[Lemma 1.3]{nutz2021introduction}), the minimum equals
    $-\delta\log\int_{\X}e^{-\|z-y\|/\delta}\nu(dy)$
    and is attained at
    \[
        K_\delta(z,dy)
        =
        \frac{e^{-\|z-y\|/\delta}}
        {\int_{\X}e^{-\|z-u\|/\delta}\nu(du)}
        \,\nu(dy).
    \]
    Integrating these minimum values against
    $\widehat{\PP}(x,a)$ yields \eqref{eq:epsilon_min};
    the second marginal
    $\PP_{\min}(dy)
    =\int_{\X}K_\delta(z,dy)\widehat{\PP}(x,a)(dz)$
    attains the minimum.
    Thus, the Sinkhorn ball is non-empty if and only if
    $\varepsilon\ge\varepsilon_{\min}$, whereas
    \[
        \varepsilon_{\min}
        \le W_\delta(\widehat{\PP}(x,a),\widehat{\PP}(x,a)),
    \]
    so Assumption \ref{asu_3} can be stronger than necessary
    for non-emptiness.
\end{rem}
}
Finally, we define ambiguity sets of probability measures on the whole time horizon for every $x \in \X$ and every action $\ab \in \A$ via
\begin{align*}
	\mathfrak{P}^\delta_{x,\ab}:=\bigg\{\delta_x \otimes \PP_0\otimes \PP_1 \otimes \cdots~\bigg|~&\text{ for all } t \in \N_0:~\PP_t:\X \rightarrow \MoneX \text{ Borel-measurable, } \\
	&\text{ and } \PP_t(\omega_t) \in \Sball\left(\widehat{\PP}(\omega_t,a_t(\omega_t))\right)
	\text{ for all } \omega_t\in \X \bigg\},
\end{align*}
where the notation $\PP=\delta_x \otimes\PP_0\otimes \PP_1 \otimes\cdots \in \mathfrak{P}^\delta_{x,\ab}$ abbreviates
\begin{equation*}
	\PP(B):=\int_{\X}\cdots \int_{\X} \cdots \one_{B}\left((\omega_t)_{t\in \N_0}\right) \cdots \PP_{t-1}(\omega_{t-1};\D\omega_t)\cdots \PP_0(\omega_0;\D\omega_1) \delta_x(\D \omega_0),\qquad B \in \mathcal{F}.
\end{equation*}

\subsection{Optimisation Problem}

Let $r:\X \times A \times \X \rightarrow \R$ be some \emph{reward function} modelling the feedback received on the quality of the realised state upon execution of an action.
We assume from now on that the reward function and discount factor $\gamma$ fulfil the following assumptions.

\begin{s_asu}[Assumptions on the reward function and the discount factor]\label{asu_2}~
	\begin{itemize}
		\item[(i)] The map
		\begin{equation*}
			\X \times A \times \X \ni (x_0,a,x_1) \mapsto r(x_0,a,x_1)
		\end{equation*}
		is continuous and bounded.
		\item[(ii)]There exists some $L > 0$ such that for all $x_0,x_0',x_1\in \X$ and $a,a'\in A$ we have
		\begin{equation*}\label{eq:c_Lipschitz}
			\left|r(x_0,a,x_1)-r(x_0',a',x_1)\right|\leq L \cdot \left(\|x_0-x_0'\|+\|a-a'\|\right).
		\end{equation*}
		\item[(iii)]
		We fix an associated \emph{discount factor} $0 < \gamma < 1$.
	\end{itemize}
\end{s_asu}

The optimisation problem is to maximise the expected value of $\sum_{t=0}^\infty \gamma^tr(X_{t},a_t,X_{t+1})$, for every initial value $x\in \X$, under the worst case measure from $\mathfrak{P}^\delta_{x,\ab}$ over all possible actions $\ab \in \A$.
More precisely, we introduce the optimal value function
\begin{equation} \label{eq:robust_problem_1}
	\begin{split}
		\X \ni x \mapsto V_\delta(x):=\sup_{\ab \in \A}\inf_{\PP \in \mathfrak{P}^\delta_{x,\ab}} \left(\E_{\PP}\bigg[\sum_{t=0}^\infty \gamma^tr(X_{t},a_t,X_{t+1})\bigg]\right).
	\end{split}
\end{equation}

\subsection{Dynamic Programming: The Bellman Equation}

In \citet[Theorem 2.7]{neufeld2023markov}, it was shown that under a certain set of assumptions, a dynamic programming principle holds which allows for a robust version of the \emph{Bellman} equation, named after the pioneering contributions of Richard E. Bellman (see \citet{bellman1952theory}).
In particular, it applies when the ambiguity set is a Wasserstein ball around a reference measure.
We extend this result to the Sinkhorn distance and the associated ambiguity set.
\begin{prop}[Robust Bellman equation for the Sinkhorn ball ambiguity set] \label{prop:bellman}
	Let $\varepsilon > 0, \delta \geq 0$ and assume that Assumptions \ref{asu_1}, \ref{asu_3} and \ref{asu_2} hold, then
	\begin{equation} \label{eq:bellman}
		V_\delta(x) = \T_\delta V_\delta(x) \text{ for all } x \in \X
	\end{equation}
	for the operator $\T_\delta$ being defined as
	\begin{align*}
		\X \ni x \mapsto \T_\delta V_\delta(x):&= \sup_{a \in A} \inf_{\PP \in \Sball \left(\widehat{\PP}(x,a)\right)} \E_{\PP} \left[r(x,a,X_1)+\gamma V_\delta(X_1)\right].
	\end{align*}
\end{prop}

\begin{proof}
    The proof involves showing the settings described in Section \ref{sec:setting} satisfy the assumptions of \citet[Theorem 2.7]{neufeld2023markov} with $p=0$ in the notation of \citet{neufeld2023markov}.
	In addition to Assumptions \ref{asu_1}, \ref{asu_3} and \ref{asu_2}, we need to show that the Sinkhorn ball $\Sball\left(\widehat{\PP}(x,a)\right)$ is weakly compact and continuous.
	Details of the proof for the case $\delta>0$,  can be found in Section \ref{sec:proofs}.
    The proof for the case $\delta=0$ can be found in \citet[Proposition 3.1]{neufeld2023markov}.
\end{proof}

\subsection{\texorpdfstring{$Q$}{Q}-learning}\label{sec:qlearning}

The goal of $Q$-learning, as introduced in \citet{watkins1992q}, is to make use of the Bellman equation~\eqref{eq:bellman} to learn the optimal policy $\ab \in \A$ to maximising \eqref{eq:robust_problem_1}.
By \citet[Theorem 2.7~(iii)]{neufeld2023markov}, it follows that the optimal policy is Markovian and stationary, i.e., it is of the form $\ab =(a^*(X_t))_{t\in \N}$ for some measurable decision rule $a^*:\X \rightarrow A$.
Hence, it suffices to determine an optimal action $a^*(x)$ for each state $x$ that maximises the term in the supremum of the Bellman operator $\T_\delta V_\delta$, i.e.,
\begin{equation}
 	\inf_{\PP \in \Sball \left(\widehat{\PP}(x,a^*(x))\right)} \E_{\PP} \left[r(x,a^*(x),X_1)+\gamma V_\delta(X_1)\right]= V_\delta(x).
\end{equation}

We define the optimal robust $Q$ function
\begin{equation} \label{eq:q_val_definition}
	\X\times A \ni (x,a) \mapsto Q_\delta^*(x,a):= \inf_{\PP \in \Sball \left(\widehat{\PP}(x,a)\right)} \E_{\PP} \left[r(x,a,X_1)+\gamma V_\delta(X_1)\right].
\end{equation}

With this definition, we obtain via Equation~\eqref{eq:bellman} that
\begin{equation} \label{eq:sup_Q_V}
	\sup_{a\in A}Q_\delta^*(x,a)=V_\delta(x) \text{ for all } x \in \X,
\end{equation}
and thus computing $Q_\delta^*$ allows us to determine the optimal action in each state $x\in \X$.
Now we observe by Equation~\eqref{eq:sup_Q_V} that
\begin{equation} \label{eq:HQ_equals_Q}
	\begin{split}
		(\HH_\delta Q_\delta^*)(x,a)&:=\inf_{\PP \in \Sball \left(\widehat{\PP}(x,a)\right)} \E_{\PP} \left[r(x,a,X_1)+\gamma\sup_{b \in A}Q_\delta^*(X_1,b)\right]\\
		&=\inf_{\PP \in \Sball \left(\widehat{\PP}(x,a)\right)} \E_{\PP} \left[r(x,a,X_1)+\gamma V_\delta(X_1)\right]= Q_\delta^*(x,a).
	\end{split}
\end{equation}

{
Under Assumptions \ref{asu_1} -- \ref{asu_2} the map $\X \times A \ni (x,a) \mapsto Q_\delta^*(x,a)$ is continuous and bounded (compare Lemma~\ref{lem:Q_continuous}). This motivates extending the notation in \eqref{eq:HQ_equals_Q} to general continuous and bounded functions $Q$.  To this end, let $C_b(\X \times A)$ denote the space of continuous and bounded functions $Q:\X\times A\to\R$, equipped with the supremum norm.
 For $Q\in C_b(\X \times A)$, we define
\begin{equation}\label{eq:robust_bellman_operator_general}
    (\HH_\delta Q)(x,a)
    :=
    \inf_{\PP \in \Sball\left(\widehat{\PP}(x,a)\right)}
    \E_{\PP}\left[r(x,a,X_1)+\gamma\sup_{b\in A}Q(X_1,b)\right].
\end{equation}
Applying the above defined operator $\HH_\delta$ iteratively to an arbitrary initial $Q_0\in C_b(\X \times A)$  then converges to $ Q_\delta^*$.

\begin{thm}[Convergence of exact robust $Q$-iteration]\label{thm:exact_robust_q_iteration}
For any $Q_0\in C_b(\X \times A) $, define
\begin{equation*}
    Q_{k+1}=\HH_\delta Q_k,
    \qquad k=0,1,2,\ldots .
\end{equation*}
Then there exists a unique fixed point $Q_\delta^\ast\in C_b(\X \times A)$ satisfying $\HH_\delta Q_\delta^\ast=Q_\delta^\ast$, and $Q_k$ converges uniformly to $Q_\delta^\ast$.  Moreover,
\begin{equation}\label{eq:exact_robust_iteration_rate}
    \norm{Q_k-Q_\delta^\ast}_\infty
    \leq
    \gamma^k\norm{Q_0-Q_\delta^\ast}_\infty .
\end{equation}
\end{thm}
}
This leads directly to the idea of $Q$-learning which is to solve the fixed point equation $\HH_\delta Q_\delta^* = Q_\delta^*$ { via iteratively applying $\HH_\delta$.}
In the case of a discrete state and action space, a corresponding algorithm exists for the Wasserstein case (see \citet{neufeld2024robust}).
For infinitely many states and actions, the algorithm is however infeasible which is why we propose to approximate the optimal $Q$ function with neural networks by pursuing a similar approach as in \citet{mnih2015human} or \citet{vanhasselt2016deep} in the non-robust case.
{
In particular, the convergence statement from Theorem~\ref{thm:exact_robust_q_iteration} should be read separately from the stochastic optimisation routine (RDQN) we propose to train the neural network in Section~\ref{sec:algorithm}: the exact robust Bellman iteration is contractive, while RDQN is a numerical implementation whose approximation errors must be controlled or monitored, which we discuss in more detail in Section~\ref{sec:approx_fitted_q_iteration}.
}

\section{The Robust \texorpdfstring{$Q$}{Q}-learning Algorithm} \label{sec:algorithm}

\subsection{Dualising and Regularising the Robust Optimisation Problem} \label{sec:regularising}

As discussed in Section~\ref{sec:qlearning}, our goal is to solve the fixed point equation  $\HH_\delta Q_\delta^* = Q_\delta^*$.
Directly computing the infimum over the Sinkhorn ball is intractable in practice.
Instead, we follow the procedure proposed in \citet{wang2021sinkhorn} and consider the dual formulation of the robust optimisation problem which is in a tractable form.
Throughout this section, we assume that Assumptions \ref{asu_1}, \ref{asu_2} and \ref{asu_3} hold.
\begin{prop} \label{prop:dual}
	Let $\nu \in \MoneX$. Assume that $\nu\left(\{0 \leq \| y - X_1^{\nu}\| < \infty\}\right)=1$, and that
    \begin{equation*}
        \E_{X_1^{\nu} \sim \nu}\left[\exp\left(\frac{- \| y - X_1^{\nu}\|}{ \delta}\right)\right] < \infty
    \end{equation*}
    for $\widehat{\PP}(x,a)$-almost every y, for all $(x,a) \in \X \times A$.
    Then, we have for all $(x,a) \in \X \times A$ and for $\varepsilon,\delta>0$ such that
	\begin{equation*}
		\overline{\varepsilon}:= \varepsilon + \delta \E_{X_1^{\PP} \sim \widehat{\PP}(x,a)}\left[\log\left(\E_{X_1^{\nu} \sim \nu}\left[\exp\left(\frac{- \| X_1^{\PP}-X_1^{\nu}\|}{ \delta}\right)\right]\right)\right] \geq 0
	\end{equation*}
	the following duality
	\begin{align*}
		&\HH_{\delta}Q_\delta^*(x,a) \\
		&= \sup_{\lambda >0} \bigg\{-\lambda {\varepsilon}- \lambda \delta \E_{X_1^{\PP} \sim \widehat{\PP}(x,a)}\left[\log\left(\E_{X_1^{\nu} \sim \nu}\left[\exp\left(\tfrac{-r(x,a,X_1^{\nu})-\gamma \sup_{b \in A}Q_\delta^*(X_1^{\nu},b)-\lambda \| X_1^{\PP}-X_1^{\nu}\|}{\lambda \delta}\right)\right]\right)\right]\bigg\}.
	\end{align*}
\end{prop}

\begin{proof}
	This follows from \citet[Theorem I]{wang2021sinkhorn} and from using that $\sup_z f(z) = - \inf_z - f(z)$.
    \citet[Theorem I]{wang2021sinkhorn} in addition requires that the function $f(z):\X \to \R$ is measurable and that every joint distribution on $\X \times \X$ with $\widehat{\PP}(x,a)$ as the first marginal has a regular conditional distribution given the value of the first marginal.
    For our setting, the function $f(z) = r(x,a,z) + \gamma \sup_{b \in A}Q_\delta^*(z,b)$ is measurable as it is continuous due to Assumption~\ref{asu_2} and the continuity of $Q_\delta^*$ as shown in Lemma~\ref{lem:Q_continuous}.
    By \citet[Theorem 8.29]{klenke2013probability}, real-valued random variables always have a regular conditional distribution.
\end{proof}

The optimisation problem over a set of probability measures in the Sinkhorn ball in Proposition~\ref{prop:dual} is transformed into one that only requires maximising over the scalar $\lambda$. The inner expectation is under the measure $\nu$ which offers flexibility in sampling.
If one were mainly interested in the Wasserstein-1 ball as the ambiguity set, the following observation offers a direct connection to the Sinkhorn distance.

\begin{cor} \label{cor:convergence_delta}
	{Let $\varepsilon > 0$ and assume that the prior measure $\nu$ has full support on $\X$, i.e., $\supp(\nu)=\X$. Then, for all $(x,a) \in \X \times A$, we have}
	\begin{equation*}
		\lim_{\delta \downarrow 0} \HH_\delta Q_0^*(x,a)= \inf_{\PP \in \mathcal{B}_{\varepsilon,0} \left(\widehat{\PP}(x,a)\right)} \E_{\PP} \left[r(x,a,X_1)+\gamma\sup_{a\in A}Q_0^*(X_1,a)\right] := \HH_0 Q_0^*(x,a) = Q_0^*(x,a)
	\end{equation*}
\end{cor}

\begin{proof}
    {The full-support condition ensures that the vanishing-entropy limit does not retain a restriction to a proper subset of $\X$.} The assertion then follows from \citet[Appendix EC.4]{wang2021sinkhorn} and from \eqref{eq:HQ_equals_Q}.
\end{proof}

In other words, as $\delta$ approaches $0$, the Bellman operator $\HH_\delta$ converges to the Bellman operator $\HH_0$ associated with the Wasserstein ball.
In the classical dual formulation of the Wasserstein ball robust Bellman operator, the computation of $\HH_0 Q_0^*$ involves the computation of the so called $\lambda -c$ transform\footnote{The $\lambda-c$ transform of a function $f: \X \rightarrow \R$ is given by $f^{\lambda c}(x):= \sup_{y \in \X}\{f(y)-\lambda c(x,y)~|~f(y)<\infty\}$, $x \in \X$, for some distance function $c:\X \times \X \rightarrow \R$.} on the dual side of the problem (compare \citet[Remark 2.1]{bartl2020computational}), which in most cases for non-finite state spaces $\X$ is intractable (\citet[Remark 3]{wang2021sinkhorn}).
The use of a small regularisation parameter $\delta$ allows us to avoid this issue and instead approximate the solution using the Sinkhorn distance.

Due to the assertion from Proposition~\ref{prop:dual}, our main goal will be to minimise the error between
\begin{equation*}
	\sup_{\lambda >0} \bigg\{-\lambda {\varepsilon}- \lambda \delta \E_{X_1^{\PP} \sim \widehat{\PP}(x,a)}\left[\log\left(\E_{X_1^{\nu} \sim \nu}\left[\exp\left(\tfrac{-r(x,a,X_1^{\nu})-\gamma \sup_{b \in A}Q_\delta^*(X_1^{\nu},b)-\lambda \| X_1^{\PP}-X_1^{\nu}\|}{\lambda \delta}\right)\right]\right)\right]\bigg\}
\end{equation*}
and $Q_\delta^*(x,a)$ with respect to $Q_\delta^*: \X \times A \rightarrow \R$ which still is an infinite dimensional intractable optimisation problem due to our setting of continuous states.
Thus, to finally obtain a numerically tractable, finite-dimensional optimisation problem, we parameterise $Q_\delta^*(x,a)$ by fully connected feed-forward neural networks.

\subsection{\texorpdfstring{$Q$}{Q}-learning with Neural Networks}

To introduce neural networks, we follow closely the presentation in \citet{neufeld2024neural}.
By fully connected feed-forward \textit{neural networks} (or simply neural networks for brevity) with input dimension $d_{\operatorname{in}} \in \N$, output dimension $d_{\operatorname{out}} \in \N$, and number of layers $l \in \N$ we refer to functions of the form
\begin{equation} \label{eq_nn_function}
	\begin{split}
		\R^{d_{\operatorname{in}}} &\rightarrow \R^{d_{\operatorname{out}}}\\
		{x} &\mapsto {A_l} \circ {\varphi}_l \circ {A_{l-1}} \circ \cdots \circ {\varphi}_1 \circ {A_0}({x}),
	\end{split}
\end{equation}
where $({A_i})_{i=0,\dots,l}$ are affine\footnote{This means for all $i=0,\dots,l$, the function ${A_i}$ is assumed to have an affine structure of the form
	$
	{A_i}({x})={M_i} {x} + {b_i}
	$
	for some matrix ${M_i} \in \R^{ h_{i+1} \times h_{i}}$ and some vector ${b_i}\in \R^{h_{i+1}}$, where $h_0:=d_{\operatorname{in}}$ and $h_{l+1}:=d_{\operatorname{out}}$. } functions of the form
\begin{equation} \label{eq_A_i_def}
	{A_0}: \R^{d_{{\operatorname{in}}}} \rightarrow \R^{h_1},\qquad {A_i}:\R^{h_i}\rightarrow \R^{h_{i+1}}\text{ for } i =1,\dots,l-1, \text{(if } l>1), \text{ and}\qquad {A_l} : \R^{h_l} \rightarrow \R^{d_{\operatorname{out}}},
\end{equation}
and where the function $\varphi_i$ is applied componentwise, i.e., for  $i=1,\dots,l$ we have ${\varphi}_i(x_1,\dots,x_{h_i})=\left(\varphi(x_1),\dots,\varphi(x_{h_i})\right)$.  The function $\varphi:\R \rightarrow \R$  is called \emph{activation function} and assumed to be continuous and non-polynomial.
We say a neural network is \emph{deep} if $l\geq 2$.
Here ${h}=(h_1,\dots,h_{l}) \in \N^{l}$ denotes the dimensions (the number of neurons) of the hidden layers, also called \emph{hidden dimension}.

Then, we denote by $\mathfrak{N}_{d_{\operatorname{in}},{d_{\operatorname{out}}}}^{l,{h}}$  the set of all neural networks with input dimension ${d_{\operatorname{in}}}$, output dimension ${d_{\operatorname{out}}}$, $l$ hidden layers, and hidden dimension ${h}$, whereas
the set of all neural networks from $\R^{d_{\operatorname{in}}}$ to $\R^{d_{\operatorname{out}}}$ (i.e.,\ without specifying the number of hidden layers and hidden dimension) is denoted by
\begin{equation*}
	\mathfrak{N}_{d_{\operatorname{in}},{d_{\operatorname{out}}}}:=\bigcup_{l \in \N}\bigcup_{{h} \in \N^l}\mathfrak{N}_{d_{\operatorname{in}},{d_{\operatorname{out}}}}^{l,{h}}.
\end{equation*}

It is well-known that the set of neural networks possesses the so-called \textit{universal approximation property}, see, e.g., \citet{pinkus1999approximation}.

\begin{prop}[Universal approximation theorem] \label{prop:universal}
	For any compact set $\K \subset \R^{d_{\operatorname{in}}} $ the set $\mathfrak{N}_{d_{\operatorname{in}},{d_{\operatorname{out}}}}|_{\K}$ is dense in ${C}(\K,\R^{d_{\operatorname{out}}})$ with respect to the topology of uniform convergence on $C(\K,\R^{d_{\operatorname{out}}})$.
\end{prop}

Since $\X \times A \ni (x,a) \rightarrow Q_\delta^*(x,a)$ is a continuous function (see Lemma~\ref{lem:Q_continuous}), we can approximate $Q_\delta^*$ arbitrarily well on compact sets with neural networks.
Our goal will then be to solve the following optimisation problem.

\begin{opt} \label{opt:problem1}
	Given some tolerance $\rm TOL>0$, find $Q^*_{\rm NN} \in \mathfrak{N}_{d +m ,1}$ such that
	\begin{align*}
		&\left|\HH_{\delta}Q^*_{\rm NN}(x,a) -Q^*_{\rm NN}(x,a)\right|<\rm TOL
	\end{align*}
	uniformly on $\X \times A$.
\end{opt}
As we will show in the subsequent Proposition~\ref{prop:solutions}, under mild assumptions, solutions $Q^*_{\rm NN}$ of  Optimisation Problem~\ref{opt:problem1} exist if $\X$ is bounded, and they approximate the optimal $Q$ function $Q_\delta^*$ arbitrarily well.
\begin{prop} \label{prop:solutions}
	Let $\varepsilon,\delta>0$ such that $\overline{\varepsilon} \geq 0$ as defined in Proposition~\ref{prop:dual}. Moreover, let  $\X$ be bounded, and let $\rm TOL>0$ be some tolerance. Recall that $\gamma$ is the discount factor which by Assumption~\ref{asu_2}~(iii) satisfies $0 < \gamma < 1$.
	\begin{itemize}
		\item[(i)] A solution $Q^*_{\rm NN} \in \mathfrak{N}_{d + m ,1}$ of Optimisation Problem ~\ref{opt:problem1} exists for $\delta > 0$.

		\item[(ii)] Any solution $Q^*_{\rm NN} \in \mathfrak{N}_{d + m ,1}$ of Optimisation Problem~\ref{opt:problem1} for $\delta > 0$ will also satisfy \\ $\sup_{(x,a) \in \X \times A}| Q^*_{\rm NN}(x,a)-Q_\delta{^*}(x,a)| \leq \frac{ {\rm TOL}}{1-\gamma}$

		\item[(iii)] There exists some $\delta'>0$ such that for all $\delta\in (0,\delta')$  solutions $Q^*_{\rm NN} \in \mathfrak{N}_{d + m ,1}$ of Optimisation Problem~\ref{opt:problem1} w.r.t.\, $\delta$ will also satisfy $\sup_{(x,a) \in \X \times A}| Q^*_{\rm NN}(x,a)-Q_0{^*}(x,a)| \leq 2 \cdot \frac{ {\rm TOL}}{1-\gamma}$.
	\end{itemize}
\end{prop}
Proposition~\ref{prop:solutions} shows three results regarding solutions to Optimisation Problem~\ref{opt:problem1} in the case $\X$ is bounded.
Firstly, it shows the existence of a neural network which can approximate a function that will satisfy the condition of Optimisation Problem~\ref{opt:problem1}.
Note that this is not simply showing that the neural network can approximate $Q_\delta^*$, but rather that there is a neural network that can approximate a function such that applying the robust Bellman operator $\HH_\delta$ will yield a function that is close to the original function in the sense of the uniform norm for all state-action pairs from the compact set $\X \times A$.
Secondly, we show that if we find a neural network fulfilling the condition of Optimisation Problem~\ref{opt:problem1}, then it will also be close to $Q_\delta^*$ justifying the formulation of Optimisation Problem~\ref{opt:problem1}.
Thirdly, we can find a $\delta>0$ such that if we have a solution to Optimisation Problem~\ref{opt:problem1}, it will approximate $Q_0^*$, which is the optimal $Q$ function for the Wasserstein case.

The results rely on Proposition~\ref{prop:universal} which ensures the neural network can approximate the $Q_\delta^*$ arbitrarily well and hence the condition for a bounded $\X$.
However, it is not necessarily the case that a neural network cannot approximate $Q_\delta^*$ for unbounded $\X$ as this is an area of ongoing research (see, e.g., \citet{neufeld2024universal}).

\subsection{The Robust DQN Algorithm} \label{sec:robustdqn}

In the following, to emphasise the dependence of neural networks on its parameters, i.e., its weights and biases, we write $Q^*_{\rm NN}(\theta; x,a)= Q^*_{\rm NN}(x,a)$ for a neural network $Q^*_{\rm NN} \in \mathfrak{N}_{d + m ,1}$ evaluated at a state action pair $(x,a) \in \X \times A$, given the parameter $\theta \in \Theta$, where $\Theta$ denotes the Euclidean set of all possible parameters of the neural network, i.e., the possible choices of weights and biases.
Motivated by Proposition~\ref{prop:dual}, our goal of solving Optimisation Problem~\ref{opt:problem1} now becomes minimising the following loss function
\begin{equation} \label{eq:Loss}
	L(\theta;(x,a)):= \left(\HH_{\delta}Q^*_{\rm NN}(\theta;x,a) -Q^*_{\rm NN}(\theta;x,a)\right)^2
\end{equation}
for any state action pair $(x,a) \in \X \times A$, with respect to $\theta\in \Theta$.
Given a batch of states $x^B:=(x_i)_{i=1,\dots,B} \subset \X$ and actions $a^B:=(a_i)_{i=1,\dots,B} \subset A$, we obtain the loss function on the batch as
\begin{equation} \label{eq:Loss_Batch}
	L^B(\theta;(x^B,a^B)):= \frac{1}{B}\sum_{i=1}^B\left(\HH_{\delta}Q^*_{\rm NN}(\theta;x_i,a_i) -Q^*_{\rm NN}(\theta;x_i,a_i)\right)^2
\end{equation}

Our setting of continuous state space and finite action space is the same as in \citet{mnih2015human} where the authors propose the Deep Q-Network (DQN) algorithm to learn the optimal $Q$ function.
We show how the DQN algorithm can be modified to the robust case in Algorithm ~\ref{algo:robustdqn}.
The DQN algorithm uses an experience replay buffer to store transitions $(x_t,a_t,r_t,x_{t+1})$, where $r_t=r(x_t,a_t,x_{t+1})$, and samples a batch of transitions to compute targets that are then used to update the neural network approximating the $Q$ function.
For the Robust DQN (RDQN) algorithm, we store transitions obtained from interacting with the environment which is assumed to be following the reference distribution $\widehat{\PP}$.
The targets in the DQN algorithm are also computed using a target network, which is a copy of the neural network approximating the $Q$ function.
The target network is updated less frequently than the neural network approximating the $Q$ function for better stability as described in \citet[Training algorithm for deep Q-networks]{mnih2015human}.

{
The main modification to DQN is to use Proposition~\ref{prop:dual} to calculate the targets $\HH_{\delta}Q^*_{\rm NN}$ to update $Q^*_{\rm NN}$.
From the replay buffer, only the state $x_t$, the action $a_t$ and the next state $x_{t+1}$ are needed to calculate the modified target.
In particular, the outer expectation in Proposition~\ref{prop:dual} will be approximated by the single unbiased sample of the next state $x_{t+1}$ and the action $a_t$.
The inner expectation can be approximated by sampling multiple times from the distribution $\nu$.

\begin{rem}[Reward evaluation in the RDQN target]
\label{rem:reward_from_nu}
    The reward appearing in Proposition~\ref{prop:dual} is
    $r(x,a,X_1^\nu)$, evaluated at the candidate next state $X_1^\nu$
    rather than the realised next state.
    Accordingly, the RDQN target evaluates the reward function
    at each of the $N_\nu$ candidates used to approximate
    the expectation under $\nu$; the observed reward $r_t$
    stored in the replay buffer is not used directly.
    The algorithm therefore requires access to the reward
    function at candidate next states, whereas standard DQN
    uses only observed rewards \citep{mnih2015human}.
    A closed-form expression for $r$ is not necessary,
    provided these evaluations can be computed.
    In the dual expression, the realised next state
    $x_1^{\PP}$ enters only through the transport cost
    $\|x_1^{\PP}-X_1^\nu\|$.
\end{rem}

This modification leads then to the numerical routine summarised in Algorithm ~\ref{algo:robustdqn} for computing the target and updating the neural network.
We omit the details of the agent interacting with the environment, the experience replay buffer, the updating of the target network, and refer the reader to \citet[Algorithm 1]{mnih2015human} for more details.
To be clear, Algorithm \ref{algo:robustdqn} replaces the sampling of the minibatch and the gradient step of \citet[Algorithm 1]{mnih2015human}.

\begin{algorithm}[h!]
    
	\renewcommand\mycommfont[1]{\footnotesize\ttfamily\textcolor{black}{#1}}
	\SetAlgoLined
	\SetNoFillComment
	\SetKwInOut{Input}{Input}
	\SetKwInOut{Output}{Output}
	\SetKwFor{While}{while}{do}{endwhile}
	\SetKwFunction{NuCand}{NuCandidates}

	\Input{
	State space $\mathcal{X} \subset \R^d$;
	Control space $A \subset \R^m$;
	Reward function $r$;
	Discount Factor $\gamma \in (0,1)$;
	Sinkhorn radius $\varepsilon>0$;
	Parameter $\delta>0$ for the Sinkhorn regularisation;
	Number of gradient steps per update $N_{Q}$;
	Number of candidates to sample from $\nu$ per transition $N_{\nu}$;
	Sampling distribution $\nu$ and its proposal mixture (Section~\ref{sec:sampling_from_nu});
	Batch size $B$;
	Newton tolerance ${\rm tol}_\lambda$ and iteration cap $N_\lambda$;
	$Q$ function neural network $Q_{\rm NN }(\theta_0;\cdot, \cdot )$ with current parameters $\theta_0$;
	$Q$ function target network $Q_{\rm NN }(\theta_{\text{target}};\cdot, \cdot )$ with current parameters $\theta_{\text{target}}$;
	}

	\For{$m =1,\dots,N_{Q}$}{

	Sample a batch of transitions $(x_i,a_i,x^{\PP}_i)$ for $i=1,\dots,B$ where $x_i$ is the current state, $x^{\PP}_i$ is the next state and $a_i$ is the action taken, together with their cached dual variables $\log \lambda_i^{\text{cache}}$;
	\tcp{(See Section \ref{sec:optimisation_of_lambda})}

	$\left(x_{i,j}^{\nu}, w_{i,j}\right)_{j=1,\dots,N_\nu} \leftarrow$ \NuCand{$x_i^{\PP}$, $\delta$} for $i=1,\dots,B$;
	\tcp{(See Section \ref{sec:sampling_from_nu})}

	Compute the transport costs $c_{i,j} = \|x_i^{\PP}-x_{i,j}^{\nu}\|$ and the candidate values
	$
	Z_{i,j} = -r(x_i,a_i,x_{i,j}^{\nu})-\gamma \sup_{b \in A}Q_{{\rm NN} }(\theta_{\text{target}};x_{i,j}^{\nu},b)
	$;

	Compute
	$
	\bar{\varepsilon}_i = \varepsilon + \delta \log\left(\frac{1}{N_{\nu}} \sum_{j=1}^{N_{\nu}} w_{i,j}\, e^{-c_{i,j}/\delta}\right)
	$
	and set $\mathcal{I} = \{i : \bar{\varepsilon}_i > 0\}$;
	\tcp{(See Section \ref{sec:warning})}

	\For{$i \in \mathcal{I}$}{
	Maximise, by safeguarded Newton iterations on $\log \lambda_i$ started at $\log \lambda_i^{\text{cache}}$ and run to tolerance ${\rm tol}_\lambda$ or at most $N_\lambda$ iterations,
	\begin{equation*}
		\widehat{\mathcal{H}_\delta} Q_{{\rm NN} } (\lambda_i, \theta_{\text{target}}; x_i, a_i) =
		-\lambda_i \varepsilon
		- \lambda_i \delta
		\left[
		C_{i} + \log
		\left(
		\frac{1}{N_{\nu}} \sum_{j=1}^{N_{\nu}} e^{\left(E_{i,j} - C_{i}\right)}
		\right)
		\right],
	\end{equation*}
	where $E_{i,j} = \frac{Z_{i,j}}{\lambda_i \delta} - \frac{c_{i,j}}{\delta} + \log w_{i,j}$ and $C_{i} = \max_{j} E_{i,j}$;
	\tcp{(See Sections \ref{sec:optimisation_of_lambda} and \ref{sec:exponential_term})}
	}

	Write the optimised $\log \lambda_i$, $i \in \mathcal{I}$, back to the cache;

	Take a gradient step on
	$
	\frac{1}{|\mathcal{I}|} \sum_{i \in \mathcal{I}} \left(\widehat{\mathcal{H}_\delta} Q_{{\rm NN} }(\lambda_i, \theta_{\text{target}}; x_i, a_i)-Q_{\rm NN }(\theta_{m-1};x_i,a_i)\right)^2
	$ w.r.t.\,$\theta_{m-1}$ to minimise the value and update parameters to $\theta_{m}$;
	}
	\caption{Robust DQN}\label{algo:robustdqn}
\end{algorithm}

{
\begin{rem}[One-sample approximation of the dual objective]\label{rem:sampled_multiplier}
Fix a bounded continuous critic $Q$, and write
\[
    F_Q(x,a,y):=r(x,a,y)+\gamma\sup_{b\in A}Q(y,b).
\]
With the inner expectation evaluated exactly, define
\[
    h_Q(\lambda,z;x,a):=-\lambda\varepsilon
    -\lambda\delta\log\E_{Y\sim\nu}
    \left[\exp\left(
        -\frac{F_Q(x,a,Y)}{\lambda\delta}
        -\frac{\|z-Y\|}{\delta}
    \right)\right].
\]
Under the conditions of Proposition~\ref{prop:dual},
\[
    (\HH_\delta Q)(x,a)
    =\sup_{\lambda>0}\E_{Z\sim\widehat{\PP}(x,a)}
    [h_Q(\lambda,Z;x,a)].
\]
Algorithm~\ref{algo:robustdqn} approximates this outer expectation
by one nominal next-state sample and maximises the resulting
Monte Carlo estimate. For each fixed $\lambda$, this estimate
is unbiased, but maximising it generally introduces an upward bias.
Indeed, with exact inner integration and per-transition
maximisation, and without transition exclusions, the expected
target is
\[
    (\widetilde{\HH}_\delta Q)(x,a)
    :=\E_{Z\sim\widehat{\PP}(x,a)}
    \left[\sup_{\lambda>0}h_Q(\lambda,Z;x,a)\right]
    \geq(\HH_\delta Q)(x,a).
\]
When both values are finite, define this bias by
\[
    b_Q(x,a):=
    (\widetilde{\HH}_\delta Q)(x,a)-(\HH_\delta Q)(x,a).
\]
Averaging more independently maximised one-sample targets,
increasing the number of candidates from $\nu$, or improving
the per-transition optimisation does not generally remove this bias.
It vanishes, for example, when a common multiplier maximises
$h_Q(\lambda,Z;x,a)$ for $\widehat{\PP}(x,a)$-almost every $Z$.

To approximate the intended expected objective, one can instead
optimise a common multiplier over several nominal next-state
samples, or learn a multiplier for each state--action pair through
successive stochastic updates using fresh nominal samples.
{The same issue arises in the distributionally robust
$Q$-learning algorithm of \citet{liu2022distributionally}, which
approximates the dual value with a multilevel Monte Carlo scheme
built from several next-state samples drawn from a simulator for
the same state--action pair at each update.
We use a single nominal next-state sample because this keeps the
algorithm applicable to settings in which only the realised
transition is observed, so that counterfactual next states for the
same state--action pair cannot be generated.
Appendix~\ref{app:sup_bias} estimates $b_Q$ in the S\&P~500
experiment of Section~\ref{sec:port_opt} and finds it small relative
to the robust penalty, in particular for long positions.}
For the present algorithm, the fitted-iteration bounds in
Section~\ref{sec:approx_fitted_q_iteration} must account for
$b_Q$ in addition to the numerical errors.
The per-transition supremum need not be finite; feasibility
checks and the bias from excluding transitions are discussed
in Section~\ref{sec:warning}.
\end{rem}
}

\subsubsection{Sampling from \texorpdfstring{$\nu$}{nu}}
\label{sec:sampling_from_nu}

The inner expectation in Proposition~\ref{prop:dual} is taken
under the chosen measure $\nu$. {Throughout this subsection, all densities are defined with respect to a common $\sigma$-finite dominating measure $\rho$ on $\X$. In particular,
$p_\nu:=\D\nu/\D\rho$, and $p_\lambda(\,\cdot\mid x^{\PP})$ denotes the density with respect to $\rho$ of the associated conditional law given the nominal next state $x^{\PP}$.}
For fixed $(x,a)$ and $\lambda>0$,} {by \citet{wang2021sinkhorn}, this density is}
\[
{
    p_\lambda(y\mid x^{\PP})
    :=
    \frac{
        p_\nu(y)
        e^{-\|x^{\PP}-y\|/\delta}
        e^{Z(y)/(\lambda\delta)}
    }{
        \displaystyle\int_{\X}
        p_\nu(z)
        e^{-\|x^{\PP}-z\|/\delta}
        e^{Z(z)/(\lambda\delta)}
        \D\rho(z)
    },}
    \qquad
    \text{ with } Z(y):=-r(x,a,y)-\gamma\sup_{b\in A}Q_\delta^*(y,b).
\]
{
For small $\delta$, the transport factor
$e^{-\|x^{\PP}-y\|/\delta}$ concentrates near $x^{\PP}$.
Direct sampling from $\nu$ may miss this region, while large
values of $Z$ can make distant candidates important as well.

We therefore use importance sampling. For transition $i$, let
$q_i$ be a proposal density satisfying $q_i>0$ $\nu$-almost
everywhere. Candidates $x^\nu_{i,j}$ receive weights
\begin{equation}\label{eq:mis_weight}
    w_{i,j}:=\frac{p_\nu(x^\nu_{i,j})}{q_i(x^\nu_{i,j})},
\end{equation}
and expectations under $\nu$ are approximated by
\begin{equation}\label{eq:mis_estimator}
    \widehat{\E}^{(i)}_\nu[f]
    :=\frac{1}{N_\nu}\sum_{j=1}^{N_\nu}
    w_{i,j}f(x^\nu_{i,j}).
\end{equation}
This estimator is unbiased for random draws from $q_i$,
provided $f$ is integrable, and reduces to the usual sample
average when $q_i=p_\nu$. The deterministic candidate sets
described below instead give numerical integration rules.

In the financial applications, we combine three groups of
candidates, using the mixture weights of
\citet{veach1995optimally}:
\begin{equation}\label{eq:proposal_mixture}
    q_i
    =\frac{N_{\rm fix}}{N_\nu}q^{\rm fix}
    +\frac{N_{\rm loc}}{N_\nu}q_i^{\rm loc}
    +\frac{N_{\rm far}}{N_\nu}q^{\rm far},
    \qquad
    N_\nu=N_{\rm fix}+N_{\rm loc}+N_{\rm far}.
\end{equation}
Here $d$ denotes the dimension of the variable sampled under $\nu$.

\emph{Fixed candidates} cover the distribution $\nu$ and are
reused across transitions. For $d=1$, we use the quantiles
$F_\nu^{-1}(k/(N_{\rm fix}+1))$, $k=1,\ldots,N_{\rm fix}$.
For $d>1$, we transform a scrambled Sobol' sequence
\citet{owen1998scrambling} using the elliptical representation
of $\nu$, pairing points with their reflections about its
location. The corresponding density in
\eqref{eq:proposal_mixture} is $q^{\rm fix}=p_\nu$.

\emph{Local candidates} cover the region near $x_i^{\PP}$.
We use
\[
    x^\nu_{i,j}=x_i^{\PP}+\delta R_j\Theta_j,
    \qquad
    R_j\sim\Gamma(d,1),
\]
where $\Theta_j$ is independent of $R_j$ and uniform on the
unit sphere $S^{d-1}$. This gives the density
\begin{equation}\label{eq:local_proposal}
    q_i^{\rm loc}(y)
    =
    \frac{e^{-\|y-x_i^{\PP}\|/\delta}}
    {\Gamma(d)\delta^d|S^{d-1}|}.
\end{equation}
Thus, the proposal matches the normalised transport factor
and has mean distance $d\delta$ from $x_i^{\PP}$.
We generate these candidates using scrambled Sobol' points
with paired opposite directions.

\emph{Tail candidates} are added when extreme returns can
strongly affect the objective. For example, the log return
of a fully short position can become very negative as the
asset price approaches a doubling. To cover such outcomes,
we divide the tails into intervals $B_k$ using a geometric
spacing of tail probabilities. Each interval has mass
$m_k=\nu(B_k)$ and is represented by its conditional mean.
The corresponding mixture density is
\[
    q^{\rm far}(y)
    =\frac{1}{N_{\rm far}}
    \sum_{k=1}^{N_{\rm far}}
    \frac{p_\nu(y)\mathbf{1}_{B_k}(y)}{m_k}.
\]
Using conditional means rather than random draws introduces
a numerical integration error.

In the financial experiments, $\nu$ is a Student-$t$
distribution restricted to a bounded box and fitted to the
training data.
For $d=1$, we fix the degrees of freedom and choose the
location and scale so that the truncated distribution
matches the sample mean and variance.
For $d>1$, we fit the location, scale, correlation and degrees
of freedom by maximum likelihood, using a common location
parameter for all assets.
}
The truncation bounds are chosen so that portfolio gross
returns remain bounded away from zero, ensuring bounded
log rewards as required by Assumption~\ref{asu_2}(i).

\subsubsection{Feasibility of the sampled dual problem}
\label{sec:warning}

Proposition~\ref{prop:dual} requires
$\bar{\varepsilon}\ge0$. Its transition-wise estimate is
\[
    \bar{\varepsilon}_i
    =
    \varepsilon+\delta\log\left(
        \frac{1}{N_\nu}\sum_{j=1}^{N_\nu}
        w_{i,j}e^{-\|x_i^{\PP}-x^\nu_{i,j}\|/\delta}
    \right).
\]
This estimate can be negative even when the condition in
the proposition holds, because we use one nominal next
state and finitely many candidates. If
$\bar{\varepsilon}_i<0$, the sampled dual objective tends
to $+\infty$ as $\lambda\to\infty$ and cannot provide a
finite target. This concerns the sampled problem, rather
than the validity of the exact duality.

We use two checks. Before training, we evaluate
$\bar{\varepsilon}_i$ on a large sample of nominal transitions
and proceed only if more than $99.5\%$ are positive.
During training, we exclude transitions with
$\bar{\varepsilon}_i\le0$ and skip the critic update if none
remain.

Excluding transitions introduces selection bias, particularly
if adverse outcomes are more often excluded. The preliminary
check limits their overall frequency but does not, by itself,
bound the resulting bias.

\subsubsection{Optimisation of \texorpdfstring{$\lambda$}{lambda}}
\label{sec:cache_lambda}
\label{sec:optimisation_of_lambda}

For fixed $(x,a)$, define
\[
    S(z,\lambda)
    :=
    \E_{Y\sim\nu}\left[
        \exp\left(
            \frac{Z(Y)-\lambda\|z-Y\|}{\lambda\delta}
        \right)
    \right],
    \qquad
    \frac{dP_{z,\lambda}}{d\nu}(y)
    :=
    \frac{
        \exp\left(
            \frac{Z(y)-\lambda\|z-y\|}{\lambda\delta}
        \right)
    }{S(z,\lambda)}.
\]
The dual objective and its derivatives are
\begin{align}
    g(\lambda)
    &=-\lambda\varepsilon
      -\lambda\delta
      \E_{X\sim\widehat{\PP}(x,a)}[\log S(X,\lambda)],
      \qquad \lambda>0,\\
    g'(\lambda)
    &=-\varepsilon
      -\delta\E_{X\sim\widehat{\PP}(x,a)}
      [\log S(X,\lambda)]
      +\frac{1}{\lambda}
      \E_{X\sim\widehat{\PP}(x,a)}
      \left[\E_{Y\sim P_{X,\lambda}}[Z(Y)]\right],\\
    g''(\lambda)
    &=-\frac{1}{\lambda^3\delta}
      \E_{X\sim\widehat{\PP}(x,a)}
      \left[\operatorname{Var}_{Y\sim P_{X,\lambda}}
      (Z(Y))\right]\le0.
\end{align}
Hence $g$ is concave in $\lambda$. The sampled derivatives
use the same exponential weights as the objective and
require little additional computation.

In Algorithm~\ref{algo:robustdqn}, we optimise the sampled
objective for each retained transition using safeguarded
Newton iterations, with a stopping tolerance and an
iteration limit. We parameterise $\lambda=e^u$ to ensure
positivity and store the resulting $u$ with the transition
for use as the next starting value. Concavity in $\lambda$
does not imply concavity in $u$.

\subsubsection{Stable evaluation of the exponential term}
\label{sec:exponential_term}

Small $\delta$ or $\lambda$ can produce large exponents.
For each transition, we therefore evaluate the sampled
logarithm using
\[
    \log\left(\frac{1}{N_\nu}\sum_j e^{E_{i,j}}\right)
    =
    C_i+\log\left(
        \frac{1}{N_\nu}\sum_j e^{E_{i,j}-C_i}
    \right),
    \qquad C_i=\max_j E_{i,j},
\]
where
\[
    E_{i,j}
    =
    \frac{Z(x^\nu_{i,j})}{\lambda\delta}
    -\frac{\|x_i^{\PP}-x^\nu_{i,j}\|}{\delta}
    +\log w_{i,j}.
\]
The shifted exponents are non-positive, preventing overflow.

{
\subsubsection{Approximation error bounds}
\label{sec:approx_fitted_q_iteration}

RDQN approximates the Bellman update using a neural network,
sampled targets and numerical optimisation. To describe how
one-step errors accumulate, consider a fitted iteration
\begin{equation}\label{eq:approx_rdqn_update}
    Q_{k+1}=\HH_\delta Q_k+e_k,
    \qquad \norm{e_k}_\infty\le\eta_k.
\end{equation}

\begin{thm}[Approximate robust fitted $Q$-iteration]
\label{thm:approx_robust_fitted_q_iteration}
Suppose $\HH_\delta$ is a $\gamma$-contraction with fixed
point $Q_\delta^\ast$, and
\begin{equation}\label{eq:approx_step_error_main}
    \norm{Q_{k+1}-\HH_\delta Q_k}_\infty\le\eta_k,
    \qquad k=0,1,\ldots.
\end{equation}
Then, for every $K\ge1$,
\begin{equation}\label{eq:approx_fitted_q_bound_main}
    \norm{Q_K-Q_\delta^\ast}_\infty
    \le
    \gamma^K\norm{Q_0-Q_\delta^\ast}_\infty
    +\sum_{k=0}^{K-1}\gamma^{K-1-k}\eta_k.
\end{equation}
If $\eta_k\le\eta$ for all $k$, then
\begin{equation}\label{eq:uniform_eta_bound_main}
    \norm{Q_K-Q_\delta^\ast}_\infty
    \le
    \gamma^K\norm{Q_0-Q_\delta^\ast}_\infty
    +\frac{1-\gamma^K}{1-\gamma}\eta
    \le
    \gamma^K\norm{Q_0-Q_\delta^\ast}_\infty
    +\frac{\eta}{1-\gamma}.
\end{equation}
\end{thm}

When separate uniform error bounds are available, we may use
\begin{equation}\label{eq:rdqn_error_decomp_main}
    \eta_k\le
    \eta_{\mathrm{fit},k}
    +\eta_{\mathrm{MC},k}
    +\eta_{\lambda,k}
    +\eta_{\mathrm{opt},k}
    +\eta_{\mathrm{struct},k}.
\end{equation}
These terms account for function approximation, target
estimation, the numerical dual solve, and incomplete
regression optimisation, respectively.
Here $\eta_{\mathrm{MC},k}$ must include both sampling and
numerical integration errors, as well as bias from excluding
transitions.
The separate term $\eta_{\mathrm{struct},k}$ bounds the
operator discrepancy in Remark~\ref{rem:sampled_multiplier}
for the critic used to construct the targets; it is not a
Monte Carlo or numerical optimisation error and need not
vanish as those errors decrease.

\begin{cor}[RDQN error decomposition]
\label{cor:rdqn_error_decomposition}
Suppose \eqref{eq:approx_step_error_main} and
\eqref{eq:rdqn_error_decomp_main} hold, with
$\eta_{j,k}\le\eta_j$ for all $k$ and
$j\in\{\mathrm{fit},\mathrm{MC},\lambda,\mathrm{opt},\mathrm{struct}\}$.
Then, for every $K\ge1$,
\begin{equation}\label{eq:rdqn_error_bound_main}
    \norm{Q_K-Q_\delta^\ast}_\infty
    \le
    \gamma^K\norm{Q_0-Q_\delta^\ast}_\infty
    +\frac{
        \eta_{\mathrm{fit}}+\eta_{\mathrm{MC}}
        +\eta_\lambda+\eta_{\mathrm{opt}}+\eta_{\mathrm{struct}}
    }{1-\gamma}.
\end{equation}
\end{cor}

These bounds describe the effect of controlled update errors;
they do not establish that RDQN satisfies the required
uniform bounds.
}

{
\subsection{Extensions to Compact Continuous Action Sets}\label{sec:continuous_action_extension}

We consider two extensions of the robust Bellman operator to a nonempty
compact action set $A\subset\R^m$: action discretisation and
entropy-regularised action selection.
Throughout this subsection, Assumptions~\ref{asu_1}, \ref{asu_2}
and~\ref{asu_3} are imposed on this enlarged action set.
The ambiguity correspondence is therefore nonempty, weakly compact and
continuous, by the arguments of Lemmas~\ref{lem:weak_compact},
\ref{lem:uhc_sinkhorn_ball} and~\ref{lem:lhc_sinkhorn_ball}.
Write $\calQ:=C_b(\X\times A)$, equipped with the supremum norm,
and take all critics in this complete space.
Compactness of $A$ makes $M_Q(x):=\max_{b\in A}Q(x,b)$ continuous
by Berge's maximum theorem.
The finite-grid and soft continuation values defined below are also
continuous and bounded.
Consequently, the continuity argument of Lemma~\ref{lem:Q_continuous}
applies to each corresponding robust operator.
Non-expansiveness of the continuation maps makes these operators
$\gamma$-contractions on $\calQ$, so their fixed points exist uniquely.

These are guarantees for the stated operators.
Numerical implementations additionally require controlled sampling,
regression and optimisation errors, including the multiplier
discrepancy in Remark~\ref{rem:sampled_multiplier}.

\subsubsection{Action discretisation with an explicit error bound}\label{sec:continuous_action_grid}

Let $A_h\subset A$ be a finite $h$-net, meaning that for every $a\in A$ there exists $a_h\in A_h$ such that
\begin{equation*}
    \norm{a-a_h}\leq h .
\end{equation*}
For $Q\in\calQ$, define for all $x\in \X$
\begin{equation*}
    M_Q(x):=\sup_{b\in A}Q(x,b),
    \qquad
    M_{h,Q}(x):=\max_{b\in A_h}Q(x,b).
\end{equation*}
The discretised robust Bellman operator is
\begin{equation}\label{eq:H_delta_h_main}
    \X \times A \ni (x,a) \mapsto (\HH_{\delta,h}Q)(x,a)
    :=
    \inf_{\PP\in\Sball\left(\widehat{\PP}(x,a)\right)}
    \E_{\PP}\left[
        r(x,a,X_1)+
        \gamma M_{h,Q}(X_1)
    \right].
\end{equation}
Here only the continuation maximisation is discretised.
Keeping the current action $a$ in $A$ places both operators on the
same function space; restricting their grid fixed point to
$\X\times A_h$ gives the optimal critic for the grid-constrained problem.
Neither $A_h$ nor this construction alters the state-transition
ambiguity sets.

\begin{lem}[Action-grid approximation of the maximisation]\label{lem:action_grid_max_main}
Assume that $Q$ is Lipschitz in the action variable with constant $L_A$, i.e.,
\begin{equation*}
    |Q(x,a)-Q(x,a')|\leq L_A\norm{a-a'}
    \qquad\text{for all }x\in \X,\; a,a'\in A.
\end{equation*}
Then, for all $x\in \X$,
\begin{equation}\label{eq:grid_max_bound_main}
    0
    \leq
    M_Q(x)-M_{h,Q}(x)
    \leq
    L_A h .
\end{equation}
Consequently,
\begin{equation}\label{eq:H_grid_bound_main}
    \norm{\HH_\delta Q-\HH_{\delta,h}Q}_\infty
    \leq
    \gamma L_A h .
\end{equation}
\end{lem}

\begin{proof}
Since $A_h\subset A$, we have $M_{h,Q}(x)\leq M_Q(x)$.  For the reverse approximation bound, fix $x\in\X$ and $\rho>0$. Choose $b_\rho\in A$ such that $Q(x,b_\rho)\geq M_Q(x)-\rho$. Since $A_h$ is an $h$-net, there exists $b_h\in A_h$ with $\norm{b_\rho-b_h}\leq h$. By Lipschitz continuity,
\begin{equation*}
    Q(x,b_h)
    \geq
    Q(x,b_\rho)-L_A h
    \geq
    M_Q(x)-\rho-L_A h .
\end{equation*}
Because $M_{h,Q}(x)\geq Q(x,b_h)$, we have $M_Q(x)-M_{h,Q}(x)\leq \rho+L_Ah$. Letting $\rho\downarrow0$ proves \eqref{eq:grid_max_bound_main}.

For \eqref{eq:H_grid_bound_main}, fix $(x,a)$. For any $\PP\in\Sball(\widehat{\PP}(x,a))$, the expectations defining $\HH_\delta Q$ and $\HH_{\delta,h}Q$ differ only in the continuation value.  By \eqref{eq:grid_max_bound_main}, this difference is between $0$ and $\gamma L_Ah$. Taking infima over the same ambiguity set preserves this upper bound, and taking the supremum over $(x,a)$ gives \eqref{eq:H_grid_bound_main}.
\end{proof}

\begin{thm}[Continuous actions by discretisation]\label{thm:continuous_actions_grid_main}
Let $Q_\delta^\ast$ and $Q_{\delta,h}^\ast$ be the unique fixed points
of $\HH_\delta$ and $\HH_{\delta,h}$, respectively.
Assume that $Q_\delta^\ast$ is uniformly Lipschitz in the action variable
with constant $L_A$.
Then $Q_{\delta,h}^\ast\leq Q_\delta^\ast$ pointwise, and
\begin{equation}\label{eq:grid_fixed_point_bound_main}
    \norm{Q_\delta^\ast-Q_{\delta,h}^\ast}_\infty
    \leq
    \frac{\gamma L_A h}{1-\gamma} .
\end{equation}
\end{thm}

\begin{proof}
Using the fixed-point identities and Lemma~\ref{lem:action_grid_max_main}
applied to $Q_\delta^\ast$,
\begin{align*}
    \norm{Q_\delta^\ast-Q_{\delta,h}^\ast}_\infty
    &\leq
    \norm{\HH_\delta Q_\delta^\ast-\HH_{\delta,h}Q_\delta^\ast}_\infty
    +
    \norm{\HH_{\delta,h}Q_\delta^\ast-\HH_{\delta,h}Q_{\delta,h}^\ast}_\infty \\
    &\leq
    \gamma L_A h
    +\gamma\norm{Q_\delta^\ast-Q_{\delta,h}^\ast}_\infty .
\end{align*}
Rearranging gives \eqref{eq:grid_fixed_point_bound_main}.
Since $\HH_{\delta,h}Q_\delta^\ast\leq\HH_\delta Q_\delta^\ast
=Q_\delta^\ast$ and $\HH_{\delta,h}$ is monotone, its iterates
started from $Q_\delta^\ast$ remain below $Q_\delta^\ast$.
Their uniform limit is $Q_{\delta,h}^\ast$, proving the pointwise order.
\end{proof}

\begin{cor}[Approximate fitted iteration with action discretisation]\label{cor:approx_grid_main}
Suppose $Q_k\in\calQ$ and the discretised fitted updates satisfy
\begin{equation*}
    \norm{Q_{k+1}-\HH_{\delta,h}Q_k}_\infty\leq \eta_k .
\end{equation*}
Under the assumptions of Theorem~\ref{thm:continuous_actions_grid_main},
\begin{equation}\label{eq:approx_grid_main}
    \norm{Q_K-Q_\delta^\ast}_\infty
    \leq
    \gamma^K\norm{Q_0-Q_\delta^\ast}_\infty
    +
    \sum_{k=0}^{K-1}\gamma^{K-1-k}\eta_k
    +
    \frac{1-\gamma^K}{1-\gamma}\gamma L_Ah .
\end{equation}
\end{cor}

\begin{proof}
Let $d_k:=\norm{Q_k-Q_\delta^\ast}_\infty$.
The contraction property and the grid error at $Q_\delta^\ast$ give
\begin{equation*}
    d_{k+1}\leq\eta_k+\gamma d_k+\gamma L_Ah.
\end{equation*}
Iterating this inequality gives \eqref{eq:approx_grid_main}.
\end{proof}

The constant $L_A$ is an additional regularity assumption on the
optimal critic; continuity alone does not give a Lipschitz rate.
More generally, all the grid bounds hold with $L_Ah$ replaced by the
uniform action modulus
\[
    \omega_\ast(h):=
    \sup_{\substack{x\in\X,\ a,b\in A\\\norm{a-b}\leq h}}
    |Q_\delta^\ast(x,a)-Q_\delta^\ast(x,b)|.
\]
Thus convergence follows whenever $\omega_\ast(h)\to0$, in particular
when $\X$ is compact.
Action grids are practical in low dimensions, but their size can grow
rapidly with the action dimension.
The bound on $\eta_k$ remains an assumption on the numerical update;
discretisation alone does not provide replay coverage or uniform
regression accuracy.

\subsubsection{Entropy-regularised continuous actions}\label{sec:continuous_action_entropy}

For higher-dimensional action spaces, a smoother alternative is the
soft action value
\begin{equation}\label{eq:soft_value_main}
    V_{\tau,Q}(x)
    :=
    \tau\log\int_A \exp\left(\frac{Q(x,b)}{\tau}\right)\mu_A(\D b),
    \qquad \tau>0,
\end{equation}
where $\mu_A \in \mathcal{M}_1(A)$ is a fixed probability reference
measure with full support on $A$.
The action temperature $\tau$ is distinct from the transport
regularisation parameter $\delta$.
Define the entropy-regularised robust Bellman operator
\begin{equation}\label{eq:soft_H_main}
    (\HH_{\delta,\tau}Q)(x,a)
    :=
    \inf_{\PP\in\Sball\left(\widehat{\PP}(x,a)\right)}
    \E_{\PP}\left[
        r(x,a,X_1)+
        \gamma V_{\tau,Q}(X_1)
    \right].
\end{equation}

\begin{lem}[Soft maximum is non-expansive]\label{lem:soft_lipschitz_main}
For any $Q_1,Q_2\in\calQ$,
\begin{equation}\label{eq:soft_lipschitz_main}
    \norm{V_{\tau,Q_1}-V_{\tau,Q_2}}_\infty
    \leq
    \norm{Q_1-Q_2}_\infty .
\end{equation}
\end{lem}

\begin{proof}
Let $c:=\norm{Q_1-Q_2}_\infty$. Then $Q_1(x,b)\leq Q_2(x,b)+c$ for all $(x,b)$, and hence
\begin{equation*}
    \int_A \exp\left(\frac{Q_1(x,b)}{\tau}\right)\mu_A(\D b)
    \leq
    e^{c/\tau}
    \int_A \exp\left(\frac{Q_2(x,b)}{\tau}\right)\mu_A(\D b).
\end{equation*}
Taking logarithms and multiplying by $\tau$ gives $V_{\tau,Q_1}(x)\leq V_{\tau,Q_2}(x)+c$. Interchanging $Q_1$ and $Q_2$ proves \eqref{eq:soft_lipschitz_main}.
\end{proof}

\begin{lem}[Contraction of the entropy-regularised robust operator]\label{lem:soft_contraction_main}
For all $Q_1,Q_2\in\calQ$,
\begin{equation}\label{eq:soft_contraction_main}
    \norm{\HH_{\delta,\tau}Q_1-\HH_{\delta,\tau}Q_2}_\infty
    \leq
    \gamma\norm{Q_1-Q_2}_\infty .
\end{equation}
\end{lem}

\begin{proof}
Lemma~\ref{lem:soft_lipschitz_main} gives
\begin{equation*}
    |V_{\tau,Q_1}(x)-V_{\tau,Q_2}(x)|
    \leq
    \norm{Q_1-Q_2}_\infty,
\end{equation*}
and the same infimum inequality over the common ambiguity set $\Sball(\widehat{\PP}(x,a))$ yields the stated $\gamma$-contraction bound.
\end{proof}

\begin{thm}[Convergence of entropy-regularised robust $Q$-iteration]\label{thm:soft_convergence_main}
For $\tau>0$ and any $Q_0\in\calQ$, define
\begin{equation*}
    Q_{k+1}=\HH_{\delta,\tau}Q_k .
\end{equation*}
Then $\HH_{\delta,\tau}$ has a unique fixed point
$Q_{\delta,\tau}^\ast\in\calQ$, and $Q_k$ converges uniformly to it.
Moreover,
\begin{equation*}
    \norm{Q_k-Q_{\delta,\tau}^\ast}_\infty
    \leq
    \gamma^k\norm{Q_0-Q_{\delta,\tau}^\ast}_\infty .
\end{equation*}
\end{thm}

\begin{proof}
By Lemma~\ref{lem:soft_contraction_main}, $\HH_{\delta,\tau}$ is a $\gamma$-contraction on the complete space $\calQ$. Banach's fixed-point theorem gives existence and uniqueness of $Q_{\delta,\tau}^\ast$, convergence of the Picard iteration, and the stated geometric error bound.
\end{proof}

\begin{prop}[Soft-to-hard approximation]\label{prop:soft_hard_main}
For $Q\in\calQ$, the soft values converge pointwise:
\begin{equation*}
    V_{\tau,Q}(x)
    \longrightarrow
    \sup_{b\in A}Q(x,b)
    \qquad\text{as }\tau\downarrow0
\end{equation*}
for each $x\in\X$.
For $\rho>0$, define
\[
    \omega_Q(\rho):=
    \sup_{\substack{x\in\X,\ a,b\in A\\\norm{a-b}\leq\rho}}
    |Q(x,a)-Q(x,b)|,
    \qquad
    m_A(\rho):=\inf_{a\in A}\mu_A(B_\rho(a)\cap A).
\]
Compactness and full support imply $m_A(\rho)>0$.
The following quantitative bound holds for every $x$ and $\rho>0$:
\begin{equation}\label{eq:soft_gap_bound_main}
    0\leq M_Q(x)-V_{\tau,Q}(x)
    \leq\omega_Q(\rho)+\tau\log\frac1{m_A(\rho)}.
\end{equation}
In particular, if $\omega_Q(\rho)\to0$ as $\rho\downarrow0$,
the soft-to-hard convergence is uniform in $x$.
Set
\[
    \beta_\tau(Q):=
    \inf_{\rho>0}\left\{\omega_Q(\rho)
    +\tau\log\frac1{m_A(\rho)}\right\}.
\]
For the fixed points, we have $Q_{\delta,\tau}^\ast\leq Q_\delta^\ast$
pointwise and
\begin{equation}\label{eq:soft_fixed_point_bound_main}
    \norm{Q_\delta^\ast-Q_{\delta,\tau}^\ast}_\infty
    \leq\frac{\gamma\beta_\tau(Q_\delta^\ast)}{1-\gamma}.
\end{equation}
Hence $Q_{\delta,\tau}^\ast\to Q_\delta^\ast$ uniformly as
$\tau\downarrow0$ whenever the optimal critic is uniformly continuous
in the action variable, uniformly over states.
\end{prop}

\begin{proof}
Let $M_Q(x):=\sup_{b\in A}Q(x,b)$. Since $\mu_A$ is a probability measure,
\begin{equation*}
    V_{\tau,Q}(x)
    =
    \tau\log\int_A \exp\left(\frac{Q(x,b)}{\tau}\right)\mu_A(\D b)
    \leq
    M_Q(x).
\end{equation*}
For the lower bound, fix $x$ and $\xi>0$. By continuity and compactness, choose $a_x\in A$ with $Q(x,a_x)=M_Q(x)$. There is a radius $\rho_x>0$ such that $Q(x,b)\geq M_Q(x)-\xi$ for all $b\in B_{\rho_x}(a_x)\cap A$. Since $\mu_A$ has full support,
\begin{align*}
    V_{\tau,Q}(x)
    &\geq
    \tau\log\left[
        \exp\left(\frac{M_Q(x)-\xi}{\tau}\right)
        \mu_A(B_{\rho_x}(a_x)\cap A)
    \right] \\
    &=
    M_Q(x)-\xi+\tau\log \mu_A(B_{\rho_x}(a_x)\cap A).
\end{align*}
Letting $\tau\downarrow0$ and then $\xi\downarrow0$ gives pointwise convergence.

To see that $m_A(\rho)>0$, cover $A$ by finitely many balls
$B_{\rho/2}(a_j)$ with centres in $A$.
For any $a\in A$, some $a_j$ satisfies $\norm{a-a_j}<\rho/2$,
so $B_{\rho/2}(a_j)\cap A\subseteq B_\rho(a)\cap A$.
Full support makes the minimum of the finitely many masses
$\mu_A(B_{\rho/2}(a_j)\cap A)$ positive.
The same lower-bound argument around a maximiser of $Q(x,\cdot)$
now gives \eqref{eq:soft_gap_bound_main}.
If $\omega_Q(\rho)\to0$, first choose $\rho$ to make this modulus
small and then let $\tau\downarrow0$, proving uniform convergence.

Applying \eqref{eq:soft_gap_bound_main} at $Q_\delta^\ast$ gives
\[
    \norm{\HH_\delta Q_\delta^\ast
    -\HH_{\delta,\tau}Q_\delta^\ast}_\infty
    \leq\gamma\beta_\tau(Q_\delta^\ast).
\]
The fixed-point identities and contraction of $\HH_{\delta,\tau}$
therefore imply
\[
    \norm{Q_\delta^\ast-Q_{\delta,\tau}^\ast}_\infty
    \leq\gamma\beta_\tau(Q_\delta^\ast)
    +\gamma\norm{Q_\delta^\ast-Q_{\delta,\tau}^\ast}_\infty,
\]
which proves \eqref{eq:soft_fixed_point_bound_main}.
Since $V_{\tau,Q}\leq M_Q$ and the soft operator is monotone,
iterating it from $Q_\delta^\ast$ proves the pointwise order as in
Theorem~\ref{thm:continuous_actions_grid_main}.
\end{proof}

For approximate soft updates satisfying
$\norm{Q_{k+1}-\HH_{\delta,\tau}Q_k}_\infty\leq\eta_k$,
the same contraction argument yields
\begin{equation}\label{eq:soft_fitted_q_bound_main}
    \begin{split}
        \norm{Q_K-Q_\delta^\ast}_\infty
        \leq{}&\gamma^K\norm{Q_0-Q_\delta^\ast}_\infty
        +\sum_{k=0}^{K-1}\gamma^{K-1-k}\eta_k\\
        &+\frac{1-\gamma^K}{1-\gamma}
        \gamma\beta_\tau(Q_\delta^\ast).
    \end{split}
\end{equation}
Indeed, the one-step recursion is
$d_{k+1}\leq\eta_k+\gamma d_k+\gamma\beta_\tau(Q_\delta^\ast)$
for $d_k:=\norm{Q_k-Q_\delta^\ast}_\infty$.
This separates action regularisation from the numerical update error.

The soft-value bound also quantifies the role of the action reference measure.
For example, if $m_A(\rho)\geq c_A\rho^p$ for
$0<\rho\leq\rho_0$, with $c_A,p>0$, and $Q_\delta^\ast$ is
uniformly $L_A$-Lipschitz in actions, then
\[
    \beta_\tau(Q_\delta^\ast)
    \leq L_A\rho+\tau\log\frac1{c_A\rho^p}.
\]
Choosing $\rho=\tau$ for sufficiently small $\tau$ gives
an $O(\tau\log(1/\tau))$ soft-value error.
Full support alone does not give this particular rate.

For implementation with a stochastic actor, the relevant entropy
is relative to $\mu_A$, not necessarily to Lebesgue measure.
The entropy variational formula is
\begin{equation}\label{eq:soft_actor_variational_main}
    V_{\tau,Q}(x)
    =\sup_{\pi\ll\mu_A}
    \left\{\int_A Q(x,a)\pi(\D a)
    -\tau H(\pi\mid\mu_A)\right\}.
\end{equation}
Its maximiser is the Gibbs action law
\[
    \frac{\D\pi_{\tau,Q}(\cdot\mid x)}{\D\mu_A}(a)
    =\exp\left(\frac{Q(x,a)-V_{\tau,Q}(x)}{\tau}\right).
\]
For an actor $\pi_\phi(\cdot\mid x)\ll\mu_A$ of finite relative
entropy, write $p_\phi(a\mid x):=\D\pi_\phi(\cdot\mid x)/\D\mu_A(a)$.
The implementable actor value is
\[
    \mathcal{V}_{\tau,Q}^{\phi}(x)
    :=\E_{a\sim\pi_\phi(\cdot\mid x)}
    [Q(x,a)-\tau\log p_\phi(a\mid x)].
\]
Substituting the Gibbs density in the entropy chain rule gives the
exact actor-gap identity
\begin{equation}\label{eq:soft_actor_gap_main}
    V_{\tau,Q}(x)-\mathcal{V}_{\tau,Q}^{\phi}(x)
    =\tau H(\pi_\phi(\cdot\mid x)\mid\pi_{\tau,Q}(\cdot\mid x))
    \geq0.
\end{equation}
This proves \eqref{eq:soft_actor_variational_main} and shows that a
uniform actor KL bound gives an explicit continuation-value error.
Replacing $V_{\tau,Q}$ by this actor value changes the robust critic
target by at most
$\gamma\tau\sup_x H(\pi_\phi(\cdot\mid x)\mid\pi_{\tau,Q}(\cdot\mid x))$,
before sampling and regression errors.
If densities are instead written relative to Lebesgue measure and
$\mu_A$ has density $p_A$, the logarithm is
$\log(p_\phi^{\mathrm{Leb}}/p_A)$.
Omitting $p_A$ generally changes the regularised objective.

A deterministic actor cannot implement this variational expression
when $\mu_A$ is non-atomic, because its point-mass action law has
infinite relative entropy.
It may instead approximate the hard maximum, provided its uniform
action-suboptimality gap is controlled.
The deterministic RTD3 continuation in Section~\ref{sec:robusttd3}
is not the soft operator \eqref{eq:soft_H_main}; an entropy bonus on
portfolio weights is also distinct from $H(\pi(\cdot\mid x)\mid\mu_A)$.
The contraction results do not establish convergence of either
actor's gradient updates or eliminate the sampled-multiplier
discrepancy.

\subsubsection{From critic accuracy to robust policy performance}
\label{sec:continuous_action_policy}

The following bound concerns the original, unregularised robust
return and makes the action-approximation errors relevant to
decision quality.
For a measurable stationary policy $\pi:\X\to A$, let $V_\delta^\pi$
be its worst-case discounted value under the same state--action
rectangular ambiguity sets as in \eqref{eq:robust_problem_1}.

\begin{prop}[Robust performance of an approximate greedy policy]
\label{prop:robust_policy_performance}
Let $Q\in\calQ$ satisfy
$\norm{Q-Q_\delta^\ast}_\infty\leq e$.
If a measurable policy $\pi$ satisfies
\[
    M_Q(x)-Q(x,\pi(x))\leq\xi
    \qquad\text{for all }x\in\X,
\]
then
\begin{equation}\label{eq:robust_policy_performance_main}
    0\leq V_\delta(x)-V_\delta^\pi(x)
    \leq\frac{2e+\xi}{1-\gamma}
    \qquad\text{for all }x\in\X.
\end{equation}
Alternatively, if $\pi_h(x)\in\argmax_{a\in A_h}Q(x,a)$ and
$Q_\delta^\ast$ is $L_A$-Lipschitz in actions, then
\begin{equation}\label{eq:grid_policy_performance_main}
    0\leq V_\delta(x)-V_\delta^{\pi_h}(x)
    \leq\frac{2e+L_Ah}{1-\gamma}.
\end{equation}
\end{prop}

\begin{proof}
Let $\T_\delta^\pi$ be the robust policy-evaluation operator,
\[
    (\T_\delta^\pi v)(x)
    :=\inf_{\PP\in\Sball(\widehat{\PP}(x,\pi(x)))}
    \E_{\PP}[r(x,\pi(x),X_1)+\gamma v(X_1)].
\]
It is a $\gamma$-contraction, and
$V_\delta^\pi=\T_\delta^\pi V_\delta^\pi$.
Since $V_\delta=M_{Q_\delta^\ast}$, critic accuracy and the assumed
action gap give
\[
    0\leq V_\delta(x)-(\T_\delta^\pi V_\delta)(x)
    =M_{Q_\delta^\ast}(x)-Q_\delta^\ast(x,\pi(x))
    \leq2e+\xi.
\]
Consequently,
\[
    \norm{V_\delta-V_\delta^\pi}_\infty
    \leq2e+\xi+\gamma\norm{V_\delta-V_\delta^\pi}_\infty.
\]
Rearranging and using optimality of $V_\delta$ proves
\eqref{eq:robust_policy_performance_main}.
For $\pi_h$, the grid bound at $Q_\delta^\ast$ gives
\begin{align*}
    M_{Q_\delta^\ast}(x)
    &\leq\max_{a\in A_h}Q_\delta^\ast(x,a)+L_Ah\\
    &\leq Q(x,\pi_h(x))+e+L_Ah\\
    &\leq Q_\delta^\ast(x,\pi_h(x))+2e+L_Ah.
\end{align*}
The same contraction argument proves
\eqref{eq:grid_policy_performance_main}.
\end{proof}

The critic-error bounds in \eqref{eq:approx_grid_main},
\eqref{eq:soft_fixed_point_bound_main} and
\eqref{eq:soft_fitted_q_bound_main} can therefore be substituted
for $e$ to obtain robust decision-quality guarantees for grid-greedy
or hard-greedy policies.
For a learned critic, $e$ must also include training and structural
operator errors.
These conclusions are conditional on the stated uniform bounds;
they do not assert global convergence of RDQN or RTD3.
}

{
\subsection{The Robust TD3 Algorithm} \label{sec:robusttd3}

The RDQN algorithm of Section~\ref{sec:robustdqn} evaluates the supremum of the critic over the action space which is feasible when the neural network outputs all values of the critic for a finite action space.
For the portfolio optimisation problem we consider in Section~\ref{sec:stocks_opt}, it is a continuous action space and hence this approach is not feasible.
Therefore, we also give an actor-critic variant, which we call Robust TD3 (RTD3), obtained by applying the same modified target to the Twin Delayed Deep Deterministic policy gradient algorithm (TD3) of \citet{fujimoto2018addressing}.

TD3 maintains a deterministic actor $\pi_{\rm NN}(\phi;\cdot)$ and two critics $Q_{\rm NN}(\theta_1;\cdot,\cdot)$, $Q_{\rm NN}(\theta_2;\cdot,\cdot)$, each with a target network, and forms its target from the stored reward and the smaller of the two target critics evaluated at the target actor's action at the realised next state.
The actor is trained to maximise the critic's value at the next state and hence it approximates the supremum over the action space.
As in the discrete-action case, the key change we make is to the target.
Algorithm~\ref{algo:robusttd3} replaces the two lines of \citet[Algorithm 1]{fujimoto2018addressing} that compute the smoothed target action $\tilde{a}$ and the target $y$.

The hard maximum $\sup_{b \in A}Q_\delta^*(x^\nu,b)$ in Proposition~\ref{prop:dual} is not available for a continuous $A$, and is replaced by the continuation value that the current target policy induces,
\begin{equation} \label{eq:rtd3_continuation}
	x^\nu \mapsto \min_{k \in \{1,2\}} Q_{\rm NN}\left(\theta_{\text{target},k}; x^\nu, \pi_{\rm NN}(\phi_{\text{target}};x^\nu)\right),
\end{equation}
evaluated at each of the $N_\nu$ candidates.
Everything else in the target, the candidate construction of Section~\ref{sec:sampling_from_nu}, the feasibility mask of Section~\ref{sec:warning} and the dual solve of Section~\ref{sec:optimisation_of_lambda}, is shared with Algorithm~\ref{algo:robustdqn} without modification. Beyond the target, our implementation differs from \citet{fujimoto2018addressing} in the following respects.

\begin{rem}[What operator RTD3 approximates] \label{rem:rtd3_operator}
	The results of Sections~\ref{sec:setting} and~\ref{sec:algorithm} concern the operator $\HH_\delta$ built with the exact maximum over $A$, whose fixed point is the robust optimal $Q$ function $Q_\delta^*$.
	Substituting \eqref{eq:rtd3_continuation} changes the object being computed: with the continuation value supplied by a fixed policy, the corresponding operator is one of robust policy \emph{evaluation}, whose fixed point is the worst-case value of that policy rather than $Q_\delta^*$.
	The actor is then improved against that critic in the usual deterministic policy gradient manner, so RTD3 is a robust policy iteration scheme in which neither half is solved exactly before the other is updated.
\end{rem}

\emph{The actor carries an entropy bonus.}
The actor loss is the usual deterministic policy gradient objective augmented by $\beta H(w)$, where $w$ is the vector of portfolio weights the actor outputs and $H$ is the Shannon entropy, in the spirit of the entropy regularisation used in policy gradient methods \citep{mnih2016asynchronous}.
The entropy bonus encourages exploration and discourages the actor from collapsing to a single asset, which is important in a portfolio context.
TD3 has no such term.

\emph{The exploration noise is autocorrelated and applied on the simplex.}
Rather than the independent Gaussian perturbations of \citet{fujimoto2018addressing}, we add an Ornstein--Uhlenbeck perturbation with autocorrelation $\rho$, decaying in magnitude over training, and project the perturbed weights back onto the simplex.
Using independent noise would create excessive turnover which comes at a cost, so an autocorrelated perturbation achieves the same exploration reach at a materially lower cost in reward.

\begin{algorithm}[h!]
	
	\renewcommand\mycommfont[1]{\footnotesize\ttfamily\textcolor{black}{#1}}
	\SetAlgoLined
	\SetNoFillComment
	\SetKwInOut{Input}{Input}
	\SetKwFunction{NuCand}{NuCandidates}
	\SetKwFunction{DualSolve}{SolveDual}

	\Input{
	Reward function $r$;
	Discount Factor $\gamma \in (0,1)$;
	Sinkhorn radius $\varepsilon>0$;
	Parameter $\delta>0$ for the Sinkhorn regularisation;
	Number of candidates per transition $N_{\nu}$;
	Sampling distribution $\nu$ and its proposal mixture (Section~\ref{sec:sampling_from_nu});
	Batch size $B$;
	Entropy coefficient $\beta \geq 0$;
	Policy delay $d_\pi$;
	Polyak coefficient $\tau$;
	Critics $Q_{\rm NN}(\theta_k;\cdot,\cdot)$ and targets $Q_{\rm NN}(\theta_{\text{target},k};\cdot,\cdot)$, $k \in \{1,2\}$;
	Actor $\pi_{\rm NN}(\phi;\cdot)$ and target actor $\pi_{\rm NN}(\phi_{\text{target}};\cdot)$;
	Update counter $m$;
	}

	Sample a batch of transitions $(x_i,a_i,x^{\PP}_i)$ for $i=1,\dots,B$ from the replay buffer, together with their cached dual variables $\log \lambda_i^{\text{cache}}$;

	$\left(x_{i,j}^{\nu}, w_{i,j}\right)_{j=1,\dots,N_\nu} \leftarrow$ \NuCand{$x_i^{\PP}$, $\delta$} for $i=1,\dots,B$;
	\tcp{(See Section \ref{sec:sampling_from_nu})}

	Compute the transport costs $c_{i,j} = \|x_i^{\PP}-x_{i,j}^{\nu}\|$ and, using the target actor's action at each candidate, the candidate values
	\begin{equation*}
		Z_{i,j} = -r(x_i,a_i,x_{i,j}^{\nu})-\gamma\ \min_{k \in \{1,2\}} Q_{\rm NN}\left(\theta_{\text{target},k}; x^\nu_{i,j}, \pi_{\rm NN}(\phi_{\text{target}};x^\nu_{i,j})\right);
	\end{equation*}

	Compute $\bar{\varepsilon}_i = \varepsilon + \delta \log\left(\frac{1}{N_{\nu}} \sum_{j=1}^{N_{\nu}} w_{i,j}\, e^{-c_{i,j}/\delta}\right)$ and set $\mathcal{I} = \{i : \bar{\varepsilon}_i > 0\}$;
	\tcp{(See Section \ref{sec:warning})}

	For $i \in \mathcal{I}$, obtain $y_i := \widehat{\mathcal{H}_\delta} Q_{\rm NN}(\lambda_i^*, \theta_{\text{target}}; x_i, a_i)$ by maximising the dual over $\lambda_i$ exactly as in Algorithm~\ref{algo:robustdqn}, warm-started at $\log\lambda_i^{\text{cache}}$, and write the optimised values back to the cache;
	\tcp{(See Sections \ref{sec:optimisation_of_lambda} and \ref{sec:exponential_term})}

	\If{$\mathcal{I} \neq \emptyset$}{
	Take a gradient step on
	$
	\sum_{k \in \{1,2\}}\frac{1}{|\mathcal{I}|} \sum_{i \in \mathcal{I}} \left(y_i-Q_{\rm NN }(\theta_k;x_i,a_i)\right)^2
	$ w.r.t.\ $\theta_1$ and $\theta_2$ to minimise the value;
	}

	\If{$m \equiv 0 \mod d_\pi$}{
	Take a gradient step on
	$
	-\frac{1}{B}\sum_{i=1}^{B}\left[ Q_{\rm NN}\left(\theta_1;x_i,\pi_{\rm NN}(\phi;x_i)\right) + \beta H\left(\pi_{\rm NN}(\phi;x_i)\right)\right]
	$ w.r.t.\ $\phi$ to minimise the value;

	Update $\theta_{\text{target},k} \leftarrow \tau \theta_k + (1-\tau)\theta_{\text{target},k}$ for $k \in \{1,2\}$ and $\phi_{\text{target}} \leftarrow \tau \phi + (1-\tau)\phi_{\text{target}}$;
	}

	$m \leftarrow m+1$;

	\caption{Robust TD3 (one update)}\label{algo:robusttd3}
\end{algorithm}
}

\subsection{Worst Case Distribution and \texorpdfstring{$\delta$}{delta}} \label{sec:worst_case_dist}

Note that as we lower $\delta$, the Sinkhorn distance approaches the Wasserstein distance but the transport plan $\pi$ that realises the Sinkhorn distance $W_{\delta}$ becomes more sparse and hence the worst case distribution tends to become more discrete in nature (see e.g. \citet[Remark 4]{wang2021sinkhorn}).
Therefore, even if the reference distribution is a smooth and continuous distribution, the worst case distribution may not be.
This may or may not be desirable depending on the specific application.

The regularisation term in the Sinkhorn distance allows us to control the smoothness of the worst case distribution by adjusting the parameter $\delta$.
Our definition using $\nu$ allows us to choose a prior, whereby if we increase $\delta$, we can make the worst case distribution more similar to $\nu$.

We illustrate this with a specially designed environment where we have an action space with a single action value and a continuous state space in the closed interval $[0,1]$.
From the equivalence in Equation~\eqref{eq:sup_Q_V} and since there is only a single action, the $Q$ function $Q_\delta(x,a)$ in this case is equivalent to the value function $V_\delta(x)$.

If we choose the reward function to be $r(x_0,a,x_1) = \one_{\{x_1 \leq x_0\}}$ and set the discount factor $\gamma$ close to zero, then from Equation~\eqref{eq:robust_problem_1}, we have
\begin{equation}
	\begin{split}
	V_\delta(x) &= \sup_{a \in A} \inf_{\PP \in \Sball\left(\widehat{\PP}(x_0,a)\right)} \E_{\PP} \bigg[r(x_0,a,X_1) + \gamma V_\delta(X_1) \bigg]
	\\
	&\approx\inf_{\PP \in \Sball\left(\widehat{\PP}(x_0,a)\right)} \E_{\PP}\left[\one_{\{X_1 \leq x_0\}}\right]
	\\
	&= \inf_{\PP \in \Sball\left(\widehat{\PP}(x_0,a)\right)} \PP\left(X_1\leq x_0\right)
	\end{split}
\end{equation}
which is the cumulative distribution function of the worst case distribution.
In this particular case, the worst case distribution would be one leading to high values of $X$.
Note that the worst case distribution is tied to the specific choice of reward function and hence we cannot use this technique to find the worst case distribution for different environments.
However, we can still use it to see the effect of $\delta$ on the smoothness of the worst case distribution.

With a reference distribution $\widehat{\PP} = \text{Beta}(2,2)$ and a sampling distribution $\nu = \text{Uniform}(0,1)$, we set $\varepsilon=0.5$ while varying $\delta$ and show the resulting $Q$ function in Figure~\ref{fig:worst_CDF_uniform}.
When $\delta=10$, the high level of entropy regularisation forces the worst case distribution to be close to $\nu$ which is a uniform distribution in this case.
This means $\nu$ is effectively a prior for the worst case distribution whereby increasing $\delta$ pushes all distributions in the Sinkhorn ball, including the worst case distribution, towards $\nu$.
As $\delta$ gets lowered towards $0.01$, the worst case distribution moves towards a much more discrete distribution with the regularisation starting to become negligible.

\begin{figure}[h]
    \centering
    \subfigure[$\delta = 10$]{
        \includegraphics[width=0.47\linewidth]{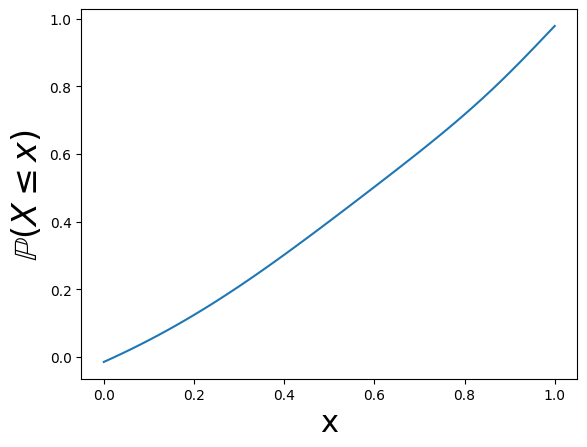}
    }
	\subfigure[$\delta = 1$]{
        \includegraphics[width=0.47\linewidth]{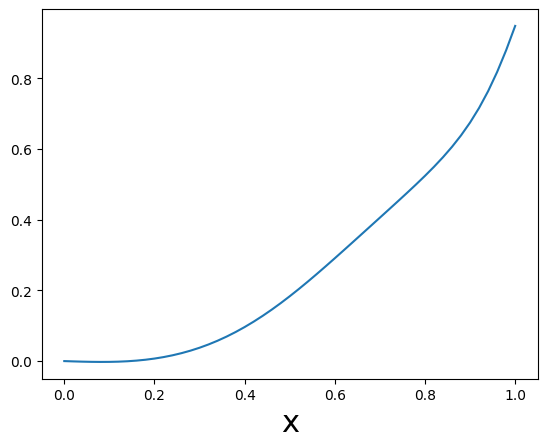}
    }
	\subfigure[$\delta = 0.1$]{
        \includegraphics[width=0.47\linewidth]{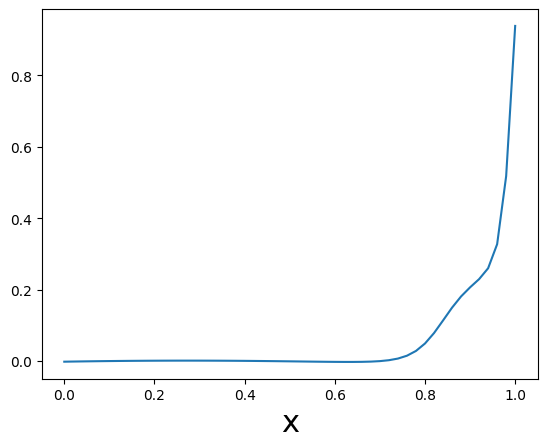}
    }
    \subfigure[$\delta = 0.01$]{
        \includegraphics[width=0.47\linewidth]{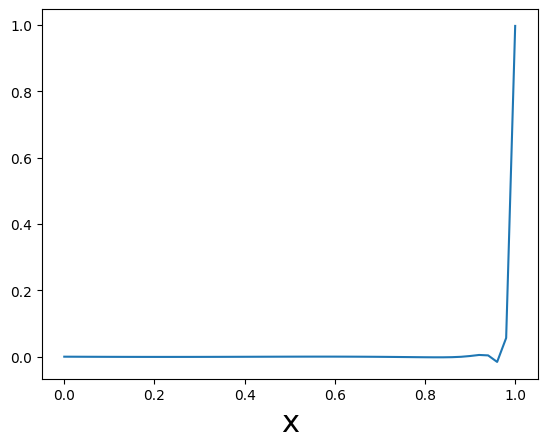}
    }
    \caption{The worst case cumulative distribution function for different values of $\delta$ with $\nu=$ Uniform$(0,1)$}
    \label{fig:worst_CDF_uniform}
\end{figure}

Choosing $\nu$ with the wrong support or where critical parts of the support of $\widehat{\PP}$ are low in probability can lead to the wrong worst case distribution as we see in Figure~\ref{fig:worst_CDF_beta}.
With a reference distribution $\widehat{\PP} = \text{Beta}(2,2)$ and a sampling distribution $\nu = \text{Beta}(1,5)$, we set the same $\varepsilon$ and see that even as $\delta$ is lowered, the worst case distribution does not approach the distribution seen above when $\nu$ is a uniform distribution.
This is because larger values in the interval $[0,1]$ have very low probability of being sampled with $\nu = \text{Beta}(1,5)$.

\begin{figure}[h]
    \centering
    \subfigure[$\delta = 10$]{
        \includegraphics[width=0.47\linewidth]{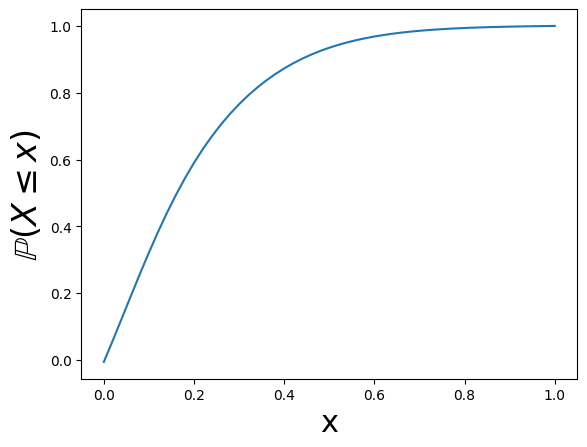}
    }
	\subfigure[$\delta = 1$]{
        \includegraphics[width=0.47\linewidth]{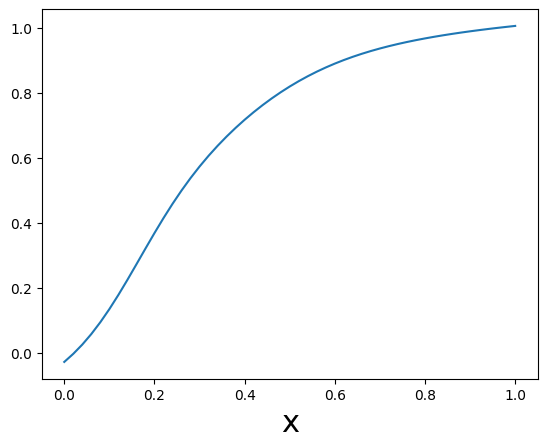}
    }
	\subfigure[$\delta = 0.1$]{
        \includegraphics[width=0.47\linewidth]{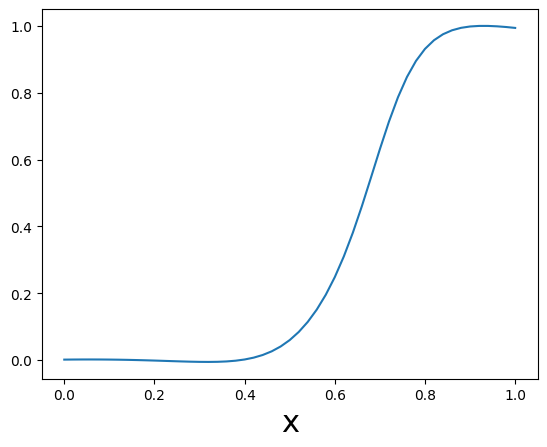}
    }
    \subfigure[$\delta = 0.01$]{
        \includegraphics[width=0.47\linewidth]{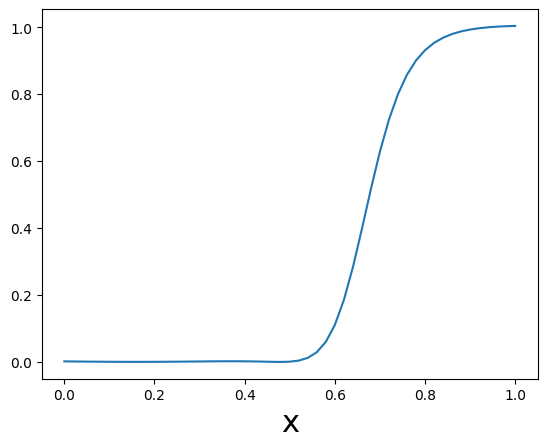}
    }
    \caption{The worst case cumulative distribution function for different values of $\delta$ with $\nu=$ Beta$(1,5)$}
    \label{fig:worst_CDF_beta}
\end{figure}

 \section{Applications} \label{sec:applications}
In this section we discuss applications of Algorithm~\ref{algo:robustdqn} and Algorithm~\ref{algo:robusttd3} to toy and real-world environments.

\subsection{Gambling on the Unit Square} \label{sec:toy}

We start by illustrating our approach by letting an agent play an easy-to-understand environment.
In this environment, the state $X_t$ is a single number in the interval $[0,1]=:\X$.
The agent can choose to gamble on this number by choosing an action $a \in A:=\{-1,0,1\}$.
The reward is then given by
\begin{equation} \label{eq:toy_reward}
	r(x_0,a,x_1) = a(x_1 - x_0)(1 + (M-1)\one_{a(x_1 - x_0) < 0}),
\end{equation}
for some $M \geq 1$.
In other words, the agent wins or loses the amount of money equal to the difference between the next state and current state multiplied by the action chosen and a reward factor $M$ (relevance to be explained subsequently).
The goal is to maximise the total reward.

In the following, we describe the \emph{true distribution} of the environment.
The initial state $X_0 \sim \text{Beta}(\alpha',\beta')$, i.e., it is distributed according to a Beta distribution with parameters $\alpha'=1.2,\beta'=2$.
The realised state $x_0$ and action $a$ are used to determine the parameters of the Beta distribution used to sample the next state.
Specifically, when $a \neq 0$, the next state, $X_1 \sim \text{Beta}(\alpha,\beta) =: \PP(x_0,a)$ where $\alpha=g(\alpha' - a x_0)$ and $\beta=g(\beta' + a(1 - x_0))$. The function $g(x) = \log(1+e^x)$ is the softplus function to ensure $\alpha,\beta > 0$.
Put simply, the agent's action and the current state combine to shift the $\alpha,\beta$ parameters used to sample the next state.
To be clear, for each state transition, the parameters $\alpha,\beta$ are always calculated using the original values $\alpha',\beta'$.
If $a=0$, then the next state is sampled from the initial distribution $\text{Beta}(\alpha',\beta') = \PP(x_0,0)$.

\subsubsection{Ambiguity in the distribution}

We create ambiguity in the \emph{true distribution} of the state transition in the following manner.
The agent knows that the initial state follows a Beta distribution but is not given the parameters.
The agent is also aware of how the actions affect the parameters of the Beta distribution in the state transition.
Before the game starts, 5 samples from the initial Beta distribution are given.
The agent will construct a \emph{reference distribution} by inferring a Beta distribution based on these 5 samples.
We use the method of moments to estimate $\alpha'$ and $\beta'$, which was shown to have better performance for small samples compared to maximum likelihood estimation with the Beta distribution \citep{ali2023investigating}.
In other words, we have $\widehat{\PP}(x_t,0) = $ Beta$(\hat{\alpha'},\hat{\beta'})$ where $\hat{\alpha'},\hat{\beta'}$ are the estimated parameters and $\widehat{\PP}(x_t,a) = $ Beta$(g(\hat{\alpha'}-a x_t),g(\hat{\beta'} + a(1-x_t)))$ for $a \neq 0$.
For $\nu$, we choose the uniform distribution on $[0,1]$ since it is a simple and valid choice for the state space.

\subsubsection{The case for robustness}

As the goal of robust optimisation is to ensure that the policy performs well under worst case scenarios, the result should be to reduce the occurrence of very unfavourable outcomes at the expense of reducing the reward under the most favourable outcomes.
Therefore, we would expect the RDQN algorithm to be at its most effective when the agent is penalised more heavily for the wrong action than it is rewarded for the right action.

To this end, we experiment by tweaking the reward function to penalise the agent more heavily for the wrong action.
We multiply the reward by a factor of 1, 5 and 10 \emph{only when it is negative} to create three different environments that vary in the severity of the penalty for the wrong action.
This means $M=1,5,10$ in \eqref{eq:toy_reward}.

\begin{lem} \label{lem:toy_assump}
	The game satisfies Assumptions \ref{asu_1} and~\ref{asu_2}.
\end{lem}

\subsubsection{Optimal policy}

Since we know the true transition dynamics, we can calculate the expected reward \[\E_{X_{t+1} \sim \PP(x_t,a)}[r(x_t,a,X_{t+1})]\] for each action given the state which is shown in Figure~\ref{fig:optimal_policy}.
Note that while the state transition is dependent on the current state and action, the expected reward given the current state and action is the same under the initial distribution or any subsequent state transition distribution.
We use a policy that takes the action with the highest expected reward, as a proxy for the optimal policy under the \emph{true distribution}, to calculate a benchmark.
Based on simulations over $10^6$ steps using this policy, the average reward received was around 0.073.

\begin{figure}
	\centering
	\includegraphics[scale=0.5]{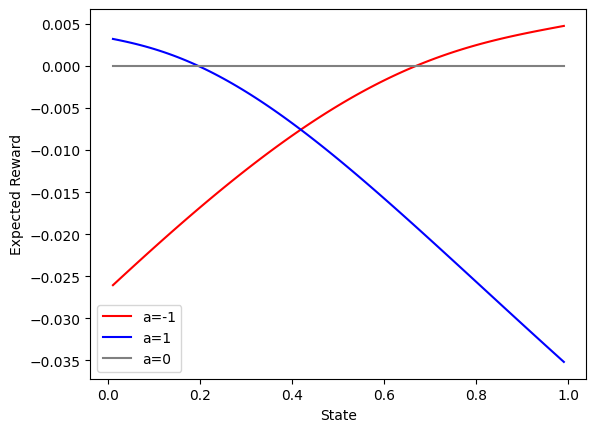}
	\caption[The expected reward for each action based on the current state]{The expected reward for each action $a_{t+1}$ based on the current state $x_t$ and action $a_t$ under $\PP(x_t,a_t)$. The x-axis is the state $x$ and the y-axis is the expected reward for the action taken.}
	\label{fig:optimal_policy}
\end{figure}

\subsubsection{Experiments}

We will compare the performance of DQN and RDQN.
Both algorithms are trained on the environment following the \emph{reference distribution}.
The performance of the two algorithms is evaluated on the environment following the \emph{true distribution}.
For implementation details, we refer the reader to \url{https://github.com/luchungi/Sinkhorn-DRL}.

Table \ref{tab:toy_perf_factor_5} shows the performance of the two algorithms in terms of average reward per step.
The entire game is repeated 100 times.
This means for each game, a new set of 5 initial samples is given, from which the \emph{reference distribution} is inferred and used to train the agent.
The trained agent is then evaluated on 100 independent environments, all using the \emph{true distribution}, for 10,000 steps each.
The average reward per step is calculated over the total of 1,000,000 reward samples.
The reward factor used is 5.
The statistics in Table \ref{tab:toy_perf_factor_5} are derived from the set of 100 games.

Firstly, for $\delta=0.0001$, we see that as $\varepsilon$ increases, the more unfavourable outcomes at the 5\% and 10\% quantiles improve as the RDQN agent becomes more conservative.
This comes at the expense of more favourable outcomes at the median level.
The best performance is achieved when $\varepsilon=0.1$, where there are no negative outcomes even at the 5\% quantile.
It also achieves a higher mean reward per step compared to the DQN agent.

Secondly, increasing the regularisation parameter $\delta$ does not help improve the performance.
Although the mean, median and maximum reward per step are similar, the higher regularisation results in more unfavourable outcomes at the 5\% and 10\% quantiles and also higher variance.
This is likely due to the fact that our prior does not match the \emph{true distribution} as we are using the uniform distribution for $\nu$.

We also evaluate the performance of the agents under the \emph{reference distribution} and it expectedly shows the DQN agent outperforming the RDQN agent.
We relegate the results to Appendix~\ref{app:toy_example}.

\begin{table}[!htbp]
\centering
\begin{adjustbox}{max width=\linewidth}\begin{tabular}{lccrrrrrrr}
	\toprule
	Model & $\varepsilon$ & $\delta$ & Mean & Std & Min & 5\% & 10\% & 50\% & Max \\
	\midrule
	DQN & - & - & 0.032 & 0.063 & -0.199 & -0.110 & -0.073 & \textbf{0.062} & 0.073 \\
	RDQN & 0.05 & 0.0001 & 0.031 & 0.046 & -0.105 & -0.076 & -0.047 & 0.051 & 0.073 \\
	RDQN & 0.1 & 0.0001 & \textbf{0.047} & 0.018 & -0.021 & \textbf{0.015} & \textbf{0.028} & 0.049 & 0.072 \\
	RDQN & 0.1 & 0.01 & 0.041 & 0.034 & -0.076 & -0.034 & -0.010 & 0.055 & \textbf{0.073} \\
	RDQN & 0.2 & 0.0001 & 0.027 & \textbf{0.017} & \textbf{0.000} & 0.004 & 0.006 & 0.025 & 0.068 \\
	\bottomrule
\end{tabular}\end{adjustbox}
\caption[Performance of DQN vs RDQN in terms of average reward per step]{Performance of DQN vs RDQN in terms of average reward per step. The entire game is repeated 100 times, i.e., a new set of 5 samples is given for each game to create the \emph{reference distribution}. The average reward per step is evaluated on the trained agent playing 10,000 steps over 100 independent environments, following the \emph{true distribution}, for a total of 1,000,000 reward samples per game. The statistics are calculated over the 100 games. The reward factor is 5 in this case.}
\label{tab:toy_perf_factor_5}
\end{table}

In Table \ref{tab:toy_perf_factors}, we show the performance of the two algorithms for different reward factors.
We have seen above that the RDQN agent outperforms the DQN agent when there is an asymmetry that penalises the agent more heavily for the wrong action.
With a symmetric reward function, the RDQN agent does not outperform DQN at any quantile or in the mean.
When the reward function becomes more asymmetric with a factor of 10, the outperformance becomes more pronounced.
Therefore, the RDQN algorithm is more appropriate in environments where the agent is penalised relatively more heavily for the wrong action.

\begin{table}[!htbp]
\centering
\begin{adjustbox}{max width=\linewidth}\begin{tabular}{lcccrrrrrrr}
	\toprule
	Model & $\varepsilon$ & $\delta$ & $M$ & Mean & Std & Min & 5\% & 10\% & 50\% & Max \\
	\midrule
	DQN & - & - & 1 & 0.101 & 0.012 & 0.040 & 0.076 & 0.084 & 0.105 & 0.113 \\
	RDQN & 0.1 & 0.0001 & 1 & 0.091 & 0.014 & 0.007 & 0.067 & 0.073 & 0.096 & 0.112 \\
	\midrule
	DQN & - & - & 5 & 0.032 & 0.063 & -0.199 & -0.110 & -0.073 & 0.062 & 0.073 \\
	RDQN & 0.1 & 0.0001 & 5 & 0.047 & 0.018 & -0.021 & 0.015 & 0.028 & 0.049 & 0.072 \\
	\midrule
	DQN & - & - & 10 & -0.009 & 0.111 & -0.441 & -0.257 & -0.177 & 0.043 & 0.057 \\
	RDQN & 0.1 & 0.0001 & 10 & 0.031 & 0.019 & -0.066 & 0.009 & 0.014 & 0.032 & 0.058 \\
	\bottomrule
\end{tabular}\end{adjustbox}
\caption[Performance of DQN vs RDQN in terms of average reward per step with $M=1,5,10$]{Performance of the agents in terms of average reward per step as in Table \ref{tab:toy_perf_factor_5} but with $M=1,5,10$.}
\label{tab:toy_perf_factors}
\end{table}

\subsection{Portfolio Optimisation with RDQN} \label{sec:port_opt}
{
We now turn to a financial trading environment.
The S\&P 500 index is a stock market index that tracks the performance of 500 large companies listed on stock exchanges in the United States.
It serves as a case study in which the \emph{reference distribution} is a generative model fitted to historical data and the \emph{true distribution} is the realised out-of-sample data, so that the distributional shift the agent faces is precisely the shift between the simulator and the market.

\subsubsection{The environment}
The action, $a \in A := \{-1,-0.75,\dots,0.75,1\}$, is the weight of the portfolio in the S\&P 500 index which starts from -1 and goes in increments of 0.25 until 1.
In other words, the agent is allowed to short the index up to an amount equal to the value of the portfolio.
The reward is the log return of the portfolio, i.e., the difference in the log value of the portfolio from one time step to the next.
The goal is to maximise the cumulative log return, effectively maximising the portfolio value.

Before constructing the state, we require the log returns for each time step to be bounded.
This is a reasonable assumption for financial markets and ensures our setting satisfies Assumption~\ref{asu_2}.
As described in Section~\ref{sec:sampling_from_nu}, the prior measure $\nu$ is truncated to a bounded box, and for the S\&P 500 we take that box to be $[\underline{y},\overline{y}] := [\log 0.05, \log 1.95]$.
These bounds keep the reward for every action in $A$ bounded; see Lemma~\ref{lem:gross_return_floor}.
The state $X_t$ at time $t$ consists of three components:
\begin{enumerate}
	\item $(X_t^{(1)},\ldots,X_t^{(60)})$ are the log returns of the past 60 time steps $\subset \R^{60}$, \label{itm:log_returns}
	\item $X_t^{(61)}$ is the current weight of the portfolio in the S\&P 500 index $\subset \R$, \label{itm:position}
	\item $X_t^{(62)}$ is the time delta from current time step to the next time step in calendar years $\in \R^+$\footnote{The time delta is relevant in two ways. It is used to compute the interest that will be accrued for cash holdings and it is part of the input for the generative model.}.
\end{enumerate}
Therefore, $\X \subset \R^{62}$.
Note that the time delta is a variable that is exogenous as it is a function of the trading calendar.
Note also that component~\ref{itm:position} is the weight the portfolio actually holds at the start of the period, which is in general \emph{not} one of the values in $A$.
The action sets the weight at the moment of trading, after which the position drifts with the price of the index to the next time step, as made precise in \eqref{eq:port_opt_fX} below.
It remains bounded, because the gross return of the portfolio is bounded by Lemma~\ref{lem:gross_return_floor}, so $\X$ is bounded.

We include a transaction cost that is $c:=0.05\%$ of the value of the notional amount being traded.
This is factored into the reward that results from the action.
The interest rate is set to a continuously compounded rate of $r_f:=2.4\%$ which is the per annum log return from investing in 3-month treasury bills for the period 1 Jan 1995 to 31 Dec 2020.

Note that given the current action and state, the only uncertainty that remains is the next log return.
Every other coordinate of $X_{t+1}$ is a deterministic function of $(x_t,a_t)$ and that single log return.
We train a generative model from \citet[Section 5]{lu2024generative} on the S\&P 500 historical data to simulate the next log return.
To make this dependence explicit, we first introduce the two quantities that both the reward and the transition are built from.
The \emph{gross return} of the portfolio over one period, as a function of the next log return $y$, is
\begin{equation} \label{eq:port_opt_gross}
	\X \times A \times \R \ni (x_t,a_t,y) \mapsto G(x_t,a_t,y) := 1 + a_t\left(e^{y}-1\right) + (1-a_t)\left(e^{r_f x_t^{(62)}} - 1\right),
\end{equation}
where the two terms represent the return on the index position and the interest earned on cash holdings $1-a_t$ over a period of length $x_t^{(62)}$.
The transaction cost is charged on the notional actually traded, which is itself reduced by the cost, so we solve the following equation to arrive at the cost $k$ expressed as a fraction of the pre-trade portfolio value such that the resulting weights after the trade are exactly $a_t$.
\begin{equation} \label{eq:port_opt_cost}
	k = c\left|(1-k)a_t - x_t^{(61)}\right|,
	\;\text{equivalently}\;
	k(x_t,a_t) = \frac{c \cdot s\left(a_t - x_t^{(61)}\right)}{1 + c \cdot s \cdot a_t},
	\quad s := \operatorname{sgn}\!\left(a_t - x_t^{(61)}\right),
\end{equation}
which admits a unique solution, depending continuously on $(x_t,a_t)$, by Lemma~\ref{lem:gross_return_floor}.
If there is no change in the position then $k=0$.
The reward is the log return of the portfolio net of this cost,
\begin{equation} \label{eq:port_opt_reward}
    \begin{split}
        \X \times A \times \X \ni (x_t,a_t,x_{t+1}) &\mapsto r(x_t,a_t,x_{t+1}) := r'(x_t,a_t,x_{t+1}^{(60)}),
        \\
        \X \times A \times \R \ni (x_t,a_t,y) &\mapsto r'(x_t,a_t,y) := \log\Big(\big(1-k(x_t,a_t)\big)\,G(x_t,a_t,y)\Big) \in \R,
    \end{split}
\end{equation}
and the next state is obtained from the next log return by the map
\begin{equation} \label{eq:port_opt_fX}
	\X \times A \times \R \ni (x_t,a_t,y) \mapsto f_X(x_t,a_t,y) := \left(x_t^{(2)},\ldots,x_t^{(60)},\; y,\; \frac{a_t\,e^{y}}{G(x_t,a_t,y)},\; \Delta_{t+1}\right) \in \X,
\end{equation}
where $\Delta_{t+1}$ is the time delta from time step $t+1$ to $t+2$.
The window of log returns is shifted and $y$ appended.
The position becomes the weight that $a_t$ drifts to as the index moves over the period, which is why this component is a function of $y$ and not the constant $a_t$.
The time delta advances along the calendar.
With the function $f_X$, the reference transition kernel is the pushforward of the generative model under it,
\begin{equation} \label{eq:port_opt_transition}
	\X \times A \ni (x_t,a_t) \mapsto \widehat{\PP}(x_t,a_t) := \big(f_X(x_t,a_t,\cdot)\big)_{\#}\,\PP_{\text{gen}}(x_t,a_t),
\end{equation}
where $\PP_{\text{gen}}(x_t,a_t)$ is the generative model that generates the next log return based on the current state and action, i.e., $X_{t+1}^{(60)} \sim \PP_{\text{gen}}(x_t,a_t)$.

Since the key uncertainty lies with the single log return in the state, it only makes sense to have the Sinkhorn ball around the measure $\PP_{\text{gen}}(x_t,a_t)$ rather than the entire state space.
We define for all $x_t \in \X$ and $a_t \in A$, the ambiguity set
\begin{equation} \label{eq:port_opt_ambiguity_set}
    \mathcal{P}(x_t,a_t) := \left\{ \big(f_X(x_t,a_t,\cdot)\big)_{\#}\,\PP : \PP \in \Sball\left(\PP_{\text{gen}}(x_t,a_t)\right) \right\}
\end{equation}
that is, the image under $f_X$ of a Sinkhorn ball taken in the one-dimensional space of log returns.

The same map $f_X$ generates the candidate states that appear in the dual below: replacing the next log return in \eqref{eq:port_opt_fX} by a sample $Y \sim \nu$ yields the state $f_X(x_t,a_t,Y)$, which carries with it the position that $a_t$ would have drifted to under that sample.
}

\begin{prop} \label{prop:port_opt}
	In the setting of Section~\ref{sec:port_opt}, a corresponding Bellman equation holds true where
    \begin{equation} \label{eq:port_opt_bellman}
        \begin{split}
            (\widetilde{\HH}_\delta \widetilde{Q}_\delta^*)(x_t,a_t)&:=\inf_{\PP \in \mathcal{P}(x_t,a_t)} \E_{\PP} \left[r(x_t,a_t,X_{t+1})+\gamma \sup_{a_{t+1} \in A}\widetilde{Q}_\delta^*(X_{t+1},a_{t+1})\right]
            \\
            &=\widetilde{Q}_\delta^*(x_t,a_t)
        \end{split}
    \end{equation}
    where
    \begin{equation}
        \widetilde{Q}_\delta^*(x_t,a_t) := \inf_{\PP \in \mathcal{P}(x_t,a_t)} \E_{\PP} \left[r(x_t,a_t,X_{t+1})+\gamma \widetilde{V}_\delta^*(X_{t+1})\right]
    \end{equation}
    for a value function $\widetilde{V}_\delta^*(x_t)$ that is defined analogously to \eqref{eq:q_val_definition} for the ambiguity set from \eqref{eq:port_opt_ambiguity_set}.
    The corresponding duality for $\varepsilon>0$ also holds true where
    \begin{equation} \label{eq:port_opt_duality}
        \begin{split}
        &(\widetilde{\HH}_\delta \widetilde{Q}_\delta)(x_t,a_t) = \sup_{\lambda > 0} \Bigg\{ -\lambda\varepsilon - \lambda \delta \E_{X_{t+1} \sim \widehat{\PP}(x_t,a_t)} \Bigg[\log
        \\
        &\E_{Y \sim \nu}\left[\exp\left(\frac{-r(x_t,a_t,f_X(x_t,a_t,Y)) - \gamma \sup_{a_{t+1} \in A}\widetilde{Q}_\delta^*(f_X(x_t,a_t,Y),a_{t+1}) - \lambda \left|X_{t+1}^{(60)}- Y\right|}{\lambda\delta}\right)\right]\Bigg] \Bigg\}
        \end{split}
    \end{equation}
    when
    \begin{equation*}
		\widetilde{\varepsilon}:= \varepsilon + \delta \E_{X_{t+1} \sim \widehat{\PP}(x_t,a_t)} \left[\log\left(\E_{Y \sim \nu}\left[\exp\left(\frac{- \left|X_{t+1}^{(60)}- Y\right|}{ \delta}\right)\right]\right)\right]\geq 0.
	\end{equation*}
\end{prop}
The key change here versus Proposition~\ref{prop:dual} is that the cost $\left|X_{t+1}^{(60)}- Y\right|$ is only on the 1-dimensional log return rather than the entire state vector as a result of the infimum in \eqref{eq:port_opt_bellman} being taken over the ambiguity set $\mathcal{P}(x_t,a_t)$ rather than a Sinkhorn ball around the reference distribution $\widehat{\PP}(x_t,a_t)$ for the entire state space.
This means the choice of $\nu$ is only relevant for the log return and not the entire state space.

\subsubsection{Ambiguity in the distribution}
{
The learned generative model serves as the \emph{reference distribution} while the \emph{true distribution} is the actual S\&P 500 out-of-sample data.
For $\nu$, we chose the Student's t-distribution with the location-scale representation, i.e., $t(\mu,\sigma,\mathrm{df})$, where $\mu$ is the location parameter, $\sigma$ is the scale parameter and $\mathrm{df}$ is the degrees of freedom, truncated to the box $[\underline{y},\overline{y}] = [\log 0.05, \log 1.95]$ of Section~\ref{sec:port_opt}.
We fit the degrees of freedom to the data and choose $(\mu,\sigma)$ so that the \emph{truncated} law reproduces the first two sample moments of the training data.
The choice of the Student's t-distribution as a prior stems from the fact that it is a bell shaped distribution with heavier tails than the normal distribution, fitting stylised facts of financial data \citep{cont2001empirical}.
The tighter truncation compared to Section \ref{sec:stocks_opt} is because $A$ permits a short position of up to the full value of the portfolio.
When fully shorting the index, the gross return \eqref{eq:port_opt_gross} approaches zero as the index approaches a doubling, so an unbounded support would leave the reward unbounded below.
For the same reason, the tails of $\nu$ are economically important, which is why the far block of Section~\ref{sec:sampling_from_nu} is used in this experiment.
}

\subsubsection{The case for robustness}

The \emph{true} distribution is likely to deviate from the \emph{reference} distribution.
In other words, there will be a distributional shift that can cause a policy that performs well in the \emph{reference} distribution to perform poorly in the \emph{true} distribution.
The source of the distributional shift can be due to a number of reasons.
For example, the generative model may not be able to capture the true distribution of the S\&P 500 index.
This can be due to missing features, the model being too simple or simply that the market behaves differently for other reasons.

A policy trained to optimally exploit the \emph{reference} distribution may not be robust to this distributional shift.
Using a distributionally robust policy, such as the one we propose, can help mitigate this issue.

\subsubsection{Experiments}

{
We train DQN and RDQN agents on the simulator, i.e., the generative model from \citet[Section 5]{lu2024generative} using the same five seeds for every agent.
The DQN control uses the same implementation as RDQN with the only difference being the computation of the target in Algorithm~\ref{algo:robustdqn}.
Any difference between the two arms is therefore attributable to the target calculation alone.
For implementation details, we refer the reader to \url{https://github.com/luchungi/Sinkhorn-DRL}.

Post-training we evaluate every agent on the simulator it was trained against, which we refer to as the \emph{validation} evaluation, and on the realised S\&P 500 index, which is the \emph{out-of-sample} evaluation.
For the validation evaluation we generate 1,000 paths independent of the training data.
The out-of-sample evaluation replays the realised index over the five years from 1 Jan 2021 to 31 Dec 2025.
Performance is measured by the per annum log return, the volatility of log returns, and the Sharpe and Sortino ratios.

Table~\ref{tab:sim_perf} shows the validation results and we highlight two findings.
First, the simulator is exploitable as evidenced by the DQN agent earning more than double the index in terms of per annum log return, and it does so at a lower volatility.
Second, the per annum log return of the robust agents is generally lower than the DQN control, which is not surprising because it is the kind of conservatism the robust target is designed to induce.
However, what does not deteriorate is the risk-adjusted performance.
Every RDQN agent attains a Sortino ratio at least as high as the DQN control and the highest Sortino ratio is achieved by an RDQN agent.
}

\begin{table}[!htbp]
    \centering
    \begin{adjustbox}{max width=\linewidth}\begin{tabular}{lcccccc}
        \toprule
        Model & $\varepsilon$ & $\delta$ & Log return p.a. & Volatility & Sortino & Sharpe \\
        \midrule
        Simulator & - & - & 0.101 & 0.140 & 0.791 & 0.557 \\
        \midrule
        DQN & - & - & 0.231 & 0.116 & 2.862 & 1.823 \\
        \midrule
        RDQN & 2e-03 & 5e-05 & 0.167 & 0.081 & 2.878 & 1.793 \\
        RDQN & 3e-03 & 5e-05 & 0.138 & \textbf{0.060} & 3.069 & 1.892 \\
        RDQN & 2e-03 & 1e-04 & \textbf{0.191} & 0.082 & 3.293 & \textbf{2.042} \\
        RDQN & 3e-03 & 1e-04 & 0.131 & 0.062 & 2.933 & 1.739 \\
        RDQN & 3e-03 & 3e-04 & 0.154 & 0.064 & \textbf{3.444} & 2.035 \\
        \bottomrule
    \end{tabular}\end{adjustbox}
    \caption[Performance of DQN vs RDQN on simulated paths]{Performance on paths generated by the simulator trained on S\&P 500. The metrics are the mean of means across 5 independent seeds, where each seed's performance is the mean over 1,000 simulated paths. The log return is the per annum log return. The volatility is calculated as the standard deviation of the log returns over the entire episode. The Sharpe ratio is calculated as the mean excess log return divided by the standard deviation of the log returns. The Sortino ratio is the mean excess log return divided by the downside deviation. For the standard deviation and ratios, we annualise the values using 252 trading days.}
    \label{tab:sim_perf}
\end{table}

{
While the simulator captures some characteristics of the S\&P 500 index, there remains a distributional shift between it and the realised index, and it is the ability to survive that shift that we are interested in.
Therefore, we evaluate the same trained agents on the realised index over an out-of-sample period, again averaging over the five seeds, in Table~\ref{tab:spx_perf}.

\begin{table}[!htbp]
	\centering
    \begin{adjustbox}{max width=\linewidth}\begin{tabular}{lccccccc}
        \toprule
        Model & $\varepsilon$ & $\delta$ & Log return & Max drawdown & Volatility & Sortino & Sharpe \\
        \midrule
        S\&P 500 & - & - & 0.123 & -0.254 & 0.170 & 0.828 & 0.585 \\
        \midrule
        DQN & - & - & 0.079 & -0.177 & 0.146 & 0.534 & 0.367 \\
        \midrule
        RDQN & 2e-03 & 5e-05 & \textbf{0.079} & -0.117 & 0.109 & 0.752 & 0.507 \\
        RDQN & 3e-03 & 5e-05 & 0.068 & \textbf{-0.108} & \textbf{0.087} & \textbf{0.761} & \textbf{0.516} \\
        RDQN & 2e-03 & 1e-04 & 0.065 & -0.150 & 0.112 & 0.533 & 0.362 \\
        RDQN & 3e-03 & 1e-04 & 0.062 & -0.114 & 0.093 & 0.621 & 0.414 \\
        RDQN & 3e-03 & 3e-04 & 0.067 & -0.118 & 0.098 & 0.674 & 0.444 \\
        \bottomrule
    \end{tabular}\end{adjustbox}
    \caption[Performance of DQN vs RDQN on S\&P 500 out-of-sample]{Performance on the realised S\&P 500 from 1 Jan 2021 to 31 Dec 2025. This period is out-of-sample for the data used to train the simulator, from whose cutoff it is separated by several years. The values for the agents are the mean of 5 independent seeds. The metrics have the same meaning as in Table~\ref{tab:sim_perf}.}
	\label{tab:spx_perf}
\end{table}

\begin{figure} [!htbp]
	\centering
    \subfigure[RDQN agent]{
        \includegraphics[width=0.90\linewidth]{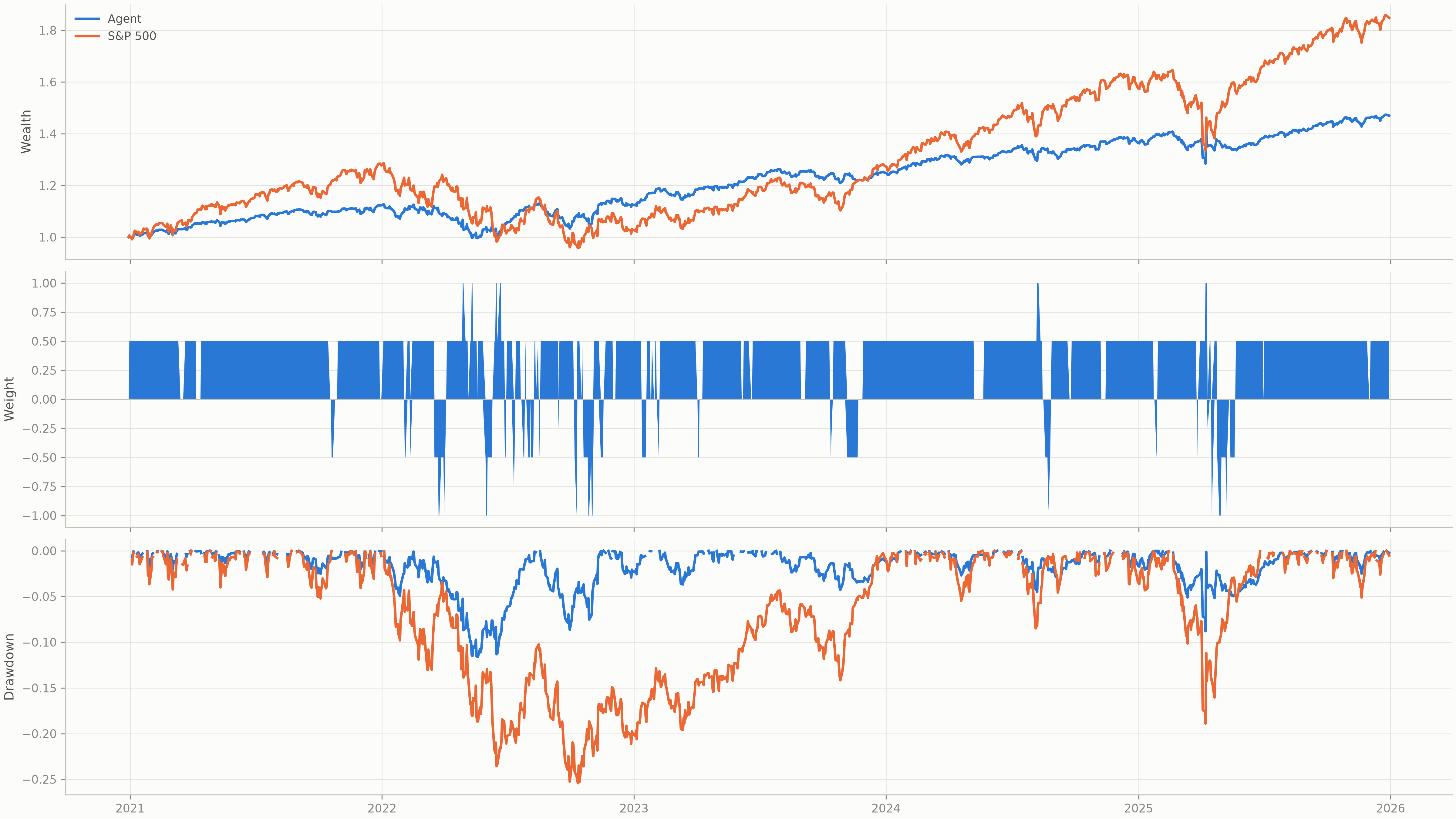}
        \label{fig:rdqn_spx}
    }
    \subfigure[DQN agent]{
        \includegraphics[width=0.90\linewidth]{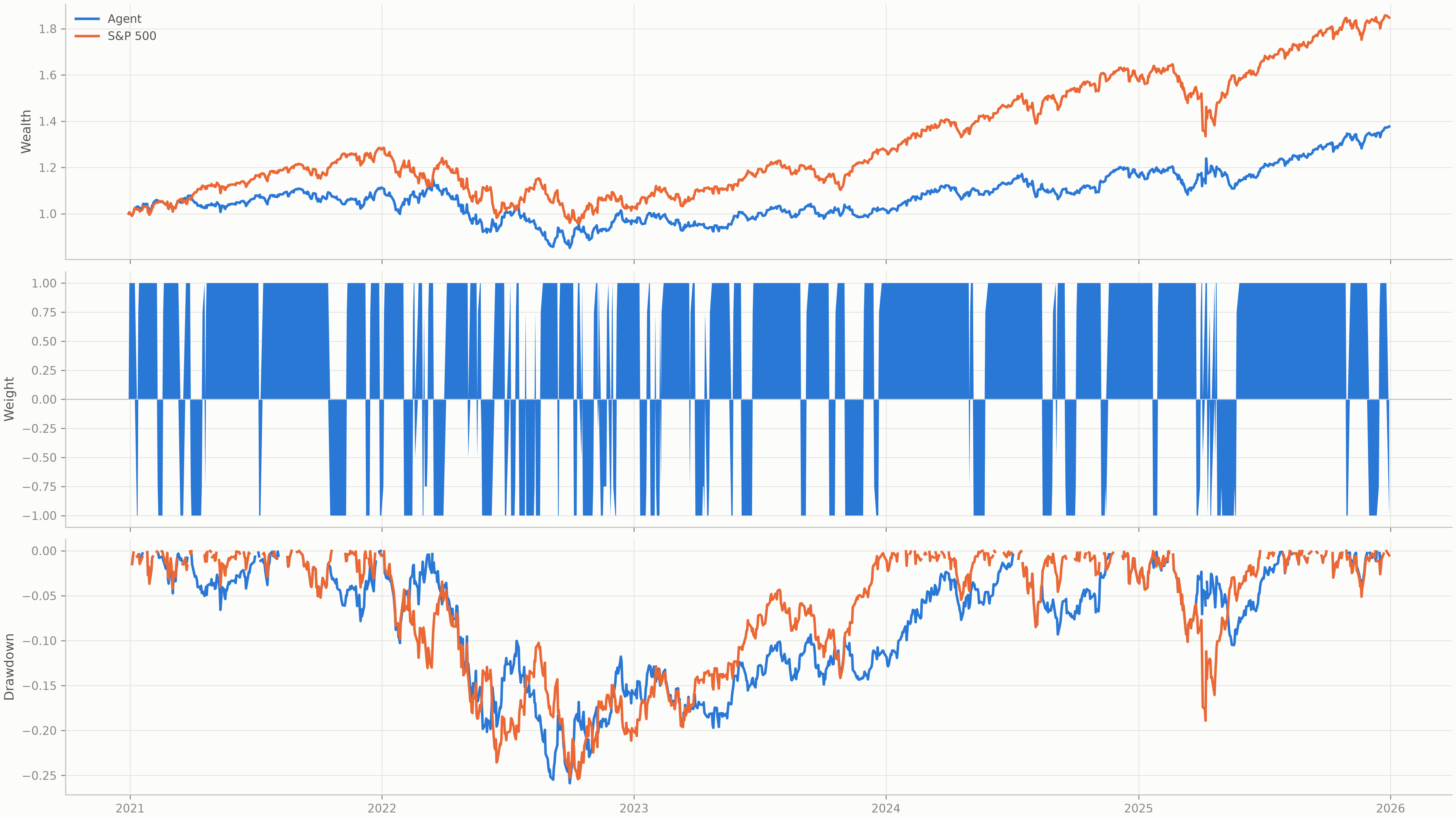}
        \label{fig:dqn_spx}
    }
    \caption[RDQN vs DQN on S\&P 500 out-of-sample]{Out-of-sample trajectories of the RDQN configuration with the highest log wealth per annum, $\varepsilon=0.002$, $\delta=0.00005$, against the DQN control. To avoid selecting on the outcome, each panel shows the \emph{median} of the five seeds ranked by log wealth. Within each figure, the top graph shows the wealth trajectory of the agent (blue line) and the S\&P 500 index (orange line). The middle graph shows the positions taken and the bottom graph shows the maximum drawdown of the agent (blue line) and the S\&P 500 index (orange line).}
    \label{fig:rdqn_vs_dqn_spx}
\end{figure}

The RDQN agents earn comparable or slightly lower returns than the DQN control but almost all the robust agents are better on every risk-adjusted measure.
Every RDQN agent has a smaller drawdown and a lower volatility than DQN and the best robust agent has significantly better Sharpe and Sortino ratios.
All agents earn less than the index and while the best robust configurations are comparable on risk-adjusted terms with a Sharpe ratio of $0.516$ and a Sortino ratio of $0.761$ against the index's $0.585$ and $0.828$, at roughly half the index's volatility and at $-0.108$ against $-0.254$ in maximum drawdown, there is still a shortfall in performance.
The obvious explanation for the shortfall in wealth is a distributional shift between the simulator and the market.
The same passive position earns $0.101$ per annum inside the simulator (Table~\ref{tab:sim_perf}) and $0.123$ in the realised market (Table~\ref{tab:spx_perf}) which means the simulator understates the drift of the evaluation period by about a fifth.
An agent that calibrated its exposure to the simulated world is therefore under-invested in the real one, and this is a shift that robustness cannot fix, because robustness moves the policy further towards caution, which is the wrong direction when the realised market turns out to be more rewarding than the reference measure.
What robustness does deliver is a policy whose risk characteristics survive the move, which is what the ratios and the drawdowns show.

Figure~\ref{fig:rdqn_vs_dqn_spx} visualises the differences in the actions of the agents.
We use the seed that produces the median log wealth among the five seeds for each agent to avoid selecting on the outcome.
The robust agent generally holds a much smaller exposure than the control but ends with higher wealth and lower drawdowns.
This is the qualitative signature of the robust target: a smaller and less extreme position, taken because the adversary in \eqref{eq:port_opt_duality} prices the possibility that the next return is worse than what the reference measure predicts.
}

\subsection{Multi-Asset Portfolio Optimisation with RTD3} \label{sec:stocks_opt}
{
The experiment of Section~\ref{sec:port_opt} has a one-dimensional uncertainty and a finite action set.
We now move to an allocation problem across five assets, where $A$ is a continuous action space and a simplex.
The RTD3 algorithm of Section~\ref{sec:robusttd3} is designed for this setting, and we use it to show that the robust target can be applied to a more complex environment.
The construction of the environment, the ambiguity set and the dual is the one already given in Section~\ref{sec:port_opt}, so we only describe what changes in the following subsections.

\subsubsection{The environment}

The universe is five US stocks Ford (F), Johnson \& Johnson (JNJ), JPMorgan Chase (JPM), Coca-Cola (KO) and Walmart (WMT).
The action is the vector of portfolio weights $a \in A := \left\{w \in \R^5 : w_i \geq 0,\ \sum_i w_i = 1\right\}$, so the agent is long-only and fully invested.
The reward is the log return of the portfolio net of transaction costs, defined exactly as in \eqref{eq:port_opt_reward}, with the gross return \eqref{eq:port_opt_gross} now summing over the five assets and the cost $k$ of \eqref{eq:port_opt_cost} solving the same self-financing fixed point.
The state consists of the daily log returns of the past 21 time steps for each of the five assets together with the five weights held before trading, so $\X \subset \R^{110}$.
Since the portfolio is always fully invested, the interest rate plays no part in the reward here but is still used in calculating the excess returns for the Sharpe and Sortino ratios at the same $r_f = 2.4\%$ used in Section~\ref{sec:port_opt}.
The transition map $f_X$ of \eqref{eq:port_opt_fX} carries over coordinatewise, the position component again being the weights that the action drifts to as prices move over the period.

The reference distribution is an extension of the simulator in \citet[Section 5]{lu2024generative} fitted to the realised daily log returns of the five stocks from 3 January 1991 to 30 June 2021.
We withhold 500 simulated paths for the validation evaluation.

\subsubsection{Ambiguity in the distribution}

The ambiguity set is \eqref{eq:port_opt_ambiguity_set} with the Sinkhorn ball now taken in the five-dimensional space of next-period log returns, and Proposition~\ref{prop:port_opt} applies except with the supremum over $A$ replaced by the continuation value of the target actor in the sense of Remark~\ref{rem:rtd3_operator}.

The reference distribution is the simulator and the true distribution is the realised price history of the five stocks.
Following Section~\ref{sec:sampling_from_nu}, $\nu$ is a multivariate Student's t-distribution truncated to a box, fitted by maximum likelihood to the realised daily log returns up to 30 June 2021.
The truncation boundaries can be wider than in Section~\ref{sec:port_opt} because the action set is long-only with no singularity of the kind a short position creates.
It is typically the case that single stocks have fatter tails than the index, so the truncation is also wider to accommodate this fact.
The box has only to be finite for Assumption~\ref{asu_2}~(i) to hold so we accordingly truncate to $[\log 0.001, \log 1000]$ per coordinate, and verify at run time that the mass it truncates is negligible, so that $\nu$ is to all intents the fitted law.
The consequence for the sampler is that the far block of Section~\ref{sec:sampling_from_nu} is not needed here.

\subsubsection{Experiments}

We train RTD3 and TD3 agents on the simulator, using the same five seeds for every agent.
As in Section~\ref{sec:port_opt} the control is the same implementation with the robust target replaced by the TD3 target, and every other component is shared by both algorithms.
Agents are evaluated on the 500 held-out paths, which we call the validation evaluation, and on the realised prices of the five stocks over the five years from 1 July 2021 to 30 June 2026.
We use two baselines for comparison: a daily rebalanced equally weighted portfolio and a buy-and-hold equally weighted portfolio.

Table~\ref{tab:sim_stocks_perf} shows the validation results and similar to the S\&P 500 experiment, the simulator is exploitable.
The TD3 control earns almost twice the return of the baselines at similar volatility with significantly better risk-adjusted performance.
However, the robust agents display a different pattern from that in Section~\ref{sec:port_opt}.
The small radius configurations earn more than the control with stronger risk-adjusted returns, but the large radius configurations are only able to achieve comparable returns to the baselines.

The reader will note that the $\varepsilon$ value is roughly an order of magnitude higher when the regularisation coefficient $\delta$ is increased by an order of magnitude.
Recall that the duality in Proposition~\ref{prop:port_opt} requires that $\bar{\varepsilon} \geq 0$.
This condition ensures there are feasible solutions to the optimisation problem and is an equivalent condition to the Sinkhorn ball being non-empty as discussed in Remark~\ref{rem:epsilon_min}.
As we increase $\delta$, the $\varepsilon$ value that keeps $\bar{\varepsilon}$ non-negative also increases, which is why the $\varepsilon$ values are higher for the larger $\delta$ configurations.

Table~\ref{tab:stocks_perf} shows the out-of-sample results.
The performance of the control is significantly worse than the baselines, which is a sign that the distributional shift is more severe than in Section~\ref{sec:port_opt}.
We see a reversal in fortunes where the large radius configurations handle the distributional shift much better than the small radius configurations.
The log returns and risk-adjusted measures of the large radius configurations are still, however, underperforming the baselines.

Figure~\ref{fig:rtd3_vs_td3_stocks} shows the two policies, again for the median performing seed in terms of log wealth.
The contrast is in the middle panel.
The robust agent holds all five names at all times, in weights that fluctuate around a diversified allocation, and its wealth tracks the equally weighted line from slightly below throughout the period.
The control swings between concentrated positions, frequently placing most of the portfolio in a single name.
Since both algorithms share the actor, the exploration noise and the entropy bonus, this difference in concentration is produced by the robust target itself and is the signature of the robust target: a more diversified allocation that is less sensitive to the distributional shift.
}

\begin{table} [!htbp]
    \begin{adjustbox}{max width=\linewidth}\begin{tabular}{lcccccc}
        \toprule
        Model & $\varepsilon$ & $\delta$ & Log return & Volatility & Sortino & Sharpe \\
        \midrule
        EW (rebalanced) & - & - & 0.096 & 0.162 & 0.832 & 0.512 \\
        BH (buy-and-hold) & - & - & 0.111 & 0.163 & 0.918 & 0.558 \\
        \midrule
        TD3 & - & - & 0.192 & 0.178 & 1.528 & 0.951 \\
        \midrule
        RTD3 & 5e-04 & 1e-05 & 0.230 & 0.188 & 1.781 & 1.104 \\
        RTD3 & 2e-03 & 5e-05 & \textbf{0.233} & 0.183 & \textbf{1.870} & \textbf{1.165} \\
        RTD3 & 2e-02 & 5e-04 & 0.113 & 0.173 & 0.895 & 0.556 \\
        RTD3 & 3e-02 & 5e-04 & 0.105 & 0.171 & 0.840 & 0.523 \\
        RTD3 & 4e-02 & 5e-04 & 0.099 & 0.172 & 0.791 & 0.489 \\
        RTD3 & 4e-02 & 1e-03 & 0.098 & \textbf{0.171} & 0.786 & 0.487 \\
        \bottomrule
    \end{tabular}\end{adjustbox}
    \caption[Performance of TD3 vs RTD3 and baseline strategies on the simulator's paths]{Performance on the 500 held-out simulated paths of F, JNJ, JPM, KO and WMT, whose conditioning windows are unseen during training. The metrics are the mean of means across 5 independent seeds, where each seed's performance is the mean over the 500 paths. The metrics have the same meaning as in Table~\ref{tab:sim_perf}.}
    \label{tab:sim_stocks_perf}
\end{table}

\begin{table} [!htbp]
    \begin{adjustbox}{max width=\linewidth}\begin{tabular}{lccccccc}
        \toprule
        Model & $\varepsilon$ & $\delta$ & Log return & Max drawdown & Volatility & Sortino & Sharpe \\
        \midrule
        EW (rebalanced) & - & - & 0.146 & -0.235 & 0.154 & 1.149 & 0.796 \\
        BH (buy-and-hold) & - & - & 0.133 & -0.256 & 0.152 & 1.037 & 0.725 \\
        \midrule
        TD3 & - & - & 0.066 & -0.234 & 0.181 & 0.329 & 0.231 \\
        \midrule
        RTD3 & 5e-04 & 1e-05 & 0.061 & -0.286 & 0.197 & 0.296 & 0.209 \\
        RTD3 & 2e-03 & 5e-05 & 0.051 & -0.279 & 0.186 & 0.230 & 0.161 \\
        RTD3 & 2e-02 & 5e-04 & 0.113 & -0.255 & 0.161 & 0.796 & 0.558 \\
        RTD3 & 3e-02 & 5e-04 & 0.124 & \textbf{-0.248} & 0.157 & 0.908 & 0.637 \\
        RTD3 & 4e-02 & 5e-04 & \textbf{0.124} & -0.260 & 0.158 & \textbf{0.909} & \textbf{0.638} \\
        RTD3 & 4e-02 & 1e-03 & 0.115 & -0.266 & \textbf{0.157} & 0.834 & 0.587 \\
        \bottomrule
    \end{tabular}\end{adjustbox}
    \caption[Performance of TD3 vs RTD3 and baseline strategies on out-of-sample data]{Performance on the realised prices of F, JNJ, JPM, KO and WMT from 1 July 2021 to 30 June 2026. This period is out-of-sample for the data used to train the simulator. The values for the agents are the mean of 5 independent seeds. The metrics have the same meaning as in Table~\ref{tab:sim_perf}.}
    \label{tab:stocks_perf}
\end{table}

\begin{figure} [!htbp]
	\centering
    \subfigure[RTD3 agent]{
        \includegraphics[width=0.90\linewidth]{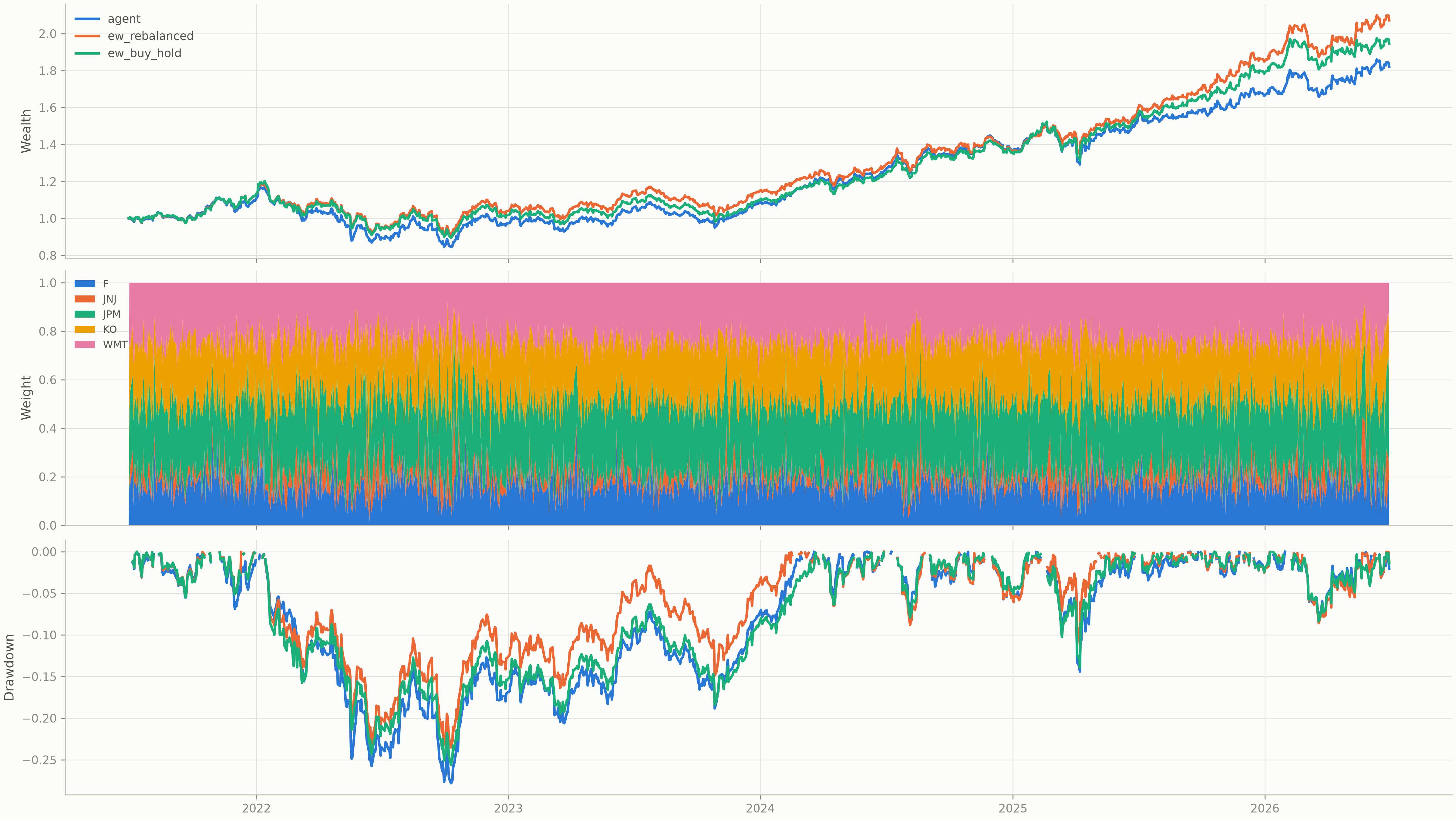}
        \label{fig:rtd3_stocks}
    }
    \subfigure[TD3 agent]{
        \includegraphics[width=0.90\linewidth]{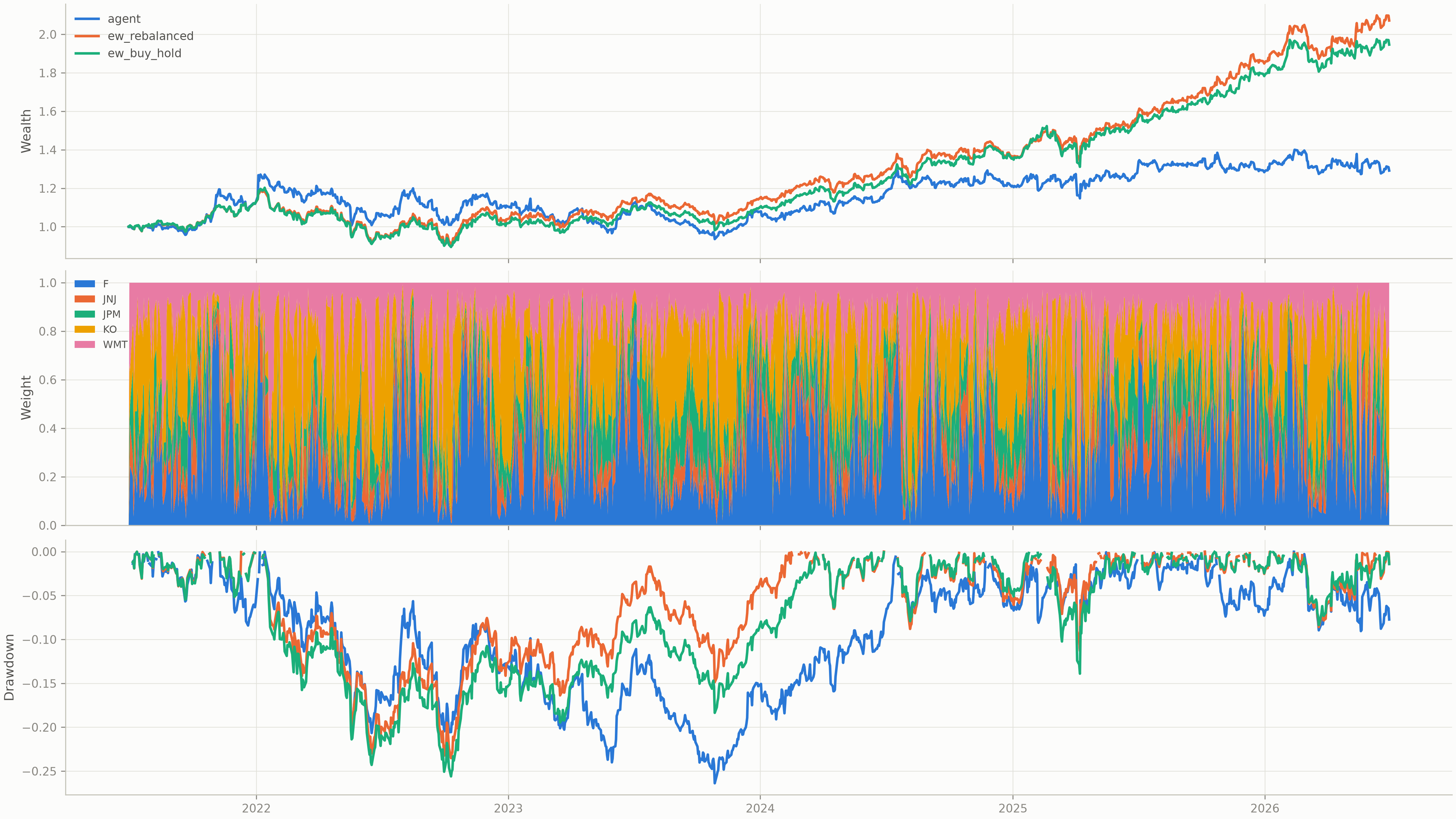}
        \label{fig:td3_stocks}
    }
    \caption[RTD3 vs TD3 on the 5 stocks out-of-sample]{Out-of-sample trajectories of the best performing RTD3 configuration, $\varepsilon=0.04$, $\delta=0.0005$, against the TD3 control. To avoid selecting on the outcome, each panel shows the \emph{median} of the five seeds ranked by log wealth. Within each figure, the top graph shows the wealth trajectory of the agent (blue line), equally weighted (balanced) strategy (orange line) and the equally weighted (buy and hold) strategy (green line). The middle graph shows the positions taken and the bottom graph shows the maximum drawdown of the agent (blue line) and the two baselines.}
    \label{fig:rtd3_vs_td3_stocks}
\end{figure}

\section{Conclusion} \label{sec:conclusion}

We introduced a novel distributionally robust $Q$-learning algorithm, Robust DQN (RDQN), for continuous state spaces and discrete action spaces subject to model uncertainty in the state transitions.
By formulating the problem within the framework of distributionally robust Markov decision processes and utilising the Sinkhorn distance to define an ambiguity set around a reference probability measure, we overcome the computational difficulties associated with the robust Bellman equation by the derivation of a more tractable dual optimisation problem.
We provided theoretical justification for our approach, showing in Proposition~\ref{prop:bellman} that the dynamic programming principle holds for our setting.

The resulting robust Q-function was then parameterised using deep neural networks, allowing for the adaptation of the successful Deep Q-Network (DQN) algorithm to the distributionally robust setting.
The RDQN algorithm, presented in Algorithm~\ref{algo:robustdqn}, leverages a modified target calculation based on the dual formulation and can be trained using standard deep learning techniques.
We provide theoretical guarantees for the existence of solutions in Proposition~\ref{prop:solutions} when the state space is compact.
{ We further showed that the robust Bellman operator is contractive, that exact robust $Q$-iteration converges uniformly to the robust fixed point, and that approximate robust fitted $Q$-iteration remains stable up to a neighbourhood determined by fitting, Monte Carlo, dual-optimisation, and training errors.}

Our empirical evaluations on a carefully designed gambling environment and on two real-world portfolio optimisation tasks demonstrated the practical viability and benefits of the algorithms.
In the gambling task discussed in Section~\ref{sec:toy}, the RDQN agent exhibited a greater resilience to unfavourable outcomes and achieved a higher mean reward per step compared to the standard DQN algorithm, particularly when the reward structure penalised wrong actions more heavily.
{
The two portfolio experiments are a natural setting for the use of distributionally robust reinforcement learning.
In both, the non-robust agent exploits its simulator convincingly, and in both that advantage does not survive distribution shift to the out-of-sample data.
In contrast, the RDQN and RTD3 agents are able to handle the distributional shift more robustly.
}

Despite these promising results, several avenues for future research remain.
{ The continuous-action extensions in Section~\ref{sec:continuous_action_extension} provide a first rigorous route through action discretisation and a scalable entropy-regularised route; future work can sharpen the latter in high-dimensional action spaces.}
{Choosing appropriate values for the Sinkhorn radius $\varepsilon$ and regularisation level $\delta$ in a principled manner is another area of interest, as the current approach relies on empirical tuning, and the two experiments here show that the choice matters more than any other we make.
The same algorithm is indistinguishable from its non-robust control at one $\varepsilon, \delta$ pair and clearly better than it at another.
Since the pair enters the dual jointly, a calibration procedure would have to select them together rather than one at a time.}

\section{Auxiliary Results and Proofs} \label{sec:aux_results_proofs}

\subsection{Auxiliary results}

\begin{lem} \label{lem:w1_convergence}
    The following statements are equivalent:
    \begin{enumerate}
        \item A sequence of measures $(\PP_n)_{n \in \N} \subset \MoneX$ converges to $\PP \in \MoneX$ in the Wasserstein-1 topology, i.e., $W_1(\PP_n,\PP) \to 0$ as $n \to \infty$.
        \item A sequence of measures $(\PP_n)_{n \in \N} \subset \MoneX$ converges to $\PP \in \MoneX$ weakly and converges in the first moment, i.e., $\int_{\X} \|x\| \D \PP_n(x) \to \int_{\X} \|x\| \D \PP(x)$ as $n \to \infty$.
    \end{enumerate}
    In other words, convergence in the Wasserstein-1 topology is equivalent to a combination of weak convergence and convergence of the first moment.
\end{lem}

\begin{proof}
    This follows from \citet[Definition 6.8]{villani2009optimal} combined with \citet[Theorem 6.9]{villani2009optimal}.
\end{proof}

\begin{lem} \label{lem:lower_semi_cont}
    {Let $\delta \geq 0$, and fix the prior measure $\nu \in \MoneX$.
    The extended-valued map $\MoneX \times \MoneX \ni (\PP_1,\PP_2) \mapsto W_\delta(\PP_1,\PP_2) \in [0,\infty]$ is jointly lower semi-continuous with respect to weak convergence.
    For $\delta>0$, we extend Definition~\ref{def:sinkhorn} using the convention $H(\alpha \mid \beta)=\infty$ whenever $\alpha$ is not absolutely continuous with respect to $\beta$; for $\delta=0$, we use $W_0=W$.}
\end{lem}

\begin{proof}
    {
    Let $\PP_i^{(n)}\to\PP_i$ weakly in $\MoneX$ for $i=1,2$, and set
    \begin{equation*}
        L:=\liminf_{n\to\infty}W_\delta(\PP_1^{(n)},\PP_2^{(n)}).
    \end{equation*}
    If $L=\infty$, the desired inequality is immediate.
    Otherwise, choose a subsequence $(n_k)$ such that the finite values
    $W_\delta(\PP_1^{(n_k)},\PP_2^{(n_k)})$ converge to $L$.
    Define the objectives with their respective first marginals by
    \begin{equation*}
        f_n(\pi):=
        \begin{cases}
            \displaystyle\int_{\X \times \X}\|x-y\| \D\pi(x,y)
            +\delta H(\pi \mid \PP_1^{(n)} \otimes \nu), & \delta>0,\\
            \displaystyle\int_{\X \times \X}\|x-y\| \D\pi(x,y), & \delta=0.
        \end{cases}
    \end{equation*}
    For each $k$, choose $\pi_k \in \Pi(\PP_1^{(n_k)},\PP_2^{(n_k)})$ such that
    \begin{equation*}
        W_\delta(\PP_1^{(n_k)},\PP_2^{(n_k)})
        \leq f_{n_k}(\pi_k)
        \leq W_\delta(\PP_1^{(n_k)},\PP_2^{(n_k)})+\frac1k.
    \end{equation*}
    In particular, $f_{n_k}(\pi_k)\to L$.

    Weak convergence of both marginals implies that their families, and hence the couplings $(\pi_k)$, are tight (see \citet[Lemma 4.4]{villani2009optimal}).
    Since $\X$ is a closed subset of $\R^d$, Prokhorov's theorem gives a further subsequence, without changing notation, such that $\pi_k\to\pi$ weakly in $\mathcal{M}_1(\X \times \X)$.
    For every $g \in C_b(\X)$,
    \begin{equation}
        \begin{split}
            \int_{\X \times \X}g(x)\D\pi(x,y)
            &=\lim_{k\to\infty}\int_{\X \times \X}g(x)\D\pi_k(x,y)\\
            &=\lim_{k\to\infty}\int_{\X}g(x)\D\PP_1^{(n_k)}(x)
            =\int_{\X}g(x)\D\PP_1(x).
        \end{split}
    \end{equation}
    The same argument for the second coordinate shows that $\pi \in \Pi(\PP_1,\PP_2)$.

    The nonnegative continuous cost satisfies, by Portmanteau's theorem,
    \begin{equation*}
        \int_{\X \times \X}\|x-y\|\D\pi(x,y)
        \leq\liminf_{k\to\infty}\int_{\X \times \X}\|x-y\|\D\pi_k(x,y).
    \end{equation*}
    Relative entropy is jointly weakly lower semicontinuous: its variational representation
    \begin{equation*}
        H(\alpha \mid \beta)
        =\sup_{g \in C_b(\X \times \X)}
        \left\{\int g\D\alpha-\log\int e^g\D\beta\right\}
    \end{equation*}
    expresses it as a supremum of jointly weakly continuous functions.
    Since $\PP_1^{(n_k)} \otimes \nu\to\PP_1 \otimes \nu$ weakly, we therefore have
    \begin{equation*}
        H(\pi \mid \PP_1 \otimes \nu)
        \leq\liminf_{k\to\infty}H(\pi_k \mid \PP_1^{(n_k)} \otimes \nu).
    \end{equation*}
    Let $f$ be the objective defined as above with $\PP_1$ in place of $\PP_1^{(n)}$.
    Combining these inequalities, using nonnegativity of both terms and omitting the entropy term when $\delta=0$, gives
    \begin{equation*}
        W_\delta(\PP_1,\PP_2)
        \leq f(\pi)
        \leq\liminf_{k\to\infty}f_{n_k}(\pi_k)
        =L.
    \end{equation*}
    This proves joint lower semi-continuity.
    }
\end{proof}

\begin{lem} \label{lem:sinkhorn_ball_finite_first_moment}
	Let $\delta \geq 0$ and $\varepsilon > 0$. Let $\widehat{\PP}(x,a)$ have a finite first moment for all $(x,a) \in \X \times A$, then for all $(x,a) \in \X \times A$ and $\PP \in \Sball(\widehat{\PP}(x,a))$, it holds that $\int_{\X} \|x\| \D\PP(x) \leq \int_{\X} \|y\| \D\widehat{\PP}(x,a)(y) + \varepsilon$.
\end{lem}

\begin{proof}
    By our definition of the Sinkhorn distance, we have
    \begin{equation*}
        W(\widehat{\PP}(x,a),\PP) \leq W_{\delta}(\widehat{\PP}(x,a),\PP)
    \end{equation*}
    where $W$ is the standard Wasserstein-1 distance since $H(\pi|\widehat{\PP}(x,a) \otimes \nu) \geq 0$.
    Therefore, for any $\PP \in \Sball(\widehat{\PP}(x,a))$, we have
    \begin{equation}
        W(\widehat{\PP}(x,a),\PP) \leq \varepsilon.
    \end{equation}

    Due to the metric property of the Wasserstein distance and the Kantorovich-Rubinstein duality \citep[Theorem 5.10]{villani2009optimal}, we have
    \begin{equation*}
        W(\widehat{\PP}(x,a),\PP) = W(\PP,\widehat{\PP}(x,a)) = \sup_{g \in \mathcal{L}_1} \left\{\int_{\X} g(x) \D\PP(x) - \int_{\X} g(y) \D \widehat{\PP}(x,a)(y) \right\}
    \end{equation*}
    where $\mathcal{L}_1$ is the set of Lipschitz continuous functions with Lipschitz constant 1.
    Since $g(x)=\|x\|$ is Lipschitz continuous with Lipschitz constant 1, we must have
    \begin{equation*}
        \int_{\X} \|x\| \D\PP(x) - \int_{\X} \|y\| \D \widehat{\PP}(x,a)(y) \leq W(\widehat{\PP}(x,a),\PP) \leq \varepsilon
    \end{equation*}
    which implies
    \begin{equation*}
        \int_{\X} \|x\| \D\PP(x) \leq \int_{\X} \|y\| \D\widehat{\PP}(x,a)(y) + \varepsilon
    \end{equation*}
    Since $\widehat{\PP}(x,a)$ has a finite first moment, we have
    \begin{equation} \label{eq:finite_first_moment}
        \int_{\X} \|x\| \D\PP(x) < \infty
    \end{equation}
    Therefore, all $\PP \in \Sball(\widehat{\PP}(x,a))$ have a finite first moment which is uniformly bounded.
\end{proof}

\begin{lem} \label{lem:weak_compact}
    Let $\delta \geq 0$ and $\varepsilon > 0$. Let $\widehat{\PP}(x,a)$ have a finite first moment for all $(x,a) \in \X \times A$, then the Sinkhorn ball $\Sball(\widehat{\PP}(x,a))$ is weakly compact for all $(x,a) \in \X \times A$.
\end{lem}

\begin{proof}
    The proof follows \citet[Theorem 1]{yue2022linear}.
    First, we show that $\Sball(\widehat{\PP}(x,a))$ is tight.
    Due to Lemma~\ref{lem:sinkhorn_ball_finite_first_moment} and Assumption~\ref{asu_1}, we can find $C>0$ such that for all $\PP \in \Sball(\widehat{\PP}(x,a))$, $\int_{\X} \|x\| \D\PP(x) \leq C$.
    Then for any $\eta > 0$, we have a compact set $B(\frac{C}{\eta}) = \{ x \in \X~|~\|x\| \leq \frac{C}{\eta} \}$ such that
    \begin{equation*}
        \PP(\X \setminus B(\frac{C}{\eta})) \leq \frac{\int_{\X} \|x\| \D\PP(x)}{\frac{C}{\eta}} \leq \eta
    \end{equation*}
    $\PP \in \Sball(\widehat{\PP}(x,a))$, where the first inequality is due to Markov's inequality which shows tightness.

    Next, due to Lemma~\ref{lem:lower_semi_cont}, if we have a sequence $(\PP^{(n)})_{n \in \N} \subset \Sball(\widehat{\PP}(x,a))$ such that $\PP^{(n)} \to \PP$ weakly as $n \to \infty$, then we have
    \begin{equation*}
        W_{\delta}(\widehat{\PP}(x,a),\PP) \leq \liminf_{n \to \infty} W_{\delta}(\widehat{\PP}(x,a),\PP^{(n)}) \leq \varepsilon
    \end{equation*}
    which implies $\PP \in \Sball(\widehat{\PP}(x,a))$.
    Therefore, the Sinkhorn ball $\Sball(\widehat{\PP}(x,a))$ is weakly closed.
    We can now conclude that $\Sball(\widehat{\PP}(x,a))$ is weakly compact by Prokhorov's theorem.
\end{proof}

{
\begin{lem} \label{lem:sinkhorn_continuity}
    Let $\delta \geq 0$, and let $\nu \in \MoneX$ be the fixed prior measure from Definition~\ref{def:sinkhorn}.
    Let $\PP_1,\PP_1^{(n)},\PP_2 \in \MoneX$ have finite first moments, and assume that $W(\PP_1^{(n)},\PP_1) \to 0$ as $n \to \infty$.
    If $\delta>0$, assume additionally that $H(\PP_2 \mid \nu)<\infty$.
    Then all the following Sinkhorn distances are finite, and
    \begin{equation*}
        \left|W_\delta(\PP_1^{(n)},\PP_2)-W_\delta(\PP_1,\PP_2)\right|
        \leq W(\PP_1^{(n)},\PP_1).
    \end{equation*}
    In particular, $\lim_{n \to \infty} W_\delta(\PP_1^{(n)},\PP_2)=W_\delta(\PP_1,\PP_2)$.
\end{lem}

\begin{proof}
    For $\delta=0$, finiteness follows from the first-moment assumptions, and the claimed estimate is the reverse triangle inequality for the Wasserstein-1 distance.
    Hence, suppose that $\delta>0$.
    For any $\mu \in \MoneX$ with a finite first moment, define
    \begin{equation} \label{eq:cuturi_sinkhorn}
        S_\delta(\mu,\PP_2)
        :=\inf_{\pi \in \Pi(\mu,\PP_2)}
        \left\{\int_{\X \times \X}\|x-y\| \D\pi(x,y)
        +\delta H(\pi \mid \mu \otimes \PP_2)\right\}.
    \end{equation}
    The relative-entropy chain rule gives, for every $\pi \in \Pi(\mu,\PP_2)$,
    \begin{equation*}
        H(\pi \mid \mu \otimes \nu)
        =H(\pi \mid \mu \otimes \PP_2)+H(\PP_2 \mid \nu).
    \end{equation*}
    Taking the infimum yields
    $W_\delta(\mu,\PP_2)=S_\delta(\mu,\PP_2)+\delta H(\PP_2 \mid \nu)$.
    The product coupling $\mu \otimes \PP_2$ has zero relative entropy with respect to itself and finite transport cost, so both distances are finite.

    Let $\mu,\mu' \in \MoneX$ have finite first moments, and fix $\eta>0$.
    Choose $\pi \in \Pi(\mu,\PP_2)$ such that
    \begin{equation*}
        \int_{\X \times \X}\|x-y\| \D\pi(x,y)
        +\delta H(\pi \mid \mu \otimes \nu)
        \leq W_\delta(\mu,\PP_2)+\eta.
    \end{equation*}
    Let $\tau \in \Pi(\mu,\mu')$ be an optimal coupling for $W(\mu,\mu')$.
    Since $\X$ is a closed subset of $\R^d$, we can disintegrate it as
    $\tau(\D x,\D x')=\mu(\D x)K(x,\D x')$ for a probability kernel $K$.
    Form the joint law $\pi(\D x,\D y)K(x,\D x')$, and denote its $(x',y)$-marginal by $\widetilde{\pi} \in \Pi(\mu',\PP_2)$.
    Applying $K$ to the first coordinate maps $\mu \otimes \nu$ to $\mu' \otimes \nu$.
    Thus, the data-processing inequality for relative entropy gives
    \begin{equation*}
        H(\widetilde{\pi} \mid \mu' \otimes \nu)
        \leq H(\pi \mid \mu \otimes \nu).
    \end{equation*}
    Moreover, the triangle inequality implies
    \begin{equation*}
        \int_{\X \times \X}\|x'-y\| \D\widetilde{\pi}(x',y)
        \leq \int_{\X \times \X}\|x-y\| \D\pi(x,y)+W(\mu,\mu').
    \end{equation*}
    Combining these estimates yields
    \begin{equation*}
        W_\delta(\mu',\PP_2)
        \leq W_\delta(\mu,\PP_2)+\eta+W(\mu,\mu').
    \end{equation*}
    Letting $\eta \downarrow 0$ and interchanging $\mu$ and $\mu'$ proves
    $|W_\delta(\mu',\PP_2)-W_\delta(\mu,\PP_2)|\leq W(\mu,\mu')$.
    Applying this estimate with $\mu=\PP_1$ and $\mu'=\PP_1^{(n)}$ proves the result.
\end{proof}
}

\begin{lem} \label{lem:uhc_sinkhorn_ball}
    Let $\delta \geq 0$ and $\varepsilon > 0$. The set valued map $F:\X \times A \ni (x,a) \twoheadrightarrow \Sball(\widehat{\PP}(x,a))$ is upper hemicontinuous.
\end{lem}

\begin{proof}
    Let $(x_n,a_n)_{n \in \N} \subseteq \X \times A$ such that $(x_n,a_n) \to (x,a) \in (\X,A)$ converge as $n \to \infty$ and let $(\PP_n)_{n \in \N} \subset \Mone$ where $\PP_n \in \Sball(\widehat{\PP}(x_n,a_n))$ for all $n \in \N$.
    We want to show that $\PP_n \to \PP$ weakly as $n \to \infty$ for some $\PP \in \Sball(\widehat{\PP}(x,a))$, then upper hemicontinuity follows from \citet[Theorem 17.20]{aliprantis2006infinite}.

    First, we show that $(\PP_n)_{n \in \N}$ is tight.
    By Assumption~\ref{asu_1} and Lemma~\ref{lem:sinkhorn_ball_finite_first_moment} we have that $\PP_n$ has a uniformly bounded first moment for all $n \in \N$.
    In other words, there exists $C>0$ such that
    \begin{equation*}
        \sup_{n \in \N} \int_{\X} \|x\| \D\PP_n(x) \leq C.
    \end{equation*}
    Then for any $\eta > 0$, we can find a compact set $B(\frac{C}{\eta}) := \{ x \in \X~|~\|x\| \leq \frac{C}{\eta} \}$ such that for all $n \in \N$,
    \begin{equation*}
        \PP_n(\X \setminus B(\frac{C}{\eta})) \leq \frac{\int_{\X} \|x\| \D\PP_n(x)}{\frac{C}{\eta}} \leq \eta
    \end{equation*}
    where the first inequality is due to Markov's inequality.
    Therefore, $(\PP_n)_{n \in \N}$ is tight.
    By Prokhorov's theorem, there exists a subsequence $(\PP_{n_k})_{k \in \N}$ such that $\PP_{n_k} \to \PP$ weakly as $k \to \infty$ for some $\PP \in \Mone$.

    Since we assume that the map $(x,a) \mapsto \widehat{\PP}(x,a)$ is continuous, we have $\widehat{\PP}(x_{n_k},a_{n_k}) \to \widehat{\PP}(x,a)$ weakly as $k \to \infty$.
    This combined with Lemma~\ref{lem:lower_semi_cont}, we have
    \begin{equation*}
        W_{\delta}(\widehat{\PP}(x,a),\PP) \leq \liminf_{k \to \infty} W_{\delta}(\widehat{\PP}(x_{n_k},a_{n_k}),\PP_{n_k}) \leq \varepsilon.
    \end{equation*}
    Therefore, we have $\PP \in \Sball(\widehat{\PP}(x,a))$.
\end{proof}

\begin{lem} \label{lem:lhc_sinkhorn_ball}
    Let $\delta \geq 0$ and $\varepsilon > 0$. {Assume that Assumptions~\ref{asu_1} and~\ref{asu_3} hold.} The set valued map $F:\X \times A \ni (x,a) \twoheadrightarrow \Sball(\widehat{\PP}(x,a))$ is lower hemicontinuous.
\end{lem}

\begin{proof}
    {
    Let $(x_n,a_n)\to(x,a)$ in $\X \times A$, and fix an arbitrary $\PP \in F(x,a)$.
    Write $\mu_n:=\widehat{\PP}(x_n,a_n)$ and $\mu:=\widehat{\PP}(x,a)$.
    By Assumption~\ref{asu_1}, $d_n:=W(\mu_n,\mu)\to0$.
    Assumption~\ref{asu_3} gives $\mu_n \in F(x_n,a_n)$ and $\mu \in F(x,a)$.
    We construct $\PP^{(n)} \in F(x_n,a_n)$ converging weakly to $\PP$.

    Every $Q \in F(x,a)$ has a finite first moment by Lemma~\ref{lem:sinkhorn_ball_finite_first_moment}.
    If $\delta>0$, projection onto the second coordinate and the data-processing inequality give
    \begin{equation*}
        \delta H(Q \mid \nu)\leq W_\delta(\mu,Q)\leq\varepsilon.
    \end{equation*}
    Thus, the Lipschitz estimate of Lemma~\ref{lem:sinkhorn_continuity} applies with any such $Q$ as the second marginal.
    Also, $W_\delta(\mu,\cdot)$ is convex: mixing couplings preserves the first marginal $\mu$, the transport cost is linear, and relative entropy is convex.

    First, suppose that there is a strictly feasible measure $\PP^0$, with
    $g:=\varepsilon-W_\delta(\mu,\PP^0)>0$.
    Choose $N$ such that $d_n<g$ for all $n\geq N$, and define
    \begin{equation*}
        \PP^{(n)}:=
        \begin{cases}
            \mu_n, & n<N,\\
            (1-t_n)\PP+t_n\PP^0,\quad t_n:=d_n/g, & n\geq N.
        \end{cases}
    \end{equation*}
    For $n\geq N$, convexity gives
    $W_\delta(\mu,\PP^{(n)})\leq\varepsilon-t_ng\leq\varepsilon$,
    so $\PP^{(n)} \in F(x,a)$.
    Lemma~\ref{lem:sinkhorn_continuity} therefore yields
    \begin{align*}
        W_\delta(\mu_n,\PP^{(n)})
        &\leq W_\delta(\mu,\PP^{(n)})+d_n\\
        &\leq (1-t_n)W_\delta(\mu,\PP)+t_nW_\delta(\mu,\PP^0)+d_n\\
        &\leq\varepsilon-t_ng+d_n=\varepsilon.
    \end{align*}
    Hence $\PP^{(n)} \in F(x_n,a_n)$ for all $n$, and $t_n\to0$ implies $\PP^{(n)}\to\PP$ weakly.

    Next, suppose that no strictly feasible measure exists.
    Since $W_0(\mu,\mu)=0<\varepsilon$, this case requires $\delta>0$.
    Every $Q \in F(x,a)$ then satisfies $W_\delta(\mu,Q)=\varepsilon$.
    The infimum defining this value is attained, because $\Pi(\mu,Q)$ is weakly compact and the transport-plus-entropy objective is weakly lower semicontinuous.
    If $Q\neq\mu$ were feasible, let $\pi$ and $\pi_0$ be minimizing couplings for $W_\delta(\mu,Q)$ and $W_\delta(\mu,\mu)$, respectively.
    These couplings are distinct and have finite objective value $\varepsilon$.
    Strict convexity of relative entropy with respect to the fixed measure $\mu \otimes \nu$, together with linearity of the transport cost, implies that $(\pi+\pi_0)/2$ has objective strictly below $\varepsilon$.
    Its second marginal $(Q+\mu)/2$ would therefore be strictly feasible, a contradiction.
    Thus $F(x,a)=\{\mu\}$, and setting $\PP^{(n)}:=\mu_n$ gives $\PP^{(n)} \in F(x_n,a_n)$ and $\PP^{(n)}\to\mu=\PP$ weakly.

    In both cases, every $\PP \in F(x,a)$ admits such a sequence of feasible approximations.
    Lower hemicontinuity follows from \citet[Theorem 17.21]{aliprantis2006infinite}.
    }
\end{proof}

\begin{lem} \label{lem:Q_continuous}
{
	Assume that Assumptions \ref{asu_1} and \ref{asu_2} hold true, and let $\delta \geq 0$ and $\varepsilon > 0$. Then
	\begin{itemize}
	\item[(i)]     $\HH_\delta C_b(\X \times A) \subseteq C_b(\X \times A)$.
	\item[(ii)] In particular,
	the map
	\begin{equation*}
		\X \times A \ni (x,a) \mapsto Q_\delta^*(x,a)= \inf_{\PP \in \Sball \left(\widehat{\PP}(x,a)\right)} \E_{\PP} \left[r(x,a,X_1)+\gamma V_\delta(X_1)\right]\in \R
	\end{equation*}
	is continuous and the minimum is attained by some $\PP \in \Sball \left(\widehat{\PP}(x,a)\right)$.
\end{itemize}	 }
	\end{lem}
\begin{proof}[Proof of Lemma~\ref{lem:Q_continuous}]
{ Let $Q \in C_b(\X \times A)$ and define
\[
    M_Q(y):=\sup_{b\in A}Q(y,b),
    \qquad y\in\X .
\]
Since $A$ is compact and $Q$ is continuous on $\X\times A$, Berge's maximum
theorem (\citet[Theorem 17.31]{aliprantis2006infinite}) applied to the constant compact correspondence $y  \twoheadrightarrow A$ implies
that $M_Q$ is continuous.}
 We then define the map
	\begin{equation} \label{eq_f_continuous_contraction_proof}
		\begin{split}
		F:\operatorname{Gr} \mathcal{P} = \{(x,a_0,\PP_0)~|~x\in \X,a_0\in A, \PP_0\in  \Sball \left(\widehat{\PP}(x,a_0)\right)\} &\rightarrow \R
		\\
		(x,a_0,\PP_0) &\mapsto \E_{\PP_0}\left[r(x,a_0,X_1)+\gamma  {M_Q(X_1)}\right].
		\end{split}
	\end{equation}
	and claim that the map $F$ is continuous. As in \citet[Proof of Theorem 2.7~(i)]{neufeld2023markov}, the imposed assumptions on the reward function imply that $F$ is continuous.
	Indeed, let $(x^{(n)},a_0^{(n)},\PP_0^{(n)}) \subseteq \operatorname{Gr} \mathcal{P}$ with $(x^{(n)},a_0^{(n)},\PP_0^{(n)}) \rightarrow (x,a_0,\PP_0)\in \X \times A \times \MoneX$ for $n \rightarrow \infty$, where the convergence of $\PP_0^{(n)}$ is in the weak topology.
	Since $\X \times A \ni (\widetilde{x},\widetilde{a}) \mapsto \Sball \left(\widehat{\PP}(x,a_0)\right)$ is compact-valued and continuous, we have $\PP_0 \in \Sball \left(\widehat{\PP}(x,a_0)\right)$.
	Moreover,
	\begin{align}
		&|F(x^{(n)},a_0^{(n)},\PP_0^{(n)})-F(x,a_0,\PP_0)| \notag
		\\
		\leq &|F(x^{(n)},a_0^{(n)},\PP_0^{(n)})-F(x,a_0,\PP_0^{(n)})| +|F(x,a_0,\PP_0^{(n)})-F(x,a_0,\PP_0)|. \label{eq_F_ineq}
	\end{align}
	The second summand $|F(x,a_0,\PP_0^{(n)})-F(x,a_0,\PP_0)|$ vanishes for $n \rightarrow \infty$ since the integrand $z \mapsto r(x,a_0,z)+\gamma  V(z)$ is continuous and bounded by Assumption~\ref{asu_2}~(i) and since $\PP_0^{(n)} \rightarrow \PP_0$ weakly.
	For the first summand of \eqref{eq_F_ineq} we obtain, by using Assumption~\ref{asu_2}~(ii), that
	\begin{align*}
		\lim_{n \rightarrow \infty}&\left|F(x^{(n)},a_0^{(n)},\PP_0^{(n)})-F(x,a_0,\PP_0^{(n)})\right|
		\\
		&\leq \lim_{n \rightarrow \infty}  \E_{\PP_0^{(n)}}\left[ \left|r(x^{(n)},a_0^{(n)},X_1)-r(x,a_0,X_1)\right|\right]\leq \lim_{n \rightarrow \infty} L (\left\|x^{(n)}-x\right\|+\left\|a_0^{(n)}-a_0\right\|) =0.
	\end{align*}
	Hence $F$ is continuous and by an application of Berge's maximum theorem (\citet[Theorem 17.31]{aliprantis2006infinite}) we get that the map
	\begin{equation}\label{eq_continuous_berge}
		\begin{split}
		Q_\delta^*:\X \times A  &\rightarrow \R \\
		(x,a_0) &\mapsto \inf_{\PP_0 \in \Sball \left(\widehat{\PP}(x,a)\right)} F(x,a_0,\PP_0)
		\end{split}
	\end{equation}
	is continuous, too, and that minimisers exist.{ This shows (i)}
	\\
	To show (ii) note that by \citet[Proposition 3.1]{neufeld2023markov} the Wasserstein ambiguity set $\X \times A  \ni (x,a) \twoheadrightarrow  \mathcal{B}_{\varepsilon,0} \left(\widehat{\PP}(x,a)\right)$ fulfils the assumptions of \citet[Assumption 2.2]{neufeld2023markov}.
	By Lemma~\ref{lem:weak_compact}, \ref{lem:uhc_sinkhorn_ball} and \ref{lem:lhc_sinkhorn_ball}, we also have that the Sinkhorn ambiguity set $(x,a) \twoheadrightarrow  \Sball \left(\widehat{\PP}(x,a)\right)$ fulfils the assumptions of \citet[Assumption 2.2]{neufeld2023markov}.
	Hence, under the imposed assumptions $V_\delta$ from \eqref{eq:robust_problem_1} is continuous and bounded, see \citet[Theorem 2.7]{neufeld2023markov}, { and repeating the argument leading to (i) with $V_\delta$ instead of $M_Q$ shows the assertion}.
\end{proof}

{
\begin{lem}[Contraction of the robust Bellman operator]\label{lem:robust_bellman_contraction}
For all $Q_1,Q_2\in\calQ$,
\begin{equation}\label{eq:robust_contraction_bound}
    \norm{\HH_\delta Q_1-\HH_\delta Q_2}_\infty
    \leq
    \gamma\norm{Q_1-Q_2}_\infty .
\end{equation}
Hence $\HH_\delta$ is a $\gamma$-contraction on $\calQ$.
\end{lem}

\begin{proof}[Proof of Lemma~\ref{lem:robust_bellman_contraction}]
For every $y\in \X$,
\begin{equation*}
    \left|\sup_{b\in A}Q_1(y,b)-\sup_{b\in A}Q_2(y,b)\right|
    \leq
    \sup_{b\in A}|Q_1(y,b)-Q_2(y,b)|
    \leq
    \norm{Q_1-Q_2}_\infty .
\end{equation*}
Fix $(x,a)\in\X\times A$ and, for $i\in\{1,2\}$, write
\begin{equation*}
    F_i(\PP):=
    \E_{\PP}\left[r(x,a,X_1)+\gamma \sup_{b\in A}Q_i(X_1,b)\right],
    \qquad
    \PP\in \Sball\left(\widehat{\PP}(x,a)\right).
\end{equation*}
Then \begin{equation}\label{eq:F_12}
|F_1(\PP)-F_2(\PP)|=\gamma \left|\E_{\PP}\left[\sup_{b\in A}Q_1(X_1,b)-\sup_{b\in A}Q_2(X_1,b)\right]\right|\leq\gamma\norm{Q_1-Q_2}_\infty
\end{equation} for every  $\PP\in \Sball\left(\widehat{\PP}(x,a)\right)$.
Using
\begin{equation*}
    \left|\inf_{\PP\in\Sball(\widehat{\PP}(x,a))}F_1(\PP)-\inf_{\PP\in\Sball(\widehat{\PP}(x,a))}F_2(\PP)\right|
    \leq
    \sup_{\PP\in\Sball(\widehat{\PP}(x,a))}|F_1(\PP)-F_2(\PP)|,
\end{equation*}
and \eqref{eq:F_12}, we get
\begin{equation*}
\begin{aligned}
    |(\HH_\delta Q_1)(x,a)-(\HH_\delta Q_2)(x,a)| &=    \left|\inf_{\PP\in\Sball(\widehat{\PP}(x,a))}F_1(\PP)-\inf_{\PP\in\Sball(\widehat{\PP}(x,a))}F_2(\PP)\right|\\
    &\leq \sup_{\PP\in\Sball(\widehat{\PP}(x,a))}|F_1(\PP)-F_2(\PP)| \leq     \gamma\norm{Q_1-Q_2}_\infty .
    \end{aligned}
\end{equation*}
Taking the supremum over $(x,a)$ proves \eqref{eq:robust_contraction_bound}.
\end{proof}
}

{
\begin{lem} \label{lem:gross_return_floor}
	Consider the setting of Section~\ref{sec:port_opt}.
	Let $A \subset [-1,1]$ be compact, let the log returns take values in a compact interval $[\underline{y},\overline{y}]$ with $e^{\overline{y}} < 2$, and let the time delta take values in a compact interval $[0,\Delta_{\max}] \subset \R^+$.
	Set
	\begin{equation} \label{eq:gross_return_floor}
		G_{\min} := \min\left\{e^{\underline{y}},\, 2-e^{\overline{y}}\right\} > 0,
		\quad
		G_{\max} := 1 + \max\left\{1-e^{\underline{y}},\, e^{\overline{y}}-1\right\} + 2\left(e^{r_f \Delta_{\max}}-1\right),
		\quad
		W := \frac{e^{\overline{y}}}{G_{\min}},
	\end{equation}
	and assume the transaction cost satisfies $c\,(1+W) < 1-c$.
	Then
	\begin{enumerate}
		\item[(i)] $G(x_t,a_t,y)$ of \eqref{eq:port_opt_gross} satisfies
		$G_{\min} \leq G(x_t,a_t,y) \leq G_{\max}$ for all $x_t \in \X$, $a_t \in A$ and $y \in [\underline{y},\overline{y}]$;
        \item[(ii)] the map $f_X$ of \eqref{eq:port_opt_fX} is continuous, and its position component is bounded in absolute value by $W$, so $\X$ is bounded;
		\item[(iii)] the fixed point \eqref{eq:port_opt_cost} has a unique solution $k(x_t,a_t) \in [0,\, c(1+W)/(1-c)]$, so that $1-k$ is bounded away from $0$, and $(x_t,a_t) \mapsto k(x_t,a_t)$ is Lipschitz;
		\item[(iv)] the reward $r'$ of \eqref{eq:port_opt_reward} is continuous, bounded, and Lipschitz in $(x_t,a_t)$.
	\end{enumerate}
\end{lem}

\begin{proof}

	(i) Write $\rho := e^{y}-1$ and $\iota := e^{r_f x_t^{(62)}}-1$, so that $G = 1 + a_t \rho + (1-a_t)\iota$.
	Both $\rho$ and $\iota$ range over compact sets and $a_t$ over the compact set $A$, which gives the upper bound.
	For the lower bound, $\iota \geq 0$ and $1-a_t \geq 0$ give $G \geq 1 + a_t \rho$, and $|a_t| \leq 1$ gives $G \geq 1 - |\rho|$.
	Since $|\rho| \leq \max\{1-e^{\underline{y}},\, e^{\overline{y}}-1\}$ we obtain $G \geq \min\{e^{\underline{y}},\, 2-e^{\overline{y}}\} = G_{\min}$, which is positive because $e^{\overline{y}} < 2$.

	(ii) The position component of $f_X$ is $a_t e^{y}/G(x_t,a_t,y)$, which is continuous because $G \geq G_{\min} > 0$ by (i), and bounded in absolute value by $e^{\overline{y}}/G_{\min} = W$ because $|a_t| \leq 1$.
	The remaining components of $f_X$ are coordinate projections, the log return $y \in [\underline{y},\overline{y}]$ and the time delta in $[0,\Delta_{\max}]$, all continuous and bounded.
	Since every initial position lies in $A \subset [-1,1]$ and $W \geq 1$, every state reachable under $f_X$ lies in a bounded set.

	(iii) Let $\Phi(k) := c\left|(1-k)a_t - x_t^{(61)}\right|$.
	For any $k,k'$ we have $|\Phi(k)-\Phi(k')| \leq c\,|a_t|\,|k-k'| \leq c\,|k-k'|$, so $\Phi$ is a contraction for $c<1$ and has a unique fixed point by Banach's fixed point theorem.
	It is non-negative because $\Phi \geq 0$.
	From $k = \Phi(k) \leq c\left(|a_t| + |k|\,|a_t| + |x_t^{(61)}|\right) \leq c\left(1 + k + W\right)$, using (ii) for the bound $|x_t^{(61)}| \leq W$, we get $k \leq c(1+W)/(1-c) < 1$.
	Since $\Phi$ depends on $(x_t,a_t)$ in a Lipschitz manner and is a uniform contraction in $k$, the fixed point is Lipschitz in $(x_t,a_t)$.

	(iv) By (i) and (iii) the argument $(1-k)G$ of the logarithm in \eqref{eq:port_opt_reward} takes values in the compact interval
	\begin{equation*}
		\left[\left(1-\tfrac{c(1+W)}{1-c}\right) G_{\min},\; G_{\max}\right] \subset (0,\infty),
	\end{equation*}
	on which $\log$ is Lipschitz with constant the reciprocal of the left endpoint.
	The map $(x_t,a_t,y) \mapsto G(x_t,a_t,y)$ is smooth with bounded derivatives on the compact domain, and $k$ is Lipschitz by (iii), so $(1-k)G$ is Lipschitz.
	Composing the two gives the claim.
\end{proof}
}

\subsection{Proofs} \label{sec:proofs}

\begin{proof}[Proof of Proposition~\ref{prop:bellman}]
	By Lemma~\ref{lem:weak_compact}, the set $\Sball\left(\widehat{\PP}(x,a)\right)$ is weakly compact and by Lemma~\ref{lem:uhc_sinkhorn_ball} and \ref{lem:lhc_sinkhorn_ball} the set valued map $F:\X \times A \ni (x,a) \twoheadrightarrow \Sball(\widehat{\PP}(x,a))$ is continuous.

	This together with Assumptions \ref{asu_1}, \ref{asu_2} and \ref{asu_3} ensures that the assumptions of \citet[Theorem 2.7]{neufeld2023markov} are satisfied.
	By \citet[Theorem 2.7]{neufeld2023markov}, the Bellman equation holds true.
\end{proof}

{
\begin{proof}[Proof of Theorem~\ref{thm:exact_robust_q_iteration}]
By Lemma~\ref{lem:robust_bellman_contraction} and Lemma~\ref{lem:Q_continuous}, $\HH_\delta$ is a $\gamma$-contraction on the complete metric space $(\calQ,\norm{\cdot}_\infty)$. Since $0<\gamma<1$, Banach's fixed-point theorem implies that $\HH_\delta$ has a unique fixed point $Q_\delta^\ast\in\calQ$ and that the Picard iteration $Q_{k+1}=\HH_\delta Q_k$ converges to it in the supremum norm.  The rate from \eqref{eq:exact_robust_iteration_rate} follows by iterating
\begin{equation*}
    \norm{Q_k-Q_\delta^\ast}_\infty
    =
    \norm{\HH_\delta Q_{k-1}-\HH_\delta Q_\delta^\ast}_\infty
    \leq
    \gamma\norm{Q_{k-1}-Q_\delta^\ast}_\infty .
\end{equation*}
\end{proof}}

\begin{proof}[Proof of Proposition~\ref{prop:solutions}~(i)]
{Since $\mathcal X$ is closed and bounded and $A$ is finite,
$\mathcal X\times A$ is compact. By Lemma~\ref{lem:Q_continuous}, the fixed point
$Q_\delta^\ast$ is continuous and bounded on $\mathcal X\times\mathcal A$.
Therefore, by the universal approximation theorem, Proposition~\ref{prop:universal}, there
exists a neural network
\[
    Q_{\mathrm{NN}}^\ast \in \mathcal N_{d+m,1}
\]
such that
\[
    \|Q_{\mathrm{NN}}^\ast-Q_\delta^\ast\|_\infty
    <
    \frac{\mathrm{TOL}}{1+\gamma}.
\]
Since $Q_\delta^\ast$ is the fixed point of $\mathcal{H}_\delta$, we have
$\mathcal{H}_\delta Q_\delta^\ast=Q_\delta^\ast$. Hence, by the triangle inequality,
\[
\begin{aligned}
    \|\mathcal{H}_\delta Q_{\mathrm{NN}}^\ast-Q_{\mathrm{NN}}^\ast\|_\infty
    &\leq
    \|\mathcal{H}_\delta Q_{\mathrm{NN}}^\ast-\mathcal{H}_\delta Q_\delta^\ast\|_\infty
    +
    \|Q_\delta^\ast-Q_{\mathrm{NN}}^\ast\|_\infty .
\end{aligned}
\]
By Lemma~\ref{lem:robust_bellman_contraction}, $\mathcal{H}_\delta$ is a $\gamma$-contraction, so
\[
    \|\mathcal{H}_\delta Q_{\mathrm{NN}}^\ast-\mathcal{H}_\delta Q_\delta^\ast\|_\infty
    \leq
    \gamma
    \|Q_{\mathrm{NN}}^\ast-Q_\delta^\ast\|_\infty .
\]
Consequently,
\[
\begin{aligned}
    \|\mathcal{H}_\delta Q_{\mathrm{NN}}^\ast-Q_{\mathrm{NN}}^\ast\|_\infty
    &\leq
    (1+\gamma)\|Q_{\mathrm{NN}}^\ast-Q_\delta^\ast\|_\infty  \\
    &<
    (1+\gamma)\frac{\mathrm{TOL}}{1+\gamma}
    =
    \mathrm{TOL}.
\end{aligned}
\]
Therefore $Q_{\mathrm{NN}}^\ast$ is a solution to Optimisation Problem~\ref{opt:problem1}. }

\end{proof}

\begin{proof}[Proof of Proposition~\ref{prop:solutions}~(ii)]
	Recall that for any solution to Optimisation Problem~\ref{opt:problem1}, we have
	\begin{equation} \label{eq:opt_soln}
		\sup_{(x,a) \in \X \times A} | \HH_\delta Q^*_{\rm NN}(x,a)-Q_{\rm NN}^*(x,a) | < \rm TOL.
	\end{equation}
	Using the triangle inequality together {with the fact that by Lemma~\ref{lem:robust_bellman_contraction}, $\mathcal{H}_\delta$ is a $\gamma$-contraction}  and with \eqref{eq:opt_soln}, we have
	\begin{equation*}
		\begin{split}
		|Q^*_{\rm NN}(x,a)-Q_\delta^*(x,a)| &\leq |\HH_\delta Q^*_{\rm NN}(x,a) - Q_{\rm NN}^*(x,a)| + |\HH_{\delta} Q_{\rm NN}^*(x,a) - Q_\delta^*(x,a)|
		\\
		&= |\HH_\delta Q^*_{\rm NN}(x,a) - Q_{\rm NN}^*(x,a)| + |\HH_\delta Q_\delta^*(x,a) - \HH_{\delta} Q_{\rm NN}^*(x,a)|
		\\
		&\leq \rm TOL + \gamma \sup_{(x',a')\in \X \times A}\left|Q^*_{\rm NN}(x',a') -Q_\delta^*(x',a')\right|
		\end{split}
	\end{equation*}
	which is true for all $(x,a) \in \X \times A$ hence it is true for the supremum of $(x,a)$ over $\X \times A$.
	Therefore, we get
	\begin{equation*}
		\sup_{(x',a')\in \X \times A}\left|Q^*_{\rm NN}(x',a') -Q_\delta^*(x',a')\right| \leq \frac{\rm TOL}{1-\gamma}
	\end{equation*}
    which shows any solution to Optimisation Problem~\ref{opt:problem1} is $\frac{\rm TOL}{1-\gamma}$-close to $Q_\delta^*$ in the supremum norm over $\X \times A$.
\end{proof}

\begin{proof}[Proof of Proposition~\ref{prop:solutions}~(iii)]
    First, note that by Corollary~\ref{cor:convergence_delta} we have for all $(x,a) \in \X \times A$ that
	\begin{equation*}
		\left|\HH_{\delta}Q_0^*(x,a) -\HH_0 Q_0{^*}(x,a)\right| \rightarrow 0 \text{ as } \delta \downarrow 0.
	\end{equation*}
	We also note for $\delta' > \delta$ we have by the definition of the Sinkhorn distance $\mathcal{B}_{\varepsilon,\delta'}\left(\widehat{\PP}(x,a)\right)\subseteq \Sball\left(\widehat{\PP}(x,a)\right)$
	and thus
	\begin{align*}
		\HH_{\delta} Q_0^*(x,a)&= \inf_{\PP \in \Sball\left(\widehat{\PP}(x,a)\right)} \E_{\PP} \left[r(x,a,X_1)+\gamma \sup_{b \in A}Q_0^*(X_1,b)\right]
		\\
		&\leq \inf_{\PP \in \mathcal{B}_{\varepsilon,\delta'}\left(\widehat{\PP}(x,a)\right)} \E_{\PP} \left[r(x,a,X_1)+\gamma \sup_{b \in A} Q_0^*(X_1,b)\right] = \HH_{\delta'} Q_0^*(x,a).
	\end{align*}
	This means the pointwise convergence $\HH_\delta Q_0^* \rightarrow \HH_0 Q_0^* $ as $\delta \downarrow 0$ is monotone, and the limit $\HH_0 Q_0^*= Q_0^*$ is continuous by Lemma~\ref{lem:Q_continuous}.
	Hence, by Dini's Theorem (\citet[Theorem 7.13]{rudin1964principles}), the convergence is uniform on the compact set $\X \times A$, and we can find some $\delta'$ so that we have for all $\delta<\delta'$
	\begin{equation} \label{eq:Dini}
		\left|\HH_{\delta}Q_0{^*}(x,a) -\HH_0 Q_0{^*}(x,a)\right| < \rm TOL
	\end{equation}
	for all $(x,a) \in \X \times A$.

    According to \citet[Theorem I~(iii)]{wang2021sinkhorn}, the condition $\overline{\varepsilon}>0$ ensures that $\HH_{\delta} Q_0^*(x,a)>-\infty$.
	This means by Proposition~\ref{prop:dual} we have that
	\begin{equation} \label{eq:H_0Q_0_lambda}
		\begin{split}
		&\HH_{\delta}Q_0^*(x,a)
		\\
		&= \sup_{\lambda >0} \bigg\{-\lambda {\varepsilon}- \lambda \delta \E_{X_1^{\PP} \sim \widehat{\PP}(x,a)}\left[\log\left(\E_{X_1^{\nu} \sim \nu}\left[\exp\left(\tfrac{-r(x,a,X_1^{\nu})-\gamma \sup_{b \in A}Q_0^*(X_1^{\nu},b)-\lambda \| X_1^{\PP}-X_1^{\nu}\|}{\lambda \delta}\right)\right]\right)\right]\bigg\}
		\\
		&= \sup_{\lambda >0} \left\{-\lambda \overline{\varepsilon} - \lambda \delta \E_{X_1^{\PP} \sim \widehat{\PP}(x,a)}\left[\log\left(\E_{X_1^{\Q} \sim \Q_{x,\delta}}\left[\exp\left(\frac{-r(x,a,X_1^{\Q})-\gamma \sup_{b \in A}Q_0^*(X_1^{\Q},b)}{\lambda \delta}\right)\right]\right)\right]\right\}
		\end{split}
	\end{equation}
	for $\D \Q_{x,\delta}(z):= \frac{\exp(-\|x-z\|/\delta)}{\E_{X^\nu \sim \nu}[\exp(-\|x-X^{\nu}\|/\delta)]} \D \nu (z)$.
	\vspace{3pt}

	{By Lemma~\ref{lem:robust_bellman_contraction}, we have}
	\begin{equation} \label{eq:HQ0_contraction}
		|\HH_{\delta} Q_{\rm NN}^*(x,a) - \HH_{\delta} Q_0^*(x,a)| \leq \gamma\sup_{(x'a')\in \X \times A}\left|Q^*_{\rm NN}(x',a') -Q_0^*(x',a')\right|
	\end{equation}

    {Since $Q_0^*$ is continuous by Lemma~\ref{lem:Q_continuous} and $\X \times A$ is compact, Proposition~\ref{prop:universal} provides some $Q^*_{\rm NN} \in \mathfrak{N}_{d + m ,1}$ such that}
	\begin{equation} \label{eq:Q0_approx_q_NN}
		\sup_{(x,a) \in \X \times A}| Q^*_{\rm NN}(x,a)-Q_0^*(x,a)| < \rm TOL
	\end{equation}

	Then, by \eqref{eq:opt_soln}, \eqref{eq:Dini} and \eqref{eq:HQ0_contraction}, we obtain for all $(x,a) \in \X \times A$ that
	\begin{align*}
		&| Q^*_{\rm NN}(x,a) -Q_0^*(x,a) |
		\\
		&=| Q^*_{\rm NN}(x,a) - \HH_0 Q_0{^*}(x,a) |
		\\
		&\leq | Q^*_{\rm NN}(x,a) - \HH_{\delta}Q^*_{\rm NN}(x,a) | + | \HH_{\delta}Q^*_{\rm NN}(x,a) - \HH_0 Q_0{^*}(x,a) |
		\\
		&\leq | Q^*_{\rm NN}(x,a) - \HH_{\delta}Q^*_{\rm NN}(x,a) | + | \HH_{\delta}Q^*_{\rm NN}(x,a) - \HH_{\delta} Q_0^*(x,a) | + | \HH_{\delta} Q_0^*(x,a) -\HH_0 Q_0{^*}(x,a) |
		\\
		&<\rm TOL +\gamma\sup_{(x'a')\in \X \times A}|Q^*_{\rm NN}(x',a') -Q_0{^*}(x',a')|+TOL.
	\end{align*}
	Hence, we have
	\begin{align*}
		&\sup_{(x'a')\in \X \times A}|Q^*_{\rm NN}(x',a') -Q_0^*(x',a')| < \frac{2 {\rm TOL}}{1-\gamma}.
	\end{align*}
\end{proof}

\begin{proof}[Proof of Lemma~\ref{lem:toy_assump}]
	First, we note that for all $(x,a) \in \X \times A$, the reference measure $\widehat{\PP}(x,a) = \text{Beta}(\alpha,\beta)$ where $\alpha=g(\alpha' - a x)$ and $\beta=g(\beta' + a(1 - x))$ for $g(x) = \log(1+e^x)$ if $a \neq 0$ and $\alpha=\alpha'$, $\beta=\beta'$ if $a=0$ for some $\alpha',\beta' \in \R^+$.

To verify Assumption~\ref{asu_1}, we aim at applying Lemma~\ref{lem:w1_convergence}. To this end, let  $(x_n,a_n)_{n \in \N} \subset \X \times A$ such that $\lim_{n \rightarrow \infty} (x_n,a_n)=(x,a) \in \X \times A$.  Since the probability density function of the Beta distribution is continuous in its parameters, the densities of $\widehat{\PP}(x_n,a_n)$ converge pointwise to the density of $\widehat{\PP}(x,a)$ as $n \rightarrow \infty$.
    By \citet[Theorems 18.1, 18.5]{jacod2012probability}, the weak convergence $\widehat{\PP}(x_n,a_n) \rightarrow  \widehat{\PP}(x,a)$ follows then by the pointwise convergence of their densities.
	The first moment of the Beta distribution is
    \begin{equation}
        \E_{Y \sim \widehat{\PP}(x,a)}[Y] = \frac{\alpha}{\alpha+\beta} = \frac{g(\alpha' - a x)}{g(\alpha' - a x) + g(\beta' + a(1 - x))}
    \end{equation}
    which is continuous in $(x,a)$. Therefore, the first moment converges to the first moment of $\widehat{\PP}(x,a)$ as $(x_n,a_n)_{n \in \N} \in \X \times A$ converges to $(x,a) \in \X \times A$.
    By Lemma~\ref{lem:w1_convergence}, the map $(x,a) \mapsto \widehat{\PP}(x,a)$ is continuous in the Wasserstein-1 topology.
	Therefore, Assumption~\ref{asu_1} is satisfied.

    The reward function is $r(x_0,a,x_1) = a(x_1 - x_0)(1 + (M-1)\one_{a(x_1 - x_0) < 0})$ where $M \geq 1$ is the reward factor, $x_0,x_1 \in [0,1]$ and actions $a \in \{-1,0,1\}$.
	We note that $\lim_{(x_1-x_0) \rightarrow 0^+} r(x_0,a,x_1) = 0 = \lim_{(x_1-x_0) \rightarrow 0^-} r(x_0,a,x_1)$.
	Therefore, the reward function is continuous.
	Since the domain is bounded, the reward function is also bounded, hence Assumption~\ref{asu_2} is satisfied.
\end{proof}
{
\begin{proof}[Proof of Theorem~\ref{thm:approx_robust_fitted_q_iteration}]
Let $d_k:=\norm{Q_k-Q_\delta^\ast}_\infty$.  Using \eqref{eq:approx_step_error_main}, the fixed-point identity $\HH_\delta Q_\delta^\ast=Q_\delta^\ast$, and the contraction property, we obtain
\begin{align*}
    d_{k+1}
    &=
    \norm{Q_{k+1}-Q_\delta^\ast}_\infty \\
    &\leq
    \norm{Q_{k+1}-\HH_\delta Q_k}_\infty
    +
    \norm{\HH_\delta Q_k-\HH_\delta Q_\delta^\ast}_\infty \\
    &\leq
    \eta_k+\gamma d_k .
\end{align*}
Applying this recursion gives
\begin{equation*}
    d_K
    \leq
    \gamma^K d_0+
    \sum_{k=0}^{K-1}\gamma^{K-1-k}\eta_k,
\end{equation*}
which proves \eqref{eq:approx_fitted_q_bound_main}.  If $\eta_k\leq\eta$, the geometric-series bound yields \eqref{eq:uniform_eta_bound_main}.
\end{proof}

\begin{proof}[Proof of Corollary~\ref{cor:rdqn_error_decomposition}]
The uniform bounds and \eqref{eq:rdqn_error_decomp_main} imply
\begin{equation*}
    \eta_k\leq
    \eta_{\mathrm{fit}}+\eta_{\mathrm{MC}}+\eta_{\lambda}+\eta_{\mathrm{opt}}+\eta_{\mathrm{struct}}
    =:\eta .
\end{equation*}
Substituting this value into \eqref{eq:uniform_eta_bound_main} gives \eqref{eq:rdqn_error_bound_main}.
\end{proof}
}

\begin{proof}[Proof of Proposition~\ref{prop:port_opt}]
    In the first part, we want to show that the set valued map
    \begin{equation} \label{eq:port_opt_sinkhorn_ball_2}
        \X \times A \ni (x,a) \mapsto \mathcal{P}(x,a) \subset \MoneX
    \end{equation}
    where $\mathcal{P}(x,a)$ is as defined in \eqref{eq:port_opt_ambiguity_set}, is weakly compact and continuous.
    First, we show, by applying Lemma~\ref{lem:w1_convergence}, that the map
    \begin{equation} \label{eq:port_opt_gen_map}
        \X \times A \ni (x,a) \mapsto \PP_{\text{gen}}(x,a) \in \mathcal{M}_1(\R)
    \end{equation}
    is continuous in the Wasserstein-1 topology $\tau_1$.
    The generative model is a neural network as described in \citet[Section 4]{lu2024generative}.
    Note that the measure actually does not depend on the action $a$.
    Given the current state $x \in \X$, the generative model maps from a 4-dimensional Gaussian random variable to a 1-dimensional log return.
    Let $Z$ be the 4-dimensional Gaussian random variable, and let the output of the generative model in dependence of state $x \in \X$ and realization of the Gaussian random variable $Z=z$ be defined via
    \begin{equation}
        \X \times \R^4 \ni (x,z) \mapsto f_{\theta,x}(z) \in \R
    \end{equation}
    where $f_{\theta,x}$ is a neural network with parameters $\theta$.
    Let $\mu \in \mathcal{M}_1(\R^4)$ be the probability measure of the 4-dimensional Gaussian random variable $Z$.
    If we have $(x_n)_{n \in \N} \subset \X$ such that $\lim_{n \rightarrow \infty} x_n = x\in \X$ then to show weak convergence of the generative model, we need to show that for any continuous and bounded function $g \in C_b(\R)$, we have
    \begin{equation}
        \lim_{n \rightarrow \infty} \int_{\R^4} g(f_{\theta,x_n}(z)) d\mu(z) = \int_{\R^4} g(f_{\theta,x}(z)) d\mu(z).
    \end{equation}
    The function $f_{\theta,x}$ is continuous in $x$ as the neural network is a composite of continuous functions and since $g$ is continuous and bounded, we can apply the dominated convergence theorem to show that the above limit holds.
    Similarly, our assumption that log returns are bounded implies that we have convergence in the first moment.
    This shows that the map \eqref{eq:port_opt_gen_map} is continuous in the Wasserstein-1 topology due to Lemma~\ref{lem:w1_convergence}.

    {Next, by Lemma~\ref{lem:gross_return_floor}~(ii) the map $f_X$ of \eqref{eq:port_opt_fX} is continuous on $\X \times A \times [\underline{y},\overline{y}]$, and bounded.
    Consequently the pushforward map
    \begin{equation} \label{eq:port_opt_pushforward_map}
        \X \times A \times \mathcal{M}_1(\R) \ni (x,a,\PP) \mapsto \big(f_X(x,a,\cdot)\big)_{\#}\PP \in \MoneX
    \end{equation}
    is continuous in the Wasserstein-1 topology since pushing forward along a continuous bounded map preserves weak convergence, and boundedness gives convergence of the first moment, so Lemma~\ref{lem:w1_convergence} applies.}

    In addition, by Lemma~\ref{lem:weak_compact}, Lemma~\ref{lem:uhc_sinkhorn_ball} and Lemma~\ref{lem:lhc_sinkhorn_ball}, the set valued map
    \begin{equation} \label{eq:port_opt_sinkhorn_ball}
        \X \times A \ni (x,a) \mapsto \mathcal{B}_{\varepsilon,\delta}(\PP_{\text{gen}}(x,a)) \subseteq \mathcal{M}_1(\R)
    \end{equation}
    is weakly compact and continuous.
    {Since $\mathcal{P}(x,a)$ is by \eqref{eq:port_opt_ambiguity_set} the image of $\mathcal{B}_{\varepsilon,\delta}(\PP_{\text{gen}}(x,a))$ under the continuous map \eqref{eq:port_opt_pushforward_map}, and continuous images of weakly compact sets are weakly compact, we can} apply the same arguments as in \citet[Lemma 6.1, Proposition 3.1]{neufeld2023markov} to conclude that the set valued map \eqref{eq:port_opt_sinkhorn_ball_2} is weakly compact and continuous.

    In the second part, we want to show that Assumption~\ref{asu_2} is satisfied.
	The reward is the log return of the portfolio for the period when the state transitions from $x_t$ to $x_{t+1}$.
	{Recall that the log return $x_{t+1}^{(60)}$ of the underlying asset, which we require to be bounded, and the current weight of the portfolio in the asset $x_t^{(61)}$, are both part of the state.
	The reward function is
    \begin{equation*}
        r(x_t,a_t,x_{t+1}) = r'\!\left(x_t,a_t,x_{t+1}^{(60)}\right) = \log\Big(\big(1-k(x_t,a_t)\big)\,G\!\left(x_t,a_t,x_{t+1}^{(60)}\right)\Big).
    \end{equation*}
	By Lemma~\ref{lem:gross_return_floor}~(i) and~(iii), the gross return satisfies $G \geq G_{\min} > 0$ and the rebalancing cost satisfies $1-k \geq 1 - c(1+W)/(1-c) > 0$, so the argument of the logarithm takes values in a compact subset of $(0,\infty)$.
	The reward is therefore continuous and bounded, which is Assumption~\ref{asu_2}~(i).

	For Assumption~\ref{asu_2}~(ii), Lemma~\ref{lem:gross_return_floor}~(iv) gives a finite Lipschitz constant for $r'$ in $(x_t,a_t)$, uniformly in the log return: the logarithm is Lipschitz on the compact range just identified, $G$ is smooth with bounded derivatives on the compact domain, and $k$ is Lipschitz as the fixed point of a uniform contraction with Lipschitz parameter dependence.
	Composing the three yields
	\begin{equation}
		\left|r(x_t,a_t,x_{t+1}) - r(x_t',a_t',x_{t+1})\right| \leq L \left(\|x_t-x_t'\| + \|a_t-a_t'\|\right)
	\end{equation}
	for some $L>0$ and all $x_t,x_t',x_{t+1} \in \X$ and $a_t,a_t' \in A$.
	Note that the cost term is what makes this estimate non-trivial: $k$ enters through an absolute value and is therefore not differentiable where the traded notional changes sign, but it is Lipschitz there, which is all that Assumption~\ref{asu_2}~(ii) requires.
	Therefore, Assumption~\ref{asu_2} is satisfied.}

    Analogous to Proposition~\ref{prop:bellman}, we can conclude by \citet[Theorem 2.7]{neufeld2023markov} that the Bellman equation holds true in the setting of Section~\ref{sec:port_opt}.

    Finally, we note that the duality follows directly from \citet[Theorem I]{wang2021sinkhorn} using the same arguments as in Proposition~\ref{prop:dual} with
    \begin{equation}
        \R \ni z \mapsto f(z) = -r(x_t,a_t,f_X(x_t,a_t,z)) - \gamma \sup_{a_{t+1} \in A} Q^*_\delta(f_X(x_t,a_t,z),a_{t+1})
    \end{equation}
    where $f(z)$ is the function being \emph{minimised} in the notation of \citet{wang2021sinkhorn}.

\end{proof}

\acks{J. Sester gratefully acknowledges support by the NUS Start-Up Grant \emph{Tackling model
uncertainty in Finance with machine learning}, by the MOE AcRF Tier 1 Grant 	 25-0428-P0001, and by the MOE AcRF Tier 2 Grant
T2EP20225-0030.}

\appendix

\section{Performance of agents in the reference distribution in the toy example} \label{app:toy_example}

Table \ref{tab:toy_perf_ref} shows the performance of the two algorithms in the same manner as Table \ref{tab:toy_perf_factor_5} but using the \emph{reference distribution} instead of the \emph{true distribution} for evaluation.
Unsurprisingly, the RDQN agent largely underperforms the DQN agent in this case due to the conservative nature of the Robust DQN agent.

\begin{table}[h!]
	\centering
	\begin{adjustbox}{max width=\linewidth}\begin{tabular}{lccrrrrrrr}
		\toprule
		Model & $\varepsilon$ & $\delta$ & Mean & Std & Min & 5\% & 10\% & 50\% & Max \\
		\midrule
		DQN & - & - & \textbf{0.077} & 0.029 & \textbf{0.015} & \textbf{0.039} & \textbf{0.042} & \textbf{0.075} & \textbf{0.155} \\
		RDQN & 0.05 & 0.0001 & 0.051 & \textbf{0.027} & -0.008 & 0.013 & 0.021 & 0.048 & 0.122 \\
		RDQN & 0.1 & 0.0001 & 0.058 & 0.037 & 0.000 & 0.007 & 0.010 & 0.057 & 0.147 \\
		RDQN & 0.1 & 0.01 & 0.058 & 0.031 & 0.000 & 0.012 & 0.020 & 0.055 & 0.135 \\
		RDQN & 0.2 & 0.0001 & 0.031 & 0.030 & 0.000 & 0.001 & 0.002 & 0.023 & 0.132 \\
		\bottomrule
	\end{tabular}\end{adjustbox}
	\caption{Performance of the strategies in terms of average reward per step as in Table \ref{tab:toy_perf_factor_5} but using the \emph{reference distribution} instead of the \emph{true distribution} for evaluation.}
	\label{tab:toy_perf_ref}
\end{table}

{
\section{Bias of the one-sample target in the S\&P 500 example} \label{app:sup_bias}

This appendix estimates the bias $b_Q$ of Remark~\ref{rem:sampled_multiplier}, which arises because Algorithm~\ref{algo:robustdqn} maximises over $\lambda$ separately for each nominal next-state sample instead of maximising the expectation over $\widehat{\PP}(x,a)$.
We use the trained RDQN agent of Section~\ref{sec:port_opt} with $\varepsilon=2\times10^{-3}$ and $\delta=5\times10^{-5}$, and take $Q$ in Remark~\ref{rem:sampled_multiplier} to be its trained critic.

\paragraph{States and nominal samples.}
We roll out the greedy policy of the agent on the simulator and, after a burn-in of 252 steps, select the first state whose latest log return falls in each of the seven ranges in Table~\ref{tab:sup_bias}.
These ranges cover ordinary days as well as rare large moves, at which the conditional distribution of the next log return is widest.
From each selected state $x$ we draw $M=8192$ next log returns $Z_1,\dots,Z_M$ from the generative model, i.e.\ from $\widehat{\PP}(x,a)$.
Since the next log return does not depend on the action, the same draws are used for every action $a$.

\paragraph{Inner expectation.}
Given $(x,a)$, the next state is a deterministic function of the next log return (Section~\ref{sec:port_opt}), so the expectation under $\nu$ in $h_Q$ is a one-dimensional integral.
We evaluate it with a deterministic quadrature rule that integrates the kernel $e^{-\|z-y\|/\delta}$ exactly on each cell.
This removes the error from sampling $N_\nu$ candidates from $\nu$ and isolates the outer expectation, which is the source of $b_Q$.

\paragraph{Estimators.}
For each $(x,a)$ we compute
\[
    \widehat{\widetilde{\HH}}_\delta Q(x,a):=\frac{1}{M}\sum_{m=1}^{M}\sup_{\lambda>0}h_Q(\lambda,Z_m;x,a),
    \qquad
    \widehat{\HH}_\delta Q(x,a):=\sup_{\lambda>0}\frac{1}{M}\sum_{m=1}^{M}h_Q(\lambda,Z_m;x,a).
\]
The first is the average of $M$ independent one-sample targets of Algorithm~\ref{algo:robustdqn} with the inner expectation evaluated exactly, and estimates $\widetilde{\HH}_\delta Q(x,a)$.
The second uses one multiplier $\hat\lambda$ common to all $M$ samples, and estimates $\HH_\delta Q(x,a)$.
Their difference
\[
    \hat b_Q(x,a):=\widehat{\widetilde{\HH}}_\delta Q(x,a)-\widehat{\HH}_\delta Q(x,a)
    =\frac{1}{M}\sum_{m=1}^{M}\left[\sup_{\lambda>0}h_Q(\lambda,Z_m;x,a)-h_Q(\hat\lambda,Z_m;x,a)\right]
\]
is an average of non-negative terms, and we use their sample standard deviation for the standard error of $\hat b_Q$.
Every estimate in Table~\ref{tab:sup_bias} lies at least $4.8$ standard errors above zero.
The pooled estimator $\widehat{\HH}_\delta Q$ is itself biased upwards by $O(1/M)$, so $\hat b_Q$ slightly underestimates $b_Q$.
Recomputing $\widehat{\HH}_\delta Q$ on disjoint subsets of $M/16$ and $M/4$ draws and extrapolating in $1/M$ puts this residual below $0.4\%$ of $\hat b_Q$ in every cell of the table.
For every draw, the feasibility condition $\bar\varepsilon>0$ of Section~\ref{sec:warning} holds, so no transitions are excluded.

\paragraph{Scale.}
To judge whether $\hat b_Q$ is small, we compare it with the robust penalty
\[
    \Pi_Q(x,a):=\E_{Z\sim\widehat{\PP}(x,a)}\left[F_Q(x,a,Z)\right]-(\HH_\delta Q)(x,a),
\]
i.e.\ the amount by which the robust target lies below the non-robust one, estimated with the same draws.
The ratio $\hat b_Q/\hat\Pi_Q$ is the fraction of the intended conservatism that the one-sample target gives back.
We report $\hat b_Q$ in the units of the reward, a daily log return.
A constant bias $b$ in the target has the same effect on the fixed point as adding $b$ to the reward at every step.

\begin{table}[!htbp]
	\centering
	\begin{adjustbox}{max width=\linewidth}\begin{tabular}{lcrrrrr}
		\toprule
		 & & \multicolumn{5}{c}{Action $a$ (weight in the S\&P 500)} \\
		\cmidrule(lr){3-7}
		Latest log return & sd of $Z$ & $-1$ & $-0.5$ & $0$ & $0.5$ & $1$ \\
		\midrule
		\multicolumn{7}{l}{\emph{Panel A: $\hat b_Q$ per step ($\times 10^{-6}$)}} \\
		$[-.200,-.100)$ & 0.038 & 556.2 & 75.9 & $<0.1$ & 6.3 & 1.0$^\dagger$ \\
		$[-.080,-.030)$ & 0.015 & 339.3 & 46.6 & $<0.1$ & 2.7 & 1.3$^\dagger$ \\
		$[-.020,-.010)$ & 0.011 & 250.1 & 34.0 & 5.1$^\dagger$ & 1.8 & 1.0 \\
		$[-.002,+.002)$ & 0.007 & 170.0 & 23.0 & 0.5 & 1.3 & 1.0 \\
		$[+.010,+.020)$ & 0.009 & 212.3 & 27.7 & 14.0 & 5.1 & 0.9 \\
		$[+.030,+.080)$ & 0.016 & 295.2 & 39.6 & 26.2 & 9.0 & 4.9$^\dagger$ \\
		$[+.100,+.200)$ & 0.047 & 1338.1 & 156.6 & 6.6 & 31.5 & 6.7$^\dagger$ \\
		\midrule
		\multicolumn{7}{l}{\emph{Panel B: $\hat b_Q/\hat\Pi_Q$ (\%)}} \\
		$[-.200,-.100)$ & 0.038 & 7.56 & 4.59 & 0.05 & 0.77 & 0.06$^\dagger$ \\
		$[-.080,-.030)$ & 0.015 & 4.46 & 2.69 & 0.02 & 0.31 & 0.07$^\dagger$ \\
		$[-.020,-.010)$ & 0.011 & 3.29 & 1.96 & 7.74$^\dagger$ & 0.21 & 0.05 \\
		$[-.002,+.002)$ & 0.007 & 2.22 & 1.31 & 0.77 & 0.14 & 0.05 \\
		$[+.010,+.020)$ & 0.009 & 2.78 & 1.59 & 11.79 & 0.62 & 0.05 \\
		$[+.030,+.080)$ & 0.016 & 3.89 & 2.31 & 26.27 & 1.05 & 0.27$^\dagger$ \\
		$[+.100,+.200)$ & 0.047 & 16.17 & 8.91 & 4.54 & 4.09 & 0.42$^\dagger$ \\
		\bottomrule
	\end{tabular}\end{adjustbox}
	\caption[Bias of the one-sample RDQN target on the S\&P 500 simulator]{Estimated bias $\hat b_Q$ of the one-sample target for the RDQN agent with $\varepsilon=2\times10^{-3}$ and $\delta=5\times10^{-5}$ (one seed) at seven states reached by its greedy policy on the simulator, using $M=8192$ nominal next log returns $Z$ per state. Each row is one state, labelled by the range of its latest daily log return, and ``sd of $Z$'' is the standard deviation of the $M$ draws. Panel A gives the bias per step in units of the reward (a daily log return). Panel B gives the bias as a percentage of the robust penalty $\hat\Pi_Q$. $^\dagger$ marks the greedy action; in the two states without a mark, the greedy action is not among the five evaluated.}
	\label{tab:sup_bias}
\end{table}

Table~\ref{tab:sup_bias} shows that the bias depends mainly on the action.
For the long position $a=1$, which is the greedy action in four of the seven states, $\hat b_Q$ is at most $6.7\times10^{-6}$ per step and at most $0.42\%$ of the robust penalty.
For $a=0.5$, it is at most $4.1\%$ of the penalty.
The bias is largest for short positions, at $2.2\%$ to $16.2\%$ of the penalty for $a=-1$, and it broadly increases with the dispersion of the next log return, peaking at the state reached after a daily gain of more than $10\%$.
For $a=0$, the ratio reaches $26\%$, but only because the penalty itself is small: an all-cash position exposes only the continuation value to the ambiguity, so $\hat\Pi_Q\leq1.5\times10^{-4}$ and $\hat b_Q\leq2.7\times10^{-5}$ per step.
Hence, at the actions the trained policy actually takes, the one-sample target retains almost all of the conservatism of the robust target, while for short positions it gives back a modest but non-negligible fraction.
}

\bibliography{literature}
\end{document}